\documentclass{article}
\usepackage{times}
\usepackage[letterpaper,textwidth=5.5in,textheight=9in,centering]{geometry}
\usepackage{natbib}

\usepackage[utf8]{inputenc}
\usepackage[dvipsnames]{xcolor}
\usepackage{graphicx}
\usepackage{verbatim}
\usepackage{subcaption}
\usepackage{amsmath}
\usepackage{amsfonts}
\usepackage{amssymb}
\usepackage{nicefrac}
\usepackage{microtype}
\usepackage{mathtools}
\usepackage{booktabs}
\usepackage{xspace}
\usepackage{url}
\usepackage{amsthm}
\usepackage[colorlinks=true, citecolor=PineGreen, linkcolor=Sepia, urlcolor=MidnightBlue]{hyperref}
\hypersetup{
  pdftitle={On Learning Expand-and-Sparsify Representations},
  pdfauthor={Alireza Asghari, Parikshit Ram, Kaushik Sinha},
  pdfsubject={},
  pdfkeywords={}
}
\usepackage{cleveref}

\usepackage{enumitem}
\setlist[itemize]{noitemsep, topsep=0pt, leftmargin=*}
\setlist[enumerate]{noitemsep, topsep=0pt, leftmargin=*}

\usepackage{wrapfig}
\usepackage{placeins}

\theoremstyle{plain}
\newtheorem{theorem}{Theorem}[section]

\newtheorem{lemma}[theorem]{Lemma}
\newtheorem{corollary}[theorem]{Corollary}
\theoremstyle{definition}

\newtheorem{assumption}[theorem]{Assumption}
\theoremstyle{remark}
\newtheorem{remark}[theorem]{Remark}

\usepackage{amsmath,amsfonts,bm}

\def\eqref#1{equation~\ref{#1}}

\def\1{\bm{1}}

\def\rvw{{\mathbf{w}}}
\def\rvx{{\mathbf{x}}}
\def\rvy{{\mathbf{y}}}
\def\rvz{{\mathbf{z}}}

\def\rmM{{\mathbf{M}}}

\def\rmW{{\mathbf{W}}}
\def\rmX{{\mathbf{X}}}
\def\rmY{{\mathbf{Y}}}

\def\mI{{\bm{I}}}

\DeclareMathAlphabet{\mathsfit}{\encodingdefault}{\sfdefault}{m}{sl}
\SetMathAlphabet{\mathsfit}{bold}{\encodingdefault}{\sfdefault}{bx}{n}

\newcommand{\E}{\mathbb{E}}

\newcommand{\R}{\mathbb{R}}

\DeclareMathOperator*{\argmin}{arg\,min}

\DeclareMathOperator{\sign}{sign}

\def\E{{\mathbb{E}}}
\def\I{{\mathbb{I}}}

\def\R{{\mathbb{R}}}
\def\Fcal{{\mathcal{F}}}
\def\Hcal{{\mathcal{H}}}
\def\Xcal{{\mathcal{X}}}
\def\Rcal{{\mathcal{R}}}

\def\relu{{\texttt{ReLU}}}
\def\topk{{\texttt{topk}}}
\def\lwta{{\texttt{LWTA@k}}}

\def\eas{{\texttt{EaS}}\xspace}
\def\bce{{\texttt{bce}}\xspace}
\def\logit{{\texttt{logistic}}\xspace}
\newcommand{\pr}{{\mathop {\rm Pr}}}

\title{On Learning Expand-and-Sparsify Representations}

\author{
  Alireza Asghari\\ Wichita State University\\ \texttt{axasghari@shockers.wichita.edu}
  \and
  Parikshit Ram\\ IBM Research\\ \texttt{parikshit.ram@ibm.com}
  \and
  Kaushik Sinha\\ Wichita State University\\ \texttt{kaushik.sinha@wichita.edu}
}
\date{}

\newcommand{\takeaway}[1]{\par\smallskip\noindent\fbox{\parbox{\dimexpr\linewidth-2\fboxsep-2\fboxrule\relax}{\textbf{Takeaway.}~#1}}\par\smallskip}

\begin{document}

\maketitle

\begin{abstract}
Expand-and-sparsify (\eas) representations map a $d$-dimensional input to a
much wider, $m$-dimensional space with a linear map and keep only its $k$
largest coordinates. The theory behind them draws the map at random and needs
$m$ to grow exponentially with $d$. We ask what happens at practical widths when the map is
instead learned by gradient descent: how much does learning add, and what does
the sparsity cost compared with a dense layer of the same width? On $121$
OpenML tasks, eight expansion factors and fully paired comparisons, we find
that (i)~learning the map matters most when the layer is narrow, adding a
median $3.4$ accuracy points at $m = d$ and less than half a point at
$m = 256d$; (ii)~from $m = 2d$ on, a learned \eas layer that keeps at most
$10\%$ of its units is within $0.7$ accuracy points of a dense $\relu$ layer of
the same width (equivalence tests; $0.83$ after correcting for testing eight
widths); (iii)~binarizing the \eas code costs less than half a point, so binary
\eas codes beat binary dense ones at the largest widths; and (iv)~the sparsity a
layer can afford within one point of $\relu$ grows with its width, from $90\%$
up to $m = 16d$ to $97\%$ from $m = 64d$. The comparison carries over to deep
models trained from scratch: in an MLP-Mixer on CIFAR-10 and a Transformer
language model on WikiText-103, top-$k$ with $10\%$ of units active matches
$\relu$ over five seeds, while $3\%$ costs up to $3.5$ points or $5\%$
perplexity, less the wider the layer. In short,
once an expanded layer is learned and at least twice as wide as its input,
keeping only a tenth of its units costs less than one accuracy point.
\end{abstract}

\section{Introduction}
Expanded layers are among the most heavily used components of modern
networks. The feedforward block of a Transformer \citep{vaswani2017attention}
and the channel-mixing MLP of an MLP-Mixer \citep{tolstikhin2021mlp} both
compute $\rmW\,\sigma(\rmM \rvx)$: $\rmM$ lifts a $d$-dimensional token $\rvx$ to a
hidden layer of width $m \approx 4d$, $\sigma$ is an activation, and $\rmW$ projects it back; in a Transformer these
blocks hold about two thirds of the non-embedding parameters. If the hidden
layer were sparse, most of that computation and memory traffic could be
skipped, which is why recent work exploits the sparsity that $\relu$ induces on
its own \citep{li2023lazy, liu2023dejavu} or enforces it with a top-$k$
operator \citep{gupta2021memory, wang2024qsparse, you2025spark}. Whether the
enforced sparsity costs accuracy has no settled answer: some studies find it
free \citep{gupta2021memory}, others find a plain top-$k$ operator to lose
quality unless it is redesigned \citep{you2025spark}, and in every case the
answer is entangled with the rest of a large architecture.

{\em Expand-and-sparsify} (\eas) representations are the cleanest form of
this idea: $\rvz = \sigma(\rmM\rvx) \in \R^m$ keeps exactly the $k \ll m$
largest coordinates of an expansion with $m > d$. They model the olfactory
circuit of the fly \citep{turner2008, caron2013}, have been used for
similarity search and novelty detection \citep{dasgupta2017flyhash,
dasgupta2018novelty}, and come with approximation guarantees: a linear readout
of $\rvz$ can represent any smooth function \citep{dasguptatosh2020,
sinha2024nonparametric}. These guarantees hold at impractical scales, with a
random $\rmM$ and $m$ exponential in $d$; at practical sizes ($m$ a small
multiple of $d$, $k$ a fixed fraction of $m$) the bound we derive does not vanish
(\cref{app:apprx-gen-theory}). So it has not been clear whether \eas layers are
useful at practical sizes.

\paragraph{The question.} We ask what happens when the expansion of an \eas
layer is learned by gradient descent at practical widths: \emph{(i)~how much
does learning add over the theoretical constructions, and (ii)~what does the
sparsity cost relative to a dense layer of the same width?}

\paragraph{Approach.} We isolate a single expanded layer, so that the
activation and the construction of $\rmM$ are the only things that change, and
compare learned, data-supported and random $\rmM$, and top-$k$, local
winner-take-all, $\relu$, \texttt{GELU} and \texttt{SiLU} activations, on $73$
classification and $48$ regression tasks from OpenML at eight expansion factors
$m/d \in [1, 256]$ with at most $10\%$ of units active. All comparisons are
paired on identical datasets, widths and splits, and every claim of parity is
backed by an equivalence test, not by a non-significant one
(\cref{sec:expsetup}). We then repeat the central comparison inside an
MLP-Mixer and a Transformer (\cref{sec:archs}).

\begin{figure}[t]
\centering
\includegraphics[width=\linewidth]{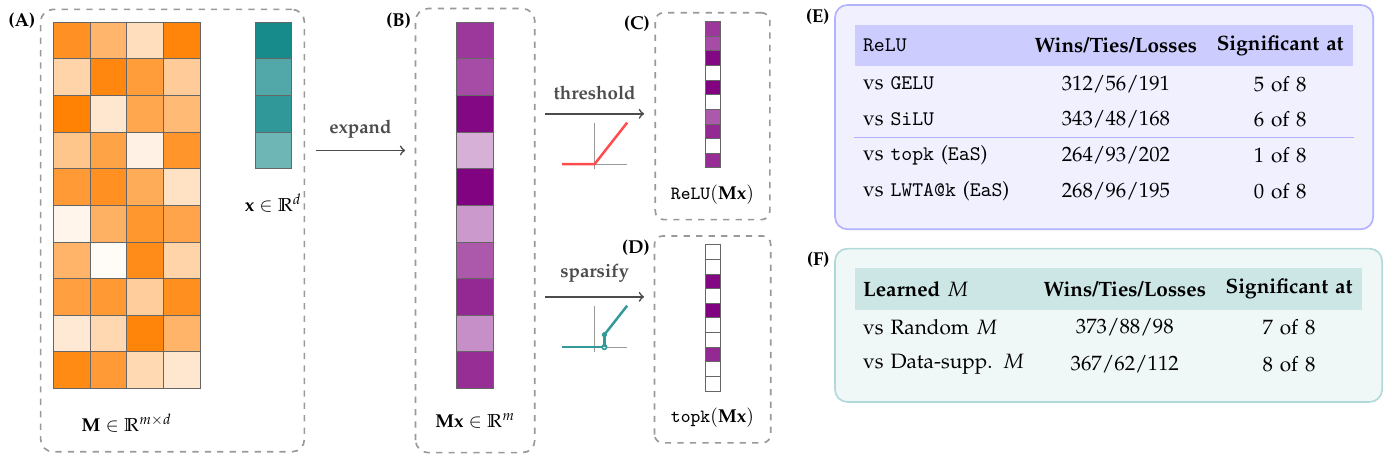}
\vskip -8pt
\caption{{\bf What we compare, and what we find.} 
{\bf (A)}~A linear operator $\rmM \in \R^{m \times d}$ is applied to
the input $\rvx \in \R^d$, giving us
{\bf (B)}~an expanded representation
$\rmM \rvx \in \R^m$,
which is then either {\bf (C)}~thresholded, giving a
dense $\relu$ layer, or 
{\bf (D)}~sparsified, giving the \eas representation. 
{\bf (E)}~The table compares $\relu$ against
alternate activations
with $\rmM$ learned throughout
--- \texttt{GELU} and \texttt{SiLU} are other dense activations, and $\topk$ and $\lwta$ serve as \eas versions with different sparsification schemes.
{\bf (F)}~The table compares a learned $\rmM$ against  theoretical
constructions from \citet{dasguptatosh2020} with $\topk$ activation throughout.
We summarize the results in the form of
wins/ties/losses over the
(dataset, $m/d$) configurations,
and ``significant at'' counts how many
of the 8 expansion factors reach $p < 0.05$ under a two-sided Wilcoxon
signed-rank test 
(per-expansion factor results are in \cref{fig:boxes} and \cref{tab:wtl}; 
more details are presented in \cref{app:tables} and \cref{app:counts}; counts include the configurations rerun for this version, \cref{app:repro}).
The results show that the learned \eas representations not only {\em significantly improve upon the theoretical constructions}, but also provide {\em competitive performance to standard activations used in feedforward networks}. 
}
\label{fig:summary}
\end{figure}

\paragraph{Findings} (\cref{fig:summary}).
\begin{itemize}[leftmargin=*, itemsep=1pt, topsep=2pt]
\item \textbf{Learning the expansion matters most when the layer is narrow.}
Against a random $\rmM$ the median gain falls from $3.4$ accuracy points at
$m = d$ to $0.13$ (95\% CI $[0.00, 0.40]$) at $m = 256d$; against an $\rmM$
sampled from the data it persists at every width (\cref{sec:res-projection}).
\item \textbf{Sparsity is nearly free once the layer is at least twice as wide
as its input.} From $m = 2d$, top-$k$ with at most $10\%$ of units active is
within $0.7$ accuracy points of $\relu$ at every expansion factor (equivalence
tests; $0.83$ with Holm correction across widths). On regression the sparse
layer trails at narrow widths and is within $0.02$ in $R^2$ from $m = 16d$
(\cref{sec:res-vs-relu}).
\item \textbf{Binary sparse codes lose almost nothing.} Binarizing top-$k$
costs less than half a point at every width, while binarizing $\relu$ costs about
$0.4$ points at $m \geq 128d$, so the binary \eas code is ahead at $m \geq 128d$
(from $16d$ uncorrected; \cref{sec:res-binary}).
\item \textbf{Wider layers can afford more sparsity.} The highest sparsity
within one point of $\relu$ rises from $90\%$ at $m \leq 16d$ to $97\%$ from
$m = 64d$ (Holm-corrected; uncorrected, $97\%$ from $16d$ and $99\%$ at $256d$);
at $99\%$ the deficit shrinks from $2.4$ points at $m = 4d$ to $0.35$ at $256d$
(\cref{sec:res-sparsity}).
\item \textbf{The comparison carries over to deep architectures.} Replacing
$\relu$ by top-$k$ with $10\%$ of units active changes neither CIFAR-10 accuracy
of an MLP-Mixer ($-0.04$ points over $15$ paired runs, $p = 0.95$) nor
WikiText-103 perplexity of a Transformer ($-0.1\%$, $p = 0.60$). At $3\%$ both
pay, $1.8$ to $3.5$ points in the Mixer and up to $4.6\%$ perplexity in the
Transformer, and in both the cost falls as the layer widens; in the largest
Transformer at the widest ratio it disappears (\cref{sec:archs}).
\end{itemize}

In short, the obstacle to using \eas layers in practice is not their
sparsity: once the expansion is learned and the layer is at least twice as wide
as its input, keeping only a tenth of its units costs less than one accuracy
point, and the wider the layer, the sparser it can afford to be.

\paragraph{Scope.} The tabular study is the controlled core of the paper; the
architecture experiments test whether its conclusions transfer, at small
scale. The exact
margins depend on the tasks and on tuning; the trends are the contribution. Datasets,
preprocessing and training details are in \cref{app:repro}; related work is in
\cref{sec:related}.

\section{Problem and Experimental Setup}
\label{sec:expsetup}

\paragraph{The model.}
We consider supervised learning problems, given training data
$S = \{ (\rvx_i, \rvy_i), i = 1, \ldots, n \}$ with input features $\rvx_i \in \R^d$ and the
output $\rvy_i \in \mathcal{Y} \subset \R^l$, and a model of the form
\begin{equation}
\label{eq:model}
f(\rvx; \rmW, \rmM) = \rmW \rvz,\quad \rvz \triangleq \sigma( \rmM \rvx ) \in \R^m, \quad
\rmM \in \R^{m \times d}, \rmW \in \R^{l \times m}.
\end{equation}
Here $\rmM$ is the expansion matrix, $\sigma$ is a nonlinear activation, $\rmW$ is
the linear readout, and $(\rmW, \rmM)$ are the learnable parameters. If $l$ is the
number of classes, $f$ is a multi-class classifier; if $l = d$, this model is
the feedforward channel mixer used in the Transformer and the MLP-Mixer. With
$\sigma$ a standard activation such as $\relu$, \cref{eq:model} is a
single-hidden-layer feedforward network; with $\sigma$ the $\topk$ function we
get $\|\rvz\|_0 = k$, which is the \eas representation. We can further restrict
the representation to be binary, $\rvz \in \{0,1\}^m$, in which case a standard
activation reduces to the Heaviside function while the \eas representation has
exactly $k$ ones and $(m - k)$ zeros.
Given a loss $\ell$, we learn $(\rmW, \rmM)$ by minimizing the empirical loss 
$\sum_{i=1}^n \ell(\rvy_i, \rmW \cdot \sigma(\rmM \rvx_i))$ with SGD, for different
activations and different values of $m$ and $k$. With theoretical
constructions of $\rmM$, we learn only $\rmW$.

\begin{wrapfigure}[17]{R}{0.5\textwidth}
\vskip -12pt
\centering
\includegraphics[width=0.5\textwidth]{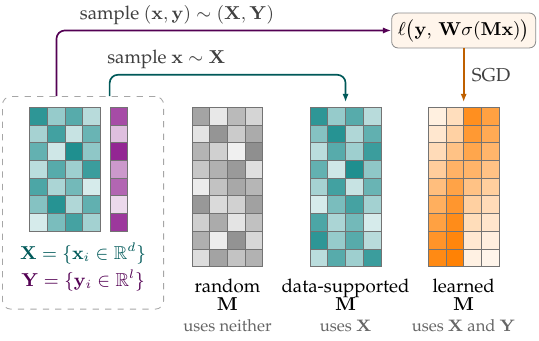}
\vskip -8pt
\caption{The three constructions of $\rmM$ and what each may use. The random
(data-independent) $\rmM$ uses neither $\rmX$ nor $\rmY$, the data-supported
$\rmM$ samples its rows from $\rmX$, and the learned $\rmM$ is trained by SGD
on a loss that sees both.}
\label{fig:mconst}
\end{wrapfigure}
\paragraph{Constructions of the expansion matrix.}
We compare three ways of obtaining the expansion matrix $\rmM$
(\cref{fig:mconst}), holding everything else fixed. In the \emph{learned} construction, $\rmM$ is trained by SGD jointly
with the readout $\rmW$. In the \emph{data-supported} construction, the $m$ rows
of $\rmM$ are drawn from the training inputs with replacement, so that the
expansion directions lie in the support of the data, which is the setting
\citet[\S5]{dasguptatosh2020} analyze. In the \emph{data-independent}
construction, the rows are drawn i.i.d.\ from a fixed distribution (the uniform distribution on the unit sphere in $\R^d$, drawn by normalizing samples from $\mathcal N(0, \mI_d)$). In all
three cases, $\rmW$ is fitted by SGD on all $n$ training points.

Drawing with replacement matters for two reasons. It keeps the construction
defined for every dataset at every expanded dimension, including $m > n$, so
the three constructions are compared on the same datasets throughout. And it
is i.i.d.\ sampling from the empirical distribution, whereas drawing $m$
distinct points is not i.i.d.\ at all. The cost is that $\rmM$ can hold at most
$n$ distinct rows, which \cref{sec:res-projection} shows is what eventually
limits this construction.

\paragraph{Activations.}
The sparse activations are $\topk$ and local winner-take-all ($\lwta$), which splits the representation into blocks of size $(m/k)$ and selects the top-1 in each block~\citep{srivastava2013compete}, resulting in one nonzero per block; with blocks of size $\mathrm{round}(m/k)$ that is exactly $k$ nonzeros when $k$ divides $m$ and can be slightly more otherwise (\cref{app:lwta}). We treat $\lwta$ as an alternative
sparsification scheme within \eas rather than as a separate method: both keep
(about) $k$ of the $m$ coordinates and so carry the same sparsity budget, and
they differ only in which $k$ survive.
The $\topk$ activation lets all $m$ units compete,
while $\lwta$ lets them compete only inside blocks of size $(m/k)$. For any single
input the two coincide under a permutation of the rows of $\rmM$ that puts the $k$
winners in distinct blocks, and when $\rmM$ is learned the assignment of
directions to blocks is learned along with the directions, so the network can
partly arrange that for itself. We therefore expect the two to behave
similarly, and \cref{sec:res-vs-relu} finds that they do. We run each of them
in a
value-passing form, which forwards the surviving pre-activations, and in a
binary form, which forwards their indicator and is trained through a
straight-through estimator. The dense baselines are $\relu$, \texttt{GELU} and \texttt{SiLU}.
Binarizing any of these three gives the same function, the indicator of a
positive pre-activation, so there is one dense binary baseline rather than
three and we run only that one. The binary forms are evaluated against each
other in \cref{sec:res-binary} and are not mixed into the comparisons among
value-passing activations.

\paragraph{Parameterization.}
We report results as functions of the expansion factor $m/d$ and the sparsity
level $k/m$ rather than of $m$ and $k$. These are equivalent, but the ratios
are comparable across datasets of differing dimension, and $k$ on its own
cannot be separated from $m$ in our grid: since $k = (k/m)\cdot(m/d)\cdot d$, a
large absolute $k$ arises only at a large expanded dimension. We sweep
$m/d \in \{1, 2, 4, 8, 16, 64, 128, 256\}$ and
$k/m \in \{0.001, 0.01, 0.03, 0.05, 0.10\}$, restricting throughout to
$k \leq 0.1m$.

\paragraph{Data.}
We use $73$ classification tasks from OpenML, admitted by OpenML's
\texttt{NumberOfFeatures} $\in [8, 120]$ (which counts the target column) and
\texttt{NumberOfInstances} $\in [300, 20000]$, with missing values imputed by
the column median (computed before splitting) and categorical features one-hot
encoded. Of these, $37$ are binary
and $36$ are multi-class. We use a further $48$ regression tasks admitted by
the same criteria. Inputs are standardized and scaled to unit norm, and the rows of
$\rmM$ are renormalized to unit norm in every forward pass, so every
pre-activation is a cosine in $[-1, 1]$. We deliberately chose a large and varied set of tasks so that we can see how
generally the behavior of \eas holds. \Cref{app:repro} lists every
OpenML ID and gives the full preprocessing.

\paragraph{Protocol.}
Every construction sees identical datasets, expanded dimensions, sparsity
levels and train/test splits, so all comparisons are paired. We repeat every
configuration on $3$ random train/\allowbreak validation/\allowbreak test splits ($70/10/20$,
stratified for classification), with a different seed for each, and report the
mean test score. We train with Adam at a learning rate of
$10^{-3}$, no weight decay and a batch size of $256$, for up to $300$ epochs
with early stopping at a patience of $30$. For each (dataset, $m/d$,
activation, construction) configuration we select $k$ on validation accuracy
and report the corresponding test accuracy; the dense activations have no $k$
and are unaffected by this selection. We cap $m \leq 30{,}000$ due to computational limitations.
Model details, the straight-through estimator for binary codes, and the
$31$ configurations rerun for this version are in \cref{app:repro}.

Two things follow from that cap and from $k$ being an integer, and both change
which datasets are available at a given expansion factor. A dataset drops out
at the wide end when $m = (m/d)\cdot d$ exceeds the cap, which happens first
for the highest-dimensional datasets. And it drops out at the narrow end when
$m$ is too small for any integer $k$ to satisfy $k \leq 0.1m$, which happens
first for the lowest-dimensional ones. The pool is therefore not the same at
every $m/d$, so we compare methods only within an expansion factor and never
read a trend along a row of counts. \Cref{app:counts} gives the pool sizes at each
expansion factor and which of the two causes is responsible. 

\paragraph{Statistics.}
For each expanded dimension, we compare two methods only on the datasets where
both are available, so every comparison is paired. Within a given $m/d$ we use
the same set of datasets for every method being compared, so that a method is
never scored on an easier subset than the one it is being compared against.
We report raw metrics (accuracy $\in [0,1]$ for classification, and $R^2$ for regression), except in \cref{fig:scaled} where we use the scaled form of the metrics to aggregate performance across datasets~\citep{ram2022flynn}.
Because $R^2$ is unbounded below,
we verify every regression conclusion twice, once with $R^2$ floored at zero
and once left raw, and we say so wherever the two disagree. In both cases
$k$ is chosen on the raw validation score, so flooring changes what is reported
but not which model is selected. 
When comparing a pair of methods
we report the median difference, the interquartile range, the number of
datasets on which the first method wins, ties or loses, and the $p$-value of a
two-sided Wilcoxon signed-rank test. We count a tie when the difference is at most $0.001$ (that is $0.1$
percentage points of accuracy or $0.001$ in $R^2$).
A non-significant test is not evidence that two methods are equivalent, so
every claim of parity is stated with a margin. We say a method is
\emph{significantly} behind another if the two-sided Wilcoxon test gives
$p < 0.05$, and \emph{within $\delta$} of it if two one-sided Wilcoxon tests
(TOST) reject a median gap of at least $\delta$ at level $0.05$. We use
$\delta = 1$ accuracy point for classification and $\delta = 0.02$ in $R^2$ for
regression, and report $\delta^\ast$, the smallest margin the data support,
together with the exact $95\%$ confidence interval of the median difference
(\cref{tab:equiv}). As a second check we use the Bayesian signed-rank test with
a region of practical equivalence \citep{benavoli2017time}, and quote
$P(\text{equivalent})$, the posterior probability of that region. Each
comparison is made at eight expansion factors, so \cref{tab:holm} also gives
Holm-corrected $p$-values and margins across them, and a result that is
significant only without the correction is marked $^\dagger$.

\section{Main Results}
\label{sec:results}

Each subsection states one finding in its title, gives the evidence, and ends
with a one-line takeaway. All comparisons are paired within an expansion factor,
and parity is always stated with a margin (\cref{sec:expsetup}, Statistics);
\cref{tab:equiv} gives the confidence interval and $\delta^\ast$ behind every
such statement.

\begin{figure}[t]
\centering
\includegraphics[width=0.97\linewidth,height=3.95in,keepaspectratio]{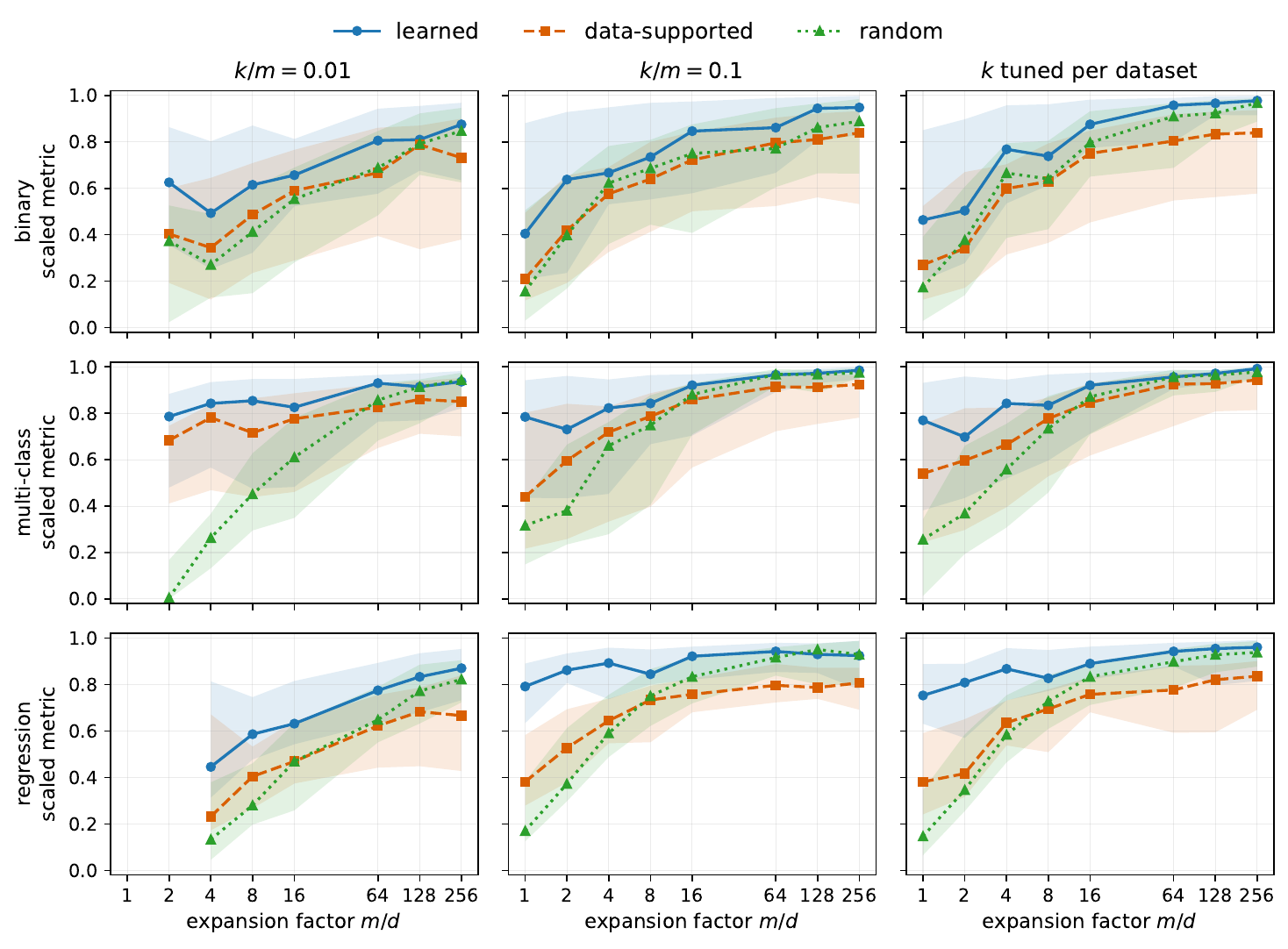}
\caption{Scaled metrics (accuracy for classification, $R^2$ for regression;
\emph{higher is better}) of the three constructions of $\rmM$ against the
expansion factor $(m/d)$, with $\topk$ sparsification throughout. Rows are task types
(binary classification, multi-class classification and regression) and columns
are sparsity levels $(k/m)$, with the rightmost column selecting $k$ on validation.
Each line is the median over datasets and the ribbon is the interquartile
range.
As discussed in \cref{sec:expsetup}, evaluations for $k < 1$ and $m > 30{,}000$ are skipped so results for small $m/d$ and small $k/m$ are absent when $d$ is small.
Learned \eas improves upon both theoretical
constructions at small expansion factors in every panel. As the expansion grows
the random construction catches up, while the data-supported one mostly stays below
(\cref{sec:res-projection}).}
\label{fig:scaled}
\end{figure}

\subsection{Learning the expansion matters most when the layer is narrow}
\label{sec:res-projection}

Against a random $\rmM$, the construction the theory uses, learning $\rmM$ adds
a median $3.4$ accuracy points at $m = d$. The gain falls steadily as the layer
widens (\cref{fig:scaled}): $3.1$, $2.5$, $1.6$ and $1.0$ points at
$m/d = 2, 4, 8, 16$, and $0.30$ and $0.24$ at $64d$ and $128d$, all significant
($p \leq 0.006$). At $m = 256d$ the gain is $0.13$ points and no longer
significant ($p = 0.13$, 95\% CI of the median $[0.00, 0.40]$), and learning is
within half a point of not learning (TOST $p = 0.005$; Bayesian
$P(\text{equivalent}) = 0.95$). Regression shows the same decay, from a median
$0.17$ in $R^2$ at $m = d$ to $0.002$ at $m = 256d$, significant through
$m = 64d$ but at neither of the two widest ($p \geq 0.14$). Learning the
expansion thus buys much of what a larger expanded dimension would have bought
anyway, so a learned \eas layer can be made small.

The data-supported $\rmM$, whose rows are training points, behaves differently.
The learned $\rmM$ stays ahead of it at every width, by $0.65$ to $1.68$ points
($p \leq 0.0011$), and on regression by $0.13$ in $R^2$ at $m = d$ down to
$0.03$ at $m = 256d$ ($p < 10^{-4}$). \Cref{app:distinct} traces the difference
to the data-supported $\rmM$ running out of distinct rows as $m$ grows.

\takeaway{Learning $\rmM$ buys what extra width would otherwise buy: worth
$3.4$ points at $m = d$, it is within half a point of a random $\rmM$ at
$m = 256d$.}

\begin{figure}[t]
\centering
\includegraphics[width=0.97\linewidth,height=3.7in,keepaspectratio]{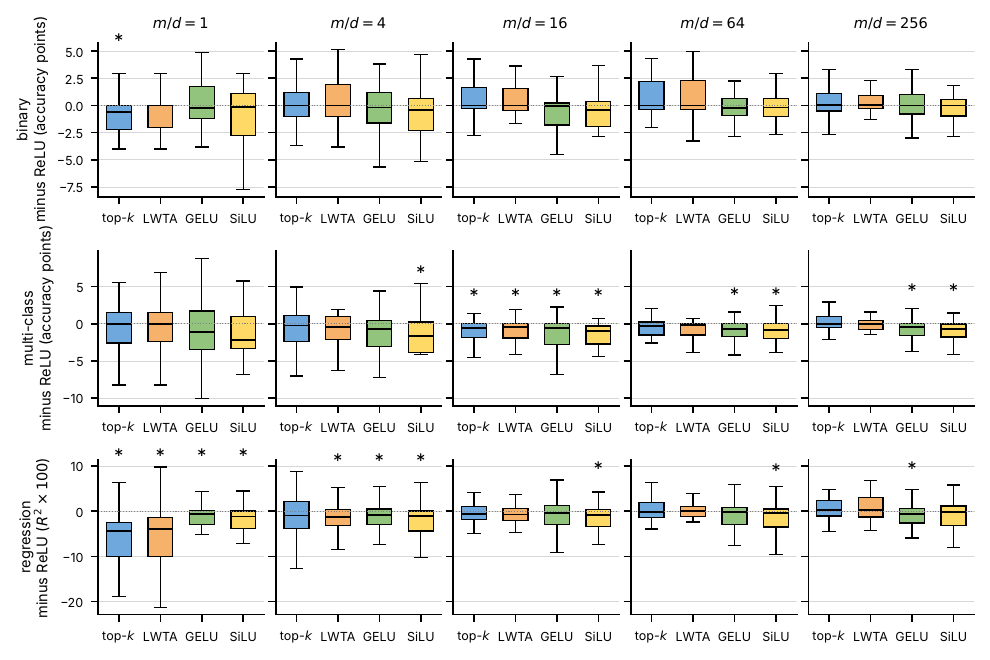}
\caption{Per-dataset differences against $\relu$ at equal expanded dimension,
with $\rmM$ learned for both methods and $k$ selected on validation for the sparse
activations. Rows are task types and columns are expansion factors. Each box is
the distribution over datasets of the difference for one activation, so values
above zero mean the activation beats $\relu$ on that dataset. Within a panel
every activation is scored on the same datasets, and a star marks a significant
difference from $\relu$ in that panel (two-sided Wilcoxon, $p < 0.05$). Whiskers
extend to $1.5$ times the interquartile range and outliers are not drawn. All
differences are multiplied by $100$: accuracy differences in percentage points for
classification and differences in $R^2$ for regression.}
\label{fig:boxes}
\end{figure}

\begin{table}[t]
\caption{Each activation against $\relu$ at equal expanded dimension on the 73 classification tasks, with $\rmM$ learned for both methods. Each cell gives wins/ties/losses over datasets and, in parentheses, the $p$-value of a two-sided Wilcoxon signed rank test on the per-dataset differences; a star marks $p<0.05$. Within a column every activation is scored on the same datasets (65 to 72 per column). A non-significant test is not evidence of equivalence; \cref{tab:equiv} gives confidence intervals and equivalence margins. Five of the eight expansion factors are shown here; \cref{tab:wtlfull} in \cref{app:tables} gives all eight.}
\label{tab:wtl}
\centering
\small
\setlength{\tabcolsep}{4pt}
\begin{tabular}{lccccc}
\toprule
activation & $m/d=1$ & $m/d=4$ & $m/d=16$ & $m/d=64$ & $m/d=256$ \\
\midrule
top-$k$ & 18/16/31 (.035)$^{\ast}$ & 27/10/35 (.576) & 24/12/35 (.727) & 27/9/35 (.800) & 31/9/27 (.321) \\
LWTA & 20/16/29 (.145) & 25/9/38 (.380) & 25/11/35 (.373) & 21/16/34 (.725) & 28/13/26 (.616) \\
GELU & 27/3/35 (.497) & 24/4/44 (.038)$^{\ast}$ & 19/10/42 (.002)$^{\ast}$ & 21/8/42 (.005)$^{\ast}$ & 22/8/37 (.101) \\
SiLU & 23/6/36 (.085) & 22/4/46 (.003)$^{\ast}$ & 15/6/50 ($<$.001)$^{\ast}$ & 21/6/44 (.002)$^{\ast}$ & 20/11/36 (.007)$^{\ast}$ \\
\bottomrule
\end{tabular}
\end{table}

\subsection{With a tenth of its units active, a learned \eas layer is within a point of a dense one}
\label{sec:res-vs-relu}

From $m = 2d$ on, a learned \eas layer with at most $10\%$ of its units active
is within $0.7$ accuracy points of a dense $\relu$ layer of the same width, at
every expansion factor we study ($\delta^\star \leq 0.70$, and $0.83$ with
Holm correction; \cref{fig:boxes}, \cref{tab:wtl,tab:equiv,tab:holm}). The median gap lies between $-0.22$ and $+0.05$
points, and the 95\% confidence interval of the median between $-0.78$ and
$+0.33$. $\lwta$ gives the same picture ($\delta^\star \leq 0.76$). The Bayesian
test puts the probability of practical equivalence at $0.93$ or more at every
width from $2d$ except $4d$, where it is $0.51$ and the remaining mass favors a
deficit. The narrowest layer is the exception: at $m = d$, top-$k$ is
significantly behind $\relu$ ($p = 0.035^\dagger$) and the data do not support a
one-point margin ($\delta^\star = 1.53$). There $k \leq 0.1m$ leaves very few
active units, one unit when $d < 20$. A one-point margin is coarse, and a
half-point margin holds only at $m/d \in \{16, 64, 128\}$, and at none after correction, so we claim parity at
the one-point level, not exact parity.

In our setup the smooth activations fall further behind: from $m = 4d$,
\texttt{GELU} trails $\relu$ by a median $0.16$ to $0.44$ points and
\texttt{SiLU} by $0.25$ to $0.87$ (significant after Holm correction at two and
at six of six widths). This is a property of our parameterization, not of the
activations. Inputs and rows of $\rmM$ have unit norm and there is no hidden
bias, so every pre-activation is a cosine in $[-1, 1]$, where \texttt{GELU} and
\texttt{SiLU} are within $1.2\%$ and $0.4\%$ of a quadratic (maximum error
relative to their range), against $9.4\%$ for $\relu$. Multiplying every
pre-activation by $c = 4$ or $16$ removes the gap: at equal $c$ both are within
$0.49$ points of $\relu$ at every width we test, $m/d \in \{2, 4, 8, 16, 64\}$
(\cref{app:scale}). Scaling also helps $\relu$ and top-$k$, by a median $1.1$ and
$1.0$ points at $m = 4d$, through optimization; with $c$ tuned on validation for
both, top-$k$ is within $1.25$ points of $\relu$ at $2d$ and $4d$, $0.9$ at $8d$ and
$0.8$ from $16d$.
\eas trains no slower, but its train--test gap is larger
(\cref{app:conv}).

On regression, the sparse layer pays a measurable price at narrow widths and
little from $m = 16d$. Top-$k$ trails $\relu$ by a median $0.044$ and $0.032$ in
$R^2$ at $m = d$ and $2d$ ($p < 10^{-3}$). From $4d$ to $16d$ the gap is smaller
(medians $-0.010$, $-0.014$, $-0.007$) but consistent: top-$k$ loses on about
two thirds of the datasets (sign test $p \leq 0.040^\dagger$). From $m = 16d$ the gap
is within $0.02$ (TOST $p \leq 0.021$; Bayesian $P(\text{equivalent}) \geq 0.94$;
$0.025$ after correction). $\lwta$ is significantly behind at $m \leq 2d$
($p < 10^{-3}$; at $4d$, $p = 0.010^\dagger$) and within $0.02$ from $m = 8d$
($0.025$ after correction). \Cref{tab:wtlreg} gives every value; the smooth
activations again fall short of $\relu$, after correction significantly only
\texttt{SiLU} up to $m = 8d$.

Two qualifications apply. First, $k$ is selected on validation for the sparse
layers, and the dense layers have no comparable knob (\cref{app:tuning});
\cref{sec:res-sparsity} removes this advantage. Second, splitting classification
by task type gives three small significant deficits out of sixteen comparisons:
binary tasks at $m = d$ ($0.63$ points, $p = 0.002$) and multi-class tasks at
$8d$ and $16d$ ($0.49$ and $0.52$ points, $p \leq 0.041^\dagger$), marked in
\cref{fig:boxes}.

\takeaway{From $m = 2d$, keeping at most a tenth of the units costs at most
about $0.7$ accuracy points against $\relu$ ($0.83$ with Holm correction);
regression needs $m \geq 16d$ for a $0.02$ margin in $R^2$.}

\begin{figure}[t]
\centering
\includegraphics[width=\linewidth]{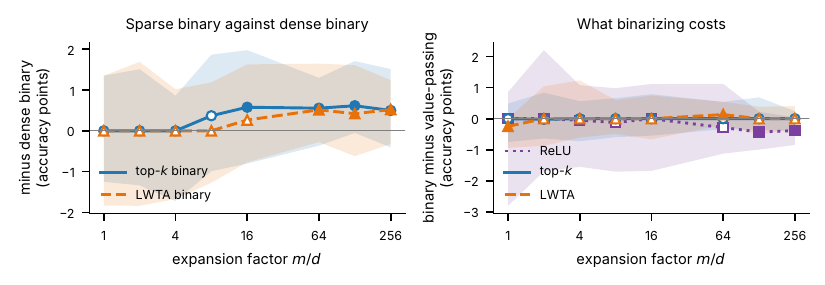}
\caption{The effect of binarizing expanded representations on classification tasks.
{\bf [Left]} The signed difference between the accuracy of the binarized \eas and binarized $\relu$ representations ({\em higher values indicate better \eas performance relative to the dense versions}).
{\bf [Right]} Per-method performance degradation due to binarization ({\em lower values imply larger degradations}).
Lines are
medians over datasets and ribbons the interquartile range; filled markers are
significant under a two-sided Wilcoxon signed-rank test ($p < 0.05$) and hollow
ones are not. \Cref{tab:wtlbin} gives every value.}
\label{fig:binary}
\end{figure}

\subsection{Binary \eas codes lose almost nothing, and overtake binary dense codes at large widths}
\label{sec:res-binary}

We now restrict the representation to be binary, $\rvz \in \{0,1\}^m$, and train
it with the straight-through estimator of \cref{app:repro}. Binarizing $\relu$,
\texttt{GELU} or \texttt{SiLU} gives the same function, so there is a single
dense binary baseline.

Binarizing the \eas code, i.e.\ keeping only which $k$ units are active, costs
less than $0.4$ accuracy points at every expansion factor
($\delta^\star \leq 0.36$, and $0.43$ with Holm correction; median gap within $0.05$ of zero; Bayesian
$P(\text{equivalent}) \geq 0.99$ at $\delta = 1$). Binarizing $\relu$ costs
more where the layer is wide: a median $0.43$ points at $m = 128d$ ($p = 0.027^\dagger$)
and $0.39$ at $256d$ ($p = 0.004$). As a result, the binary \eas code is
significantly ahead of the binary dense code from $m = 16d$, by a median $0.58$,
$0.56$, $0.62$ and $0.50$ points at $m/d = 16, 64, 128, 256$ ($p \leq 0.043$;
\cref{fig:binary}, \cref{tab:wtlbin}); after Holm correction across widths, the
lead remains significant at $128d$ and $256d$ (\cref{tab:holm}). The binarized $\lwta$ code is ahead from
$m = 64d$, by $0.42$ to $0.52$ points ($p = 0.011^\dagger$, $0.034^\dagger$ and $0.005$);
binarizing $\lwta$ costs $0.25$ points at $m = d$ ($p = 0.012^\dagger$) and nothing
measurable elsewhere. So the
advantage does not come from the sparse code getting better. It comes from the
sparse code losing almost nothing while the dense code loses about $0.4$ points,
and the memory and compute savings of a binary code
(\cref{app:sparse-advantages}) come at that cost.

On regression, binarizing costs every activation something, so the mechanism is
different. Top-$k$ loses $0.010$ to $0.038$ in $R^2$ and $\relu$ loses $0.015$ to
$0.068$, both significantly at every width. The binary \eas code is ahead at
$m/d \in \{16, 64, 128\}$, by $0.014$, $0.028$ and $0.017$ ($p = 0.030^\dagger$,
$0.001$ and $0.009^\dagger$), because it loses less; at $m = d$ it is behind
($0.041$, $p = 0.043^\dagger$)
(\cref{tab:wtlreg}).

\takeaway{A binary \eas code is within half a point of its real-valued version
at every width; a binary dense code is not, and falls behind at large widths.}

\subsection{Wider layers can afford more sparsity}
\label{sec:res-sparsity}
\begin{wrapfigure}[22]{r}{0.43\textwidth}
\vskip -0.25in
\centering
\includegraphics[width=0.43\textwidth]{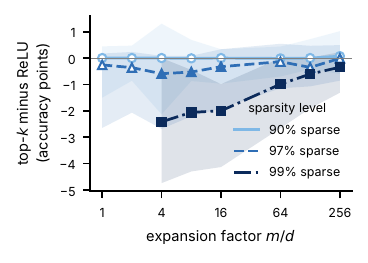}
\caption{Top-$k$ \eas at a fixed sparsity $k/m$ minus dense $\relu$, against the expansion factor $m/d$.
Lines are medians over datasets, ribbons the interquartile range; filled markers are significant ($p < 0.05$).
Configurations whose realized $k/m$ exceeds $0.105$ are excluded. A point is drawn only where at least $30$ datasets are present: $99\%$ sparsity %
leaves $8$ and $23$ datasets at $m = d$ and $m = 2d$. \Cref{tab:wtlsparse} gives every value.}
\label{fig:sparsity}
\end{wrapfigure}

\Cref{sec:res-vs-relu} selected $k$ on validation. Here we fix $k/m$ at $0.10$,
$0.03$ or $0.01$, i.e.\ $90\%$, $97\%$ or $99\%$ of units inactive
(\cref{fig:sparsity}, \cref{tab:wtlsparse}). Following the \emph{critical
sparsity} of \citet{szatkowski2026universal}, the highest sparsity with at most a
$1\%$ performance drop, we call a sparsity level \emph{affordable} at a width if
its gap to $\relu$ is within one accuracy point. Among the three levels, the
highest affordable one rises with width: it is $90\%$ for $m \leq 8d$, $97\%$
from $m = 16d$ to $128d$, and $99\%$ at $m = 256d$ (bold entries of
\cref{tab:wtlsparse}). With Holm correction across widths (\cref{tab:holm}) the
step to $97\%$ comes later, at $m = 64d$, and $99\%$ at $256d$ is no longer
supported ($p = 0.064$).

At $90\%$ sparsity the \eas layer is within one point of $\relu$ at every width
(TOST $p \leq 0.021$; $\delta^\star$ between $0.28$ and $0.82$); the median gap
lies between $-0.01$ and $+0.07$ points and no Wilcoxon test is significant
($p \geq 0.21$). At $97\%$ it is significantly behind at $m/d \in \{4, 8, 16\}$,
by a median $0.59$, $0.53$ and $0.33$ points ($p = 0.002$, $0.003$ and
$0.033^\dagger$). From $16d$ it is nonetheless within one point (from $64d$ after
correction), and from $64d$ within $0.6$ ($0.83$ after correction). At $99\%$ it is
significantly behind at every width from $4d$ (at $256d$, $p = 0.039^\dagger$), but the deficit shrinks
monotonically: $2.42$, $2.06$, $2.00$, $0.99$, $0.61$ and $0.35$ points at
$m/d = 4, 8, 16, 64, 128, 256$, reaching the one-point margin only at
$m = 256d$ (TOST $p = 0.008$; $0.064$ after correction).

What the experiment adds is the exchange rate: doubling the width from $8d$ to $16d$ moves the affordable level
from $90\%$ to $97\%$, and widening to $256d$ moves it to $99\%$ (before
correction). We test three
levels, so the table brackets the critical sparsity rather than pinning it down,
and at $256d$ the $99\%$ level is still $0.35$ points behind, so the trend may
not have saturated.

\takeaway{The affordable sparsity grows with width, from $90\%$ to $97\%$
(from $m = 64d$ with Holm correction, $16d$ without); a sparse \eas layer should
be wide.}

\section{The findings carry over to an MLP-Mixer and a Transformer}
\label{sec:archs}

Replacing $\relu$ by top-$k$ with $10\%$ of units active in the expanded layers
of two deep architectures leaves quality unchanged within seed-to-seed
variation: over $15$ paired runs (five seeds, three expansion ratios) the mean
change is $-0.04$ accuracy points in an MLP-Mixer on CIFAR-10 ($p = 0.95$) and
$-0.1\%$ perplexity in a Transformer language model ($p = 0.60$)
(\cref{tab:archs}). At $3\%$ active units both models pay a price, and, as in
\cref{sec:res-sparsity}, the price falls as the layer widens: clearly in the
Mixer and in the two larger Transformers, and only slightly in the smallest.

\paragraph{Setup.} We change only the activation of the expanded layer and keep
everything else fixed, including all hyperparameters, which are not tuned for
any arm. Both models are trained from scratch with five seeds, and all arms
share the seeds and the order of training examples, so we compare arms seed by
seed. \emph{MLP-Mixer} \citep{tolstikhin2021mlp} on CIFAR-10: width $128$,
depth $8$, patch size $4$, and expansion ratio $r = m/d \in \{4, 8, 16\}$ for
the channel MLP (the token MLP is widened by the same ratio), trained for $100$
epochs. The arms are $\relu$ in both mixers; top-$k$ at $k/m \in \{0.10, 0.03\}$
in the channel MLP only, or in both mixers; and \texttt{GELU} in both mixers.
\emph{Transformer}: a GPT-style decoder (width $128$, $6$ layers, $4$
heads, context $256$) trained on WikiText-103 for $10{,}000$ steps ($164$M
tokens), with FFN ratio $r \in \{4, 8, 16\}$. The arms are $\relu$, top-$k$ at
$k/m \in \{0.10, 0.03\}$, and \texttt{GELU}. Two larger Transformers
(width $256$; width $512$ with $8$ layers and $51$M to $101$M parameters) check
other sizes with the same ratios, seeds and top-$k$ arms.
\Cref{app:archs-extra} gives the remaining settings and every run. We report the
mean paired difference, call a difference \emph{consistent} when all seeds
agree in sign, and test each arm with an exact Wilcoxon test over its $15$
seed-and-ratio pairs.

\begin{table}[t]
\centering\small
\caption{Replacing $\relu$ in the expanded layers of two deep models trained from
scratch. $\relu$ rows give mean $\pm$ standard deviation over five seeds; other rows give the mean paired change relative
to $\relu$ at the same ratio and seed (accuracy points for the Mixer, relative
change in test perplexity for the Transformer; lower perplexity is better).
$^{\ast}$: all seeds agree in sign. Last column: exact two-sided Wilcoxon test
over the $15$ seed-and-ratio pairs (zero differences dropped). Larger Transformers: change relative to $\relu$ of the same size. Every entry uses the same five seeds. Per-seed values are in \cref{app:archs-extra}.}
\label{tab:archs}
\begin{tabular}{lcccc}
\toprule
Expansion ratio $r = m/d$ & $4$ & $8$ & $16$ & pooled $p$ \\
\midrule
\multicolumn{5}{l}{\emph{MLP-Mixer (width $128$), CIFAR-10: test accuracy (\%)}} \\
$\relu$ (both mixers), absolute & $81.98 \pm 0.44$ & $82.20 \pm 0.85$ & $81.89 \pm 0.55$ & \\
top-$k$ $10\%$, channel MLP & $-0.28$ & $+0.16$ & $+0.01$ & $.952$ \\
top-$k$ $10\%$, both mixers & $+0.05$ & $-0.06$ & $-0.26$ & $.847$ \\
top-$k$ $3\%$, channel MLP & $-3.46^{\ast}$ & $-2.20^{\ast}$ & $-1.77^{\ast}$ & $<.001$ \\
top-$k$ $3\%$, both mixers & $-2.13^{\ast}$ & $-1.49^{\ast}$ & $-1.14$ & $<.001$ \\
\texttt{GELU} (both mixers) & $+0.78^{\ast}$ & $+0.79$ & $+0.97^{\ast}$ & $<.001$ \\
\midrule
\multicolumn{5}{l}{\emph{Transformer (width $128$), WikiText-103: test perplexity}} \\
$\relu$, absolute & $61.35 \pm 0.31$ & $57.84 \pm 0.28$ & $54.55 \pm 0.28$ & \\
top-$k$ $10\%$ & $-0.4\%$ & $-0.1\%$ & $0.0\%$ & $.599$ \\
top-$k$ $3\%$ & $+3.3\%^{\ast}$ & $+3.3\%^{\ast}$ & $+2.9\%^{\ast}$ & $<.001$ \\
\texttt{GELU} & $+2.1\%^{\ast}$ & $+3.9\%^{\ast}$ & $+6.2\%^{\ast}$ & $<.001$ \\
\midrule
\multicolumn{5}{l}{\emph{Larger Transformers, WikiText-103: test perplexity}} \\
width $256$, $\relu$, absolute & $44.70 \pm 0.25$ & $42.06 \pm 0.24$ & $40.27 \pm 0.24$ & \\
width $256$, top-$k$ $10\%$ & $+0.7\%$ & $+0.4\%$ & $-0.6\%$ & $.330$ \\
width $256$, top-$k$ $3\%$ & $+4.6\%^{\ast}$ & $+3.5\%^{\ast}$ & $+1.6\%^{\ast}$ & $<.001$ \\
width $512$, $8$ layers, $\relu$, absolute & $34.04 \pm 0.16$ & $32.57 \pm 0.18$ & $31.38 \pm 0.18$ & \\
width $512$, $8$ layers, top-$k$ $10\%$ & $-0.3\%$ & $-0.5\%$ & $-0.3\%$ & $.003$ \\
width $512$, $8$ layers, top-$k$ $3\%$ & $+2.2\%^{\ast}$ & $+1.2\%^{\ast}$ & $-0.2\%$ & $.005$ \\
\bottomrule
\end{tabular}
\end{table}

\paragraph{Ten percent active units match $\relu$.} With top-$k$ at $10\%$ in
the channel MLP, the Mixer's mean accuracy changes by $-0.28$, $+0.16$ and
$+0.01$ points at $r = 4$, $8$, $16$; in the Transformer, test perplexity
changes by $-0.4\%$, $-0.1\%$ and $0.0\%$. No difference is consistent: at every
ratio the seeds disagree in sign, the pooled tests find nothing ($p = 0.95$ and
$0.60$), and the per-seed differences ($-1.7$ to $+1.2$ points; $-1.2\%$ to
$+1.2\%$) are of the same size as the spread of $\relu$ itself across seeds
(standard deviation up to $0.9$ points). In the larger Transformers the change
is $+0.7\%$, $+0.4\%$ and $-0.6\%$ at width $256$, with no consistent difference
($p = 0.33$), and $-0.3\%$, $-0.5\%$ and $-0.3\%$ at width $512$, where top-$k$ is
slightly ahead ($12$ of $15$ seed pairs, $p = 0.003$). Sparsifying the token
MLP as well gives the same picture ($+0.05$, $-0.06$ and $-0.26$ points;
$p = 0.85$).

\paragraph{Three percent costs, less the wider the layer.} At
$3\%$ active units, top-$k$ in the Mixer's channel MLP and in the Transformer's
FFN is worse than $\relu$ for every one of the $15$ seed pairs ($p < 0.001$ in
both). The Mixer's accuracy cost falls from $3.46$ points at $r = 4$ to $2.20$
and $1.77$ at $r = 8$ and $16$, the trend of \cref{sec:res-sparsity}. The
Transformer's perplexity cost barely moves at width $128$ ($3.3\%$, $3.3\%$ and
$2.9\%$) but falls clearly in the larger models: from $4.6\%$ to $3.5\%$ and
$1.6\%$ at width $256$, and from $2.2\%$ to $1.2\%$ and $-0.2\%$ at width $512$,
where at $r = 16$ the seeds disagree in sign and $3\%$ costs nothing.

\paragraph{Smooth activations, and natural sparsity.} \texttt{GELU} is ahead of $\relu$ in the
Mixer ($+0.78$ to $+0.97$ points) and behind it in the Transformer, by more the
wider the layer ($+2.1\%$, $+3.9\%$ and $+6.2\%$ perplexity at $r = 4$, $8$,
$16$), so we make no general claim about their ranking. The Transformer's $\relu$ layers
become sparse on their own ($13\%$, $9\%$ and $6\%$ of units active at $r = 4$,
$8$, $16$, as in \citet{li2023lazy}), which may be why fixing $10\%$ costs nothing
there. Top-$k$ makes this level fixed and known in advance, which is what sparse
kernels need (\cref{app:sparse-advantages}). 

\takeaway{In an MLP-Mixer and in $\relu$ Transformers of three sizes trained from
scratch, top-$k$ with $10\%$ of units active matches $\relu$ within seed noise, or
is slightly ahead in the largest Transformer; at $3\%$ it costs up to $3.5$ points
or $5\%$ perplexity, less the wider the layer.}

\section{Concluding Remarks}
\label{sec:conc}

We asked what an expand-and-sparsify layer costs and gains when its expansion is
learned at practical widths. Over $121$ tabular tasks and eight expansion
factors, the answer has two parts. Learning the expansion matters most when the
layer is narrow and fades as it widens. Sparsity costs little once the layer is
at least twice as wide as its input: with a tenth of its units active, the layer
is within about $0.7$ accuracy points of a dense $\relu$ layer, and binary sparse
codes lose almost nothing. The affordable sparsity grows with width. At a tenth
of the units, the comparison carries over to an MLP-Mixer and a Transformer
trained from scratch, and the growth with width holds in the Mixer and in the
two larger Transformers. The exact margins depend on
the tasks and on tuning; the trends are the contribution. For practitioners, the
implication is direct: a sparse expanded layer should be learned and wide, and
its sparsity level should be set together with its width.

\paragraph{Limitations.}
\emph{Scale.} The controlled study uses one expanded layer on tabular data, and
the architecture study is small (widths $128$ to $512$, at most $101$M
parameters and $164$M training tokens, one dataset per architecture).
\citet{szatkowski2026universal} find that tolerance to activation sparsity grows
with model size; across our three sizes the cost of $3\%$ sparsity shows no
trend at $r = 4$ ($3.3\%$, $4.6\%$, $2.2\%$) but falls at $r = 16$ ($2.9\%$,
$1.6\%$, $-0.2\%$). \emph{Tuning asymmetry.} $k$ is
selected on validation for the sparse layers only (\cref{app:tuning}); the
fixed-sparsity results of \cref{sec:res-sparsity} remove this advantage.
\emph{Efficiency.} At batch size $1$, a plain top-$k$ FFN is slower than a
dense one for small layers and up to $1.8\times$ (GPU) and $2.5\times$ (CPU)
faster for large ones (\cref{app:bench}); speed in general needs
sparse kernels \citep{liu2023dejavu, liu2025teal}. \emph{Activations.} How $\relu$
compares with \texttt{GELU} and \texttt{SiLU} depends on the scale of the
pre-activations (\cref{app:scale}), so we do not rank them. \emph{Theory.} Our guarantee
(\cref{app:apprx-gen-theory}) covers a random $\rmM$ and binary top-$k$, not a
learned $\rmM$ or value-passing top-$k$. \emph{Coverage.} $k \geq 1$ and $m \leq 30{,}000$ remove some datasets at the
extremes, and a data-supported $\rmM$ repeats rows when $m \gg n$
(\cref{app:counts,app:distinct}).

\clearpage
\section*{Reproducibility statement}

The experimental protocol is given in full in \cref{sec:expsetup}: the dataset
admission criteria, the three constructions of $\rmM$, the activation set, the
expansion and sparsity grids, the train/validation/test splits, the selection of $k$
on validation, and the pairing and significance testing used for every
comparison. \Cref{app:counts} states every case in which a method is absent
from a comparison and why. All comparisons are paired on identical datasets,
expanded dimensions, sparsity levels and train/test splits. The sweep and analysis code, and
the raw result files from which every figure and number in this paper is
computed, are available from the authors on request.
\Cref{app:repro} lists the OpenML ID of every dataset, the preprocessing, the
straight-through estimator and every configuration rerun for this version, and
\cref{app:archs-extra} gives every run of \cref{sec:archs}.

\section*{Ethics statement}

This work uses only publicly available benchmark datasets from OpenML. We collect
no new data and involve no human subjects, and we do not foresee ethical
concerns specific to this study.

\section*{AI use statement}
\label{sec:ai-disclosure}

In this work, we used generative AI tools for writing and editing portions of
the text, for writing the experiment, analysis and plotting code (including the
dataset preprocessing) from which every figure and table is generated, for
proposing experimental designs and diagnostics, for helping interpret results,
and for retrieving and summarizing prior work. We have not used generative AI
tools to generate the research question or the hypotheses, which are the
authors' own, and the remaining categories requiring disclosure are not
applicable to this work.

We have reviewed all AI-assisted work. All code was run by the authors and its
outputs verified; the analysis was cross-checked against a second,
independently written script. Every citation and every attribution of a result
to prior work was checked by the authors against the original paper; two claims
of novelty produced this way were incorrect and were withdrawn following that
check. We take responsibility for the final content of this work, including
text, claims or artifacts produced with the aid of generative AI.

\clearpage
\bibliographystyle{plainnat}
\bibliography{refs,refs_rebuttal}

\begin{thebibliography}{28}
\providecommand{\natexlab}[1]{#1}
\providecommand{\url}[1]{\texttt{#1}}
\expandafter\ifx\csname urlstyle\endcsname\relax
  \providecommand{\doi}[1]{doi: #1}\else
  \providecommand{\doi}{doi: \begingroup \urlstyle{rm}\Url}\fi

\bibitem[Benavoli et~al.(2017)Benavoli, Corani, Dem{\v{s}}ar, and
  Zaffalon]{benavoli2017time}
Alessio Benavoli, Giorgio Corani, Janez Dem{\v{s}}ar, and Marco Zaffalon.
\newblock Time for a change: a tutorial for comparing multiple classifiers
  through {B}ayesian analysis.
\newblock \emph{Journal of Machine Learning Research}, 18\penalty0
  (77):\penalty0 1--36, 2017.

\bibitem[Boucheron et~al.(2005)Boucheron, Bousquot, and
  Lugosi]{boucheron2005_theory_of_classification}
S.~Boucheron, O.~Bousquot, and G.~Lugosi.
\newblock Theory of classification : A survey of some recent advances.
\newblock \emph{ESAIM: Probability and Statistics}, 9:\penalty0 323--375, 2005.

\bibitem[Caron et~al.(2013)Caron, Ruta, Abbott, and Axel]{caron2013}
Sophie J.~C Caron, Vanessa Ruta, L.~F Abbott, and Richard Axel.
\newblock Random convergence of olfactory inputs in the drosophila mushroom
  body.
\newblock \emph{Nature}, 497\penalty0 (7447):\penalty0 113--117, 2013.
\newblock URL \url{https://doi.org/10.1038/nature12063}.

\bibitem[Dasgupta and Tosh(2020)]{dasguptatosh2020}
Sanjoy Dasgupta and Christopher Tosh.
\newblock Expressivity of expand-and-sparsify representations, 2020.
\newblock URL \url{https://arxiv.org/abs/2006.03741}.

\bibitem[Dasgupta et~al.(2017)Dasgupta, Stevens, and
  Navlakha]{dasgupta2017flyhash}
Sanjoy Dasgupta, Charles~F. Stevens, and Saket Navlakha.
\newblock A neural algorithm for a fundamental computing problem.
\newblock \emph{Science}, 358\penalty0 (6364):\penalty0 793--796, 2017.
\newblock \doi{10.1126/science.aam9868}.
\newblock URL \url{https://doi.org/10.1126/science.aam9868}.

\bibitem[Dasgupta et~al.(2018)Dasgupta, Sheehan, Stevens, and
  Navlakha]{dasgupta2018novelty}
Sanjoy Dasgupta, Timothy~C. Sheehan, Charles~F. Stevens, and Saket Navlakha.
\newblock A neural data structure for novelty detection.
\newblock \emph{Proceedings of the National Academy of Sciences}, 115\penalty0
  (51):\penalty0 13093--13098, 2018.
\newblock URL \url{https://doi.org/10.1073/pnas.1814448115}.

\bibitem[Dem{\v{s}}ar(2006)]{demsar2006statistical}
Janez Dem{\v{s}}ar.
\newblock Statistical comparisons of classifiers over multiple data sets.
\newblock \emph{Journal of Machine Learning Research}, 7:\penalty0 1--30, 2006.

\bibitem[Gao et~al.(2025)Gao, Dupr{\'e}~la Tour, Tillman, Goh, Troll, Radford,
  Sutskever, Leike, and Wu]{gao2025scaling}
Leo Gao, Tom Dupr{\'e}~la Tour, Henk Tillman, Gabriel Goh, Rajan Troll, Alec
  Radford, Ilya Sutskever, Jan Leike, and Jeffrey Wu.
\newblock Scaling and evaluating sparse autoencoders.
\newblock In \emph{International Conference on Learning Representations}, 2025.

\bibitem[Grave et~al.(2017)Grave, Joulin, and Usunier]{grave2017improving}
Edouard Grave, Armand Joulin, and Nicolas Usunier.
\newblock Improving neural language models with a continuous cache.
\newblock In \emph{International Conference on Learning Representations}, 2017.
\newblock URL \url{https://openreview.net/forum?id=B184E5qee}.

\bibitem[Gupta et~al.(2021)Gupta, Dar, Goodman, Ciprut, and
  Berant]{gupta2021memory}
Ankit Gupta, Guy Dar, Shaya Goodman, David Ciprut, and Jonathan Berant.
\newblock Memory-efficient transformers via top-k attention.
\newblock In \emph{Proceedings of the Second Workshop on Simple and Efficient
  Natural Language Processing}, pages 39--52, 2021.
\newblock URL \url{https://aclanthology.org/2021.sustainlp-1.5}.

\bibitem[Li et~al.(2023)Li, You, Bhojanapalli, Li, Rawat, Reddi, Ye, Chern, Yu,
  Guo, and Kumar]{li2023lazy}
Zonglin Li, Chong You, Srinadh Bhojanapalli, Daliang Li, Ankit~Singh Rawat,
  Sashank~J. Reddi, Ke~Ye, Felix Chern, Felix Yu, Ruiqi Guo, and Sanjiv Kumar.
\newblock The lazy neuron phenomenon: On emergence of activation sparsity in
  transformers.
\newblock In \emph{International Conference on Learning Representations}, 2023.

\bibitem[Liu et~al.(2025)Liu, Ponnusamy, Cai, Guo, Kim, and
  Athiwaratkun]{liu2025teal}
James Liu, Pragaash Ponnusamy, Tianle Cai, Han Guo, Yoon Kim, and Ben
  Athiwaratkun.
\newblock Training-free activation sparsity in large language models.
\newblock In \emph{International Conference on Learning Representations}, 2025.

\bibitem[Liu et~al.(2023)Liu, Wang, Dao, Zhou, Yuan, Song, Shrivastava, Zhang,
  Tian, R{\'e}, and Chen]{liu2023dejavu}
Zichang Liu, Jue Wang, Tri Dao, Tianyi Zhou, Binhang Yuan, Zhao Song, Anshumali
  Shrivastava, Ce~Zhang, Yuandong Tian, Christopher R{\'e}, and Beidi Chen.
\newblock Deja vu: Contextual sparsity for efficient {LLMs} at inference time.
\newblock In \emph{International Conference on Machine Learning}, pages
  22137--22176, 2023.

\bibitem[Mirzadeh et~al.(2024)Mirzadeh, Alizadeh, Mehta, Del~Mundo, Tuzel,
  Samei, Rastegari, and Farajtabar]{mirzadeh2024relu}
Iman Mirzadeh, Keivan Alizadeh, Sachin Mehta, Carlo~C. Del~Mundo, Oncel Tuzel,
  Golnoosh Samei, Mohammad Rastegari, and Mehrdad Farajtabar.
\newblock {ReLU} strikes back: Exploiting activation sparsity in large language
  models.
\newblock In \emph{International Conference on Learning Representations}, 2024.

\bibitem[Ram and Sinha(2022)]{ram2022flynn}
Parikshit Ram and Kaushik Sinha.
\newblock Federated nearest neighbor classification with a colony of
  fruit-flies.
\newblock In \emph{Proceedings of the AAAI Conference on Artificial
  Intelligence}, volume~36, pages 8097--8105, 2022.
\newblock URL \url{https://doi.org/10.1609/aaai.v36i7.20775}.

\bibitem[Shalev-Shwartz and Ben-David(2014)]{shalev2014understanding}
Shai Shalev-Shwartz and Shai Ben-David.
\newblock \emph{Understanding Machine Learning: From Theory to Algorithms}.
\newblock Cambridge University Press, 2014.
\newblock ISBN 9781107057135.

\bibitem[Sinha(2024)]{sinha2024nonparametric}
Kaushik Sinha.
\newblock Non-parametric classification via expand-and-sparsify representation.
\newblock In \emph{Advances in Neural Information Processing Systems}, pages
  29657--29696, 2024.
\newblock URL \url{https://doi.org/10.52202/079017-0933}.

\bibitem[Sinha and Ram(2021)]{sinha2021fbfc}
Kaushik Sinha and Parikshit Ram.
\newblock Fruit-fly inspired neighborhood encoding for classification.
\newblock In \emph{Proceedings of the 27th ACM SIGKDD Conference on Knowledge
  Discovery \& Data Mining}, pages 1470--1480, 2021.
\newblock \doi{10.1145/3447548.3467246}.
\newblock URL \url{https://doi.org/10.1145/3447548.3467246}.

\bibitem[Sinha and Tosh(2026)]{sinhatosh2026density}
Kaushik Sinha and Christopher Tosh.
\newblock Optimal rates for density and mode estimation with
  expand-and-sparsify representations.
\newblock In \emph{Proceedings of the International Conference on Artificial
  Intelligence and Statistics}, 2026.
\newblock URL \url{https://openreview.net/forum?id=kHuqnMOJwD}.

\bibitem[Srivastava et~al.(2013)Srivastava, Masci, Kazerounian, Gomez, and
  Schmidhuber]{srivastava2013compete}
Rupesh~Kumar Srivastava, Jonathan Masci, Sohrob Kazerounian, Faustino Gomez,
  and J{\"u}rgen Schmidhuber.
\newblock Compete to compute.
\newblock In \emph{Advances in Neural Information Processing Systems}, 2013.
\newblock URL \url{https://dl.acm.org/doi/10.5555/2999792.2999870}.

\bibitem[Szatkowski et~al.(2026)Szatkowski,
  B{\fontencoding{T1}\selectfont\k{e}}dkowski, Devoto, Dubi{\'n}ski, Minervini,
  Pi{\'o}rczy{\'n}ski, Scardapane, and W{\'o}jcik]{szatkowski2026universal}
Filip Szatkowski, Patryk B{\fontencoding{T1}\selectfont\k{e}}dkowski, Alessio
  Devoto, Jan Dubi{\'n}ski, Pasquale Minervini, Miko{\l}aj Pi{\'o}rczy{\'n}ski,
  Simone Scardapane, and Bartosz W{\'o}jcik.
\newblock Universal properties of activation sparsity in modern large language
  models.
\newblock In \emph{International Conference on Learning Representations}, 2026.

\bibitem[Tolstikhin et~al.(2021)Tolstikhin, Houlsby, Kolesnikov, Beyer, Zhai,
  Unterthiner, Yung, Steiner, Keysers, Uszkoreit, et~al.]{tolstikhin2021mlp}
Ilya~O Tolstikhin, Neil Houlsby, Alexander Kolesnikov, Lucas Beyer, Xiaohua
  Zhai, Thomas Unterthiner, Jessica Yung, Andreas Steiner, Daniel Keysers,
  Jakob Uszkoreit, et~al.
\newblock Mlp-mixer: An all-mlp architecture for vision.
\newblock \emph{Advances in neural information processing systems},
  34:\penalty0 24261--24272, 2021.
\newblock URL
  \url{https://proceedings.neurips.cc/paper/2021/hash/cba0a4ee5ccd02fda0fe3f9a3e7b89fe-Abstract.html}.

\bibitem[Turner et~al.(2008)Turner, Bazhenov, and Laurent]{turner2008}
Glenn~C. Turner, Maxim Bazhenov, and Gilles Laurent.
\newblock Olfactory representations by drosophila mushroom body neurons.
\newblock \emph{Journal of Neurophysiology}, 99\penalty0 (2):\penalty0
  734--746, 2008.
\newblock URL \url{https://doi.org/10.1152/jn.01283.2007}.

\bibitem[Vaswani et~al.(2017)Vaswani, Shazeer, Parmar, Uszkoreit, Jones, Gomez,
  Kaiser, and Polosukhin]{vaswani2017attention}
Ashish Vaswani, Noam Shazeer, Niki Parmar, Jakob Uszkoreit, Llion Jones,
  Aidan~N Gomez, {\L}ukasz Kaiser, and Illia Polosukhin.
\newblock Attention is all you need.
\newblock \emph{Advances in neural information processing systems}, 30, 2017.
\newblock URL
  \url{https://proceedings.neurips.cc/paper/2017/hash/3f5ee243547dee91fbd053c1c4a845aa-Abstract.html}.

\bibitem[Wang et~al.(2024)Wang, Ma, Wang, and Wei]{wang2024qsparse}
Hongyu Wang, Shuming Ma, Ruiping Wang, and Furu Wei.
\newblock {Q-Sparse}: All large language models can be fully
  sparsely-activated.
\newblock \emph{arXiv preprint arXiv:2407.10969}, 2024.

\bibitem[Wimbauer et~al.(2024)Wimbauer, Wu, Schoenfeld, Dai, Hou, He,
  Sanakoyeu, Zhang, Tsai, Kohler, Rupprecht, Cremers, Vajda, and
  Wang]{wimbauer2024_cache}
Felix Wimbauer, Bichen Wu, Edgar Schoenfeld, Xiaoliang Dai, Ji~Hou, Zijian He,
  Artsiom Sanakoyeu, Peizhao Zhang, Sam Tsai, Jonas Kohler, Christian
  Rupprecht, Daniel Cremers, Peter Vajda, and Jialiang Wang.
\newblock Cache me if you can: Accelerating diffusion models through block
  caching.
\newblock In \emph{2024 IEEE/CVF Conference on Computer Vision and Pattern
  Recognition (CVPR)}, pages 6211--6220, 2024.
\newblock URL \url{https://doi.org/10.1109/CVPR52733.2024.00594}.

\bibitem[You et~al.(2025)You, Wu, Jia, Chen, Bhojanapalli, Guo, Evci,
  Wassenberg, Netrapalli, Willcock, Subramanian, Chern, Andreev, Pathak, Yu,
  Jain, Culler, Levy, and Kumar]{you2025spark}
Chong You, Kan Wu, Zhipeng Jia, Lin Chen, Srinadh Bhojanapalli, Jiaxian Guo,
  Utku Evci, Jan Wassenberg, Praneeth Netrapalli, Jeremiah~J. Willcock, Suvinay
  Subramanian, Felix Chern, Alek Andreev, Shreya Pathak, Felix Yu, Prateek
  Jain, David~E. Culler, Henry~M. Levy, and Sanjiv Kumar.
\newblock Spark transformer: Reactivating sparsity in {FFN} and attention.
\newblock \emph{arXiv preprint arXiv:2506.06644}, 2025.

\bibitem[Zhang(2004)]{zhang2004_statistical}
Tong Zhang.
\newblock Statistical behavior and consistency of classification methods based
  on convex risk minimization.
\newblock \emph{The Annals of Statistics}, 32\penalty0 (1):\penalty0 56--134,
  Feb 2004.

\end{thebibliography}

\clearpage
\appendix
\crefalias{section}{appendix}\crefalias{subsection}{subappendix}

\section{Related work}
\label{sec:related}
\paragraph{Expand-and-sparsify representations.}
An \eas representation maps $\rvx \in \R^d$ to a much higher-dimensional sparse
code by applying a linear map $\rmM \in \R^{m \times d}$ with $m \gg d$ and
retaining only the $k \ll m$ largest coordinates. The construction formalizes a
circuit in the \emph{Drosophila} olfactory system, in which a $d=50$-dimensional
signal is lifted to roughly $m=2{,}000$ Kenyon cells through a connectivity
pattern that appears random, after which feedback from a single inhibitory
neuron suppresses all but the $5\%$ ($k/m = 0.05$) highest-firing cells 
\citep{dasgupta2017flyhash, dasgupta2018novelty}. Similar motifs are reported
in the olfactory system of rodents, the visual cortex of cats and the
electrosensory system of electric fish. \citet{dasguptatosh2020} give the
framework we adopt: the rows of $\rmM$ are drawn i.i.d.\ from a distribution
$\nu$ over $\R^d$, and $\rmM \rvx$ is sparsified either by $k$-winner-take-all
or by $k$-thresholding. \citet[Theorem 3]{dasguptatosh2020} establishes universal approximation
under $k$-WTA with $\nu$ uniform on $\mathbb{S}^{d-1}$ (the $d$ dimensional unit sphere); the bound
holds for any $k$, takes a simple form when $k \geq c_o(d \log m + \log 1/\delta)$,
and vanishes only if $k/m \to 0$; we restate their results, and the
corresponding lower bounds on $k$ from \citet{sinha2024nonparametric} and
\citet{sinhatosh2026density}, in \cref{app:theory}.

\paragraph{Data-dependent expansions.}
\citet[\S5]{dasguptatosh2020} analyze what they call a data-dependent expansion
mapping, in which the rows of $\rmM$ are drawn i.i.d.\ from a distribution
$\nu$ whose support coincides with the support of the data distribution, so the rows of $\rmM$
take values over the same support as the inputs. Their Theorem 9 gives an
approximation bound under this assumption with two properties the
data-independent bounds lack: $k$ may be taken as small as desired, and there
is no dependence on the ambient dimension $d$, \emph{even in the constants}.
They describe this as ``mild data-dependence'' and note it ``leads to
significantly better approximation'', and their stated open problem is ``more
sophisticated forms of adaptivity, for instance based on a clustering of the
data''.

The subsequent literature nonetheless works in the data-independent setting.
\citet{sinha2021fbfc} and \citet{ram2022flynn} take
$\rmM \in \{0,1\}^{m \times d}$ random with $s \ll d$ nonzero entries per row.
\citet{sinha2024nonparametric} draws the rows i.i.d.\ uniform on $\mathbb{S}^{d-1}$ or
i.i.d.\ Gaussian $\mathcal{N}(\mathbf{0}, \mI_d)$, and closes by noting that ``in both algorithms, the
projection directions are chosen in a data-independent manner'', with
data-dependent choices left to future work. \citet{sinhatosh2026density}
likewise sample the rows from a fixed $\nu$.

\paragraph{Sparsity level and expanded dimension, empirically.}
\citet{ram2022flynn} tune over $70$ OpenML datasets with $k \in [8, 256]$ and
$m \in [2d, 2048d]$, and in a per-parameter study observe that for fixed $k$, increasing $m$ improves accuracy but only up to a bound,
and that increasing $k$ helps ``up to a point after which it can hurt
performance unless $m$ is increased as well'', since it reduces the sparsity of
the code. \Cref{sec:res-sparsity} quantifies the same interaction against a
$\relu$ baseline, which their study does not include.

\paragraph{Competition as an activation function.}
\citet{srivastava2013compete} propose local winner-take-all layers, in which
units are partitioned into blocks and only the maximum within each block
propagates its activation. Their framework permits any block size and any
activation, but in order to study competition in isolation they take the
activation to be the identity and use blocks of size two throughout, which
fixes the sparsity of the layer at exactly $50\%$.
We are explicitly interested in the \eas setting where we first have an expansion to $m$ dimensions followed by a $k$-sparsification with $k \ll m$. Thus, our setting introduces the expansion, and the high sparsity regime, neither of which is considered in \citet{srivastava2013compete}. But we consider $\lwta$ as a different sparsification scheme in addition to $\topk$ considered in existing literature on \eas representations for completeness.

\paragraph{Activation sparsity in Transformers.}
The feedforward blocks of trained Transformers are naturally sparse under
$\relu$: \citet{li2023lazy} report about $3\%$ nonzero activations in T5 and
$6\%$ in ViT, and find top-$k$ Transformers with a reasonable $k$ on par with
dense ones. \citet{mirzadeh2024relu} train models with $\relu$, \texttt{GELU} and
\texttt{SiLU} from scratch and find them to perform similarly, and use the
resulting sparsity for faster inference, as do \citet{liu2023dejavu} through
predicted contextual sparsity and \citet{liu2025teal} without training.
\citet{szatkowski2026universal} measure the \emph{critical sparsity} of modern
LLMs, the highest sparsity with at most a $1\%$ performance drop, and find it
grows with model size; \cref{sec:res-sparsity} applies the same idea to the
width of a single layer. Top-$k$ has also been enforced directly:
\citet{gupta2021memory} in T5, \citet{wang2024qsparse} with a straight-through
estimator in LLMs, and \citet{you2025spark}, who find that a plain top-$k$
operator loses quality in large-scale pretraining unless the architecture is
adapted. Sparse autoencoders use exactly an \eas layer learned by gradient
descent, $\rvz = \topk(\rmM \rvx)$, and \citet{gao2025scaling} find that top-$k$
beats $\relu$ on the sparsity--reconstruction frontier by a margin that grows
with the number of latents, which parallels our finding that wider layers
afford more sparsity. These works study sparsity inside large architectures,
where the activation interacts with the rest of the model; we isolate a single
expanded layer, compare it against dense activations of the same width across
many datasets with paired tests, and then check the conclusions in two deep
models (\cref{sec:archs}).

\paragraph{Comparing models over many datasets.}
We compare methods with paired tests over datasets, following
\citet{demsar2006statistical}, and add equivalence tests (two one-sided tests)
and the Bayesian signed-rank test with a region of practical equivalence of
\citet{benavoli2017time}, so that a claim of parity rests on a margin rather
than on a non-significant test.

\section{What the sparsity is worth in practice}
\label{app:sparse-advantages}

\Cref{sec:results} measures accuracy. This appendix sets out what an
extremely ($k \ll m$) and predictably (always exactly $k$ nonzeros) sparse
expanded representation buys in compute and memory, which is the reason the
accuracy question is worth asking at all.

\paragraph{Loading the next layer's weights.}
With $\rvz \in \R^m$ as a $k$-sparse \eas representation, only the $k$ nonzero coordinates of $\rvz$ contribute to $\rmW \rvz$, so only
the $k$ columns of $\rmW \in \R^{l \times m}$ that multiply them need to be
loaded. For a single input this is a factor $m/k$ fewer columns, which at our loosest setting is $10$ and at our
tightest is $1000$. With a batch of $b$ inputs the union of the active
coordinates is what matters, so at most $bk$ columns are needed. That bound is
attained only when no two inputs in the batch share an active coordinate. In
practice inputs that are close to each other activate overlapping sets, so the
number of columns actually loaded is smaller than $bk$, and how much smaller
depends on how correlated the batch is. Overall, this allows for a reduction of the memory overhead from the naive $O(lm)$ to load the whole $\rmW$ to $O(lbk)$ where we just load the ``batch-activated'' rows of $\rmW$. Sparse linear algebra kernels can also reduce the naive computational complexity of $\rmW \rvz$ from $O(lm)$ to $O(lk)$ for each $\rvz$.

\paragraph{Backpropagation and caching.}
The same argument applies to the backward pass: gradients flow only through
the $k$ nodes that were active, so the memory needed to backpropagate through
the layer scales with $k$ rather than $m$ \citep{gupta2021memory}. Where
intermediate activations are cached rather than recomputed
\citep{grave2017improving, wimbauer2024_cache}, a representation with exactly
$k$ nonzeros can be stored as $k$ index-value pairs instead of $m$ values,
which is a factor $m/2k$ in the value-passing form and $m/k$ indices alone in
the binary form considered in \cref{sec:res-binary}.

\paragraph{Binary codes.}
A binary \eas representation is fixed entirely by which $k$ coordinates are
active, so it can be stored as $k$ indices of $\lceil \log_2 m \rceil$ bits each
and no values; at $m = 30{,}000$ that is $15$ bits per active unit. The next
layer then needs no multiplications, since $\rmW \rvz$ is the sum of the $k$
columns of $\rmW$ that $\rvz$ selects. A binarized $\relu$ layer shares the
second property, but the number of units it activates varies from input to
input, so its storage cannot be fixed in advance. \Cref{sec:res-binary} shows
that on classification \eas gets these savings for less than half an
accuracy point.

\paragraph{Why predictability matters.}
A representation that is merely sparse on average still has to be allocated
for its worst case. Top-$k$ is sparse by construction, with the same count for
every input, so the working set is known ahead of time and does not have to be
sized for an unlucky batch. This is the property that separates \eas from
sparsity that emerges from a $\relu$, and it is why we hold $k$ fixed rather
than letting the layer choose how many units to activate.
\Cref{app:bench} measures what this is worth in time for one FFN at
batch size $1$.

\section{The tuning asymmetry}
\label{app:tuning}

We select $k$ on validation for top-$k$ and LWTA, while $\relu$, GELU and SiLU
have no comparable hyper-parameter, so part of what \cref{sec:res-vs-relu}
compares is the freedom to tune. Two things bound that freedom. Every setting
we allow keeps at most a tenth of the coordinates, so the sparse activations
are under a constraint the dense ones do not face even at their loosest. And
\cref{sec:res-sparsity} repeats the comparison at fixed sparsity levels with no
selection over $k$ at all, where the $90\%$ level is within one
accuracy point of $\relu$ at every expansion factor (TOST). A reader who wants the strictest reading should take that
result rather than \cref{sec:res-vs-relu} as the headline. Validation also
tends to pick the loosest level we allow: the selected $k/m$ is the top of our
grid ($0.10$) in about half of the classification configurations and in $59\%$
to $69\%$ of the regression ones, so the cap $k \leq 0.1m$ binds for the sparse
activations rather than being a formality. The scale of the
pre-activations is a knob that every activation has; \cref{app:scale} tunes it
for all of them and reports how the comparison changes.

\section{Theoretical guarantees and the resulting constraint on $k$}
\label{app:theory}

This appendix collects the results referred to in \cref{sec:related}. Nothing
here is ours (\cref{app:apprx-gen-theory} gives our own guarantee, for a random
$\rmM$ and any $k$); we state it because the constraint it places on $k$ is what
makes our sparsity grid a departure from the theory rather than an instance of it.
Together with our own guarantee, these results show how weak the
guarantees become in the practical regime we study ($m$ a small multiple of $d$,
$k = s \cdot m$), which is where learning $\rmM$ helps most.

\paragraph{Approximation under a data-independent projection.}
\citet{dasguptatosh2020} draw the rows of $\rmM$ 
i.i.d.\ from a distribution $\nu$ over $\R^d$ and sparsify $\rvy = \rmM \rvx$
either by $k$-winner-take-all or by \emph{$k$-thresholding}, in which each
coordinate has its own threshold $\tau_j$ set so that $z_j$ fires a $k/m$
fraction of the time. Their Theorem 3 establishes a universal approximation
property under $k$-WTA with $\nu$ uniform on $\mathbb{S}^{d-1}$: for $\lambda$-Lipschitz
$f: \R^d \to \R$, an explicitly specified linear function of $\rvz = \topk(\rmM \rvx)$, denoted by $\widehat f: \R^d \to \R$, attains
$\sup_\rvx |\widehat f(\rvx) - f(\rvx)| \leq c_d \lambda (4k/m)^{1/(d-1)}$ whenever
$k \geq c_o(d \log m + \log 1/\delta)$. The bound itself holds for any $k$; the
condition only brings it to this form. At a fixed fraction $k/m$ it does not
shrink as $m$ grows. Their Theorem 4 shows that this scheme
is \emph{not} adaptive to manifold structure, since on a one-dimensional circle
in $\mathbb{S}^{d-1}$ most $z_j$ are never activated, while their Theorem 8 shows
that $k$-thresholding is.

\paragraph{Approximation under a data-supported projection.}
Theorem 9 of \citet{dasguptatosh2020} gives an optimal $k$ of $d_o/2$ under
their data-dependent sampling, an absolute quantity fixed by the intrinsic
dimension $d_o$. We do not use this as a prescription: $d_o$ is not known in
practice, and parameterizing by $k = s \cdot m$ has the practical advantage of
controlling the sparsity of the representation directly, which is the quantity
we want to vary.

\paragraph{The lower bounds on $k$.}
The guarantees above constrain $k$ from below.
\citet{dasguptatosh2020} use $k \geq c_o(d \log m + \log 1/\delta)$ to simplify
their cell-diameter bound; Theorem 3.12 of \citet{sinha2024nonparametric}
carries the same condition, and its empirical section accordingly sets
$k = d \log m$; Lemma 5.2 of \citet{sinhatosh2026density} requires
$k > C_\delta^2 \, d \log m$. These are lower bounds ensuring the response
regions are well estimated, not prescriptions for an optimal $k$, and they sit
inside asymptotic regimes far from what is reachable in practice: Corollary
3.13 of \citet{sinha2024nonparametric} sets $m = k\,n^{(d-1)/(2\beta+d-1)}$,
which is exponentially large in $d$. We therefore parameterize by
$k = s \cdot m$ and report results as functions of $s$.
Since our $m$ is not exponential in $d$, and at the fractions $k \leq 0.1m$ we
use these bounds do not vanish as $m$ grows (\cref{app:apprx-gen-theory}), our
results should be read as evidence about that regime rather than as a test of
the theory.

\section{Wins, ties and losses at every expansion factor}
\label{app:tables}
\suppressfloats[t]
The tables below give every comparison of \cref{sec:results} at every
expansion factor: \cref{tab:wtlfull} the classification comparisons against $\relu$ summarized in
\cref{tab:wtl}, \cref{tab:wtlreg} the same comparisons and the binary ones on the
regression tasks, \cref{tab:wtlbin} the binary codes on classification, and
\cref{tab:wtlsparse} the fixed sparsity levels of \cref{sec:res-sparsity},
and \cref{tab:equiv} the confidence interval and equivalence margin of every
comparison. All tables are computed by one script from the raw result files,
including the configurations rerun for this version.

\begin{table}[htb]
\caption{\Cref{tab:wtl} at every expansion factor. Each cell gives wins/ties/losses and, below it, the Wilcoxon $p$-value; a star marks $p<0.05$.}
\label{tab:wtlfull}
\centering
\scriptsize
\setlength{\tabcolsep}{3pt}
\begin{tabular}{lcccccccc}
\toprule
activation & $m/d=1$ & $m/d=2$ & $m/d=4$ & $m/d=8$ & $m/d=16$ & $m/d=64$ & $m/d=128$ & $m/d=256$ \\
\midrule
top-$k$ & \shortstack{18/16/31$^{\ast}$\\ (.035)} & \shortstack{23/13/36\\ (.237)} & \shortstack{27/10/35\\ (.576)} & \shortstack{22/12/37\\ (.178)} & \shortstack{24/12/35\\ (.727)} & \shortstack{27/9/35\\ (.800)} & \shortstack{30/12/28\\ (.574)} & \shortstack{31/9/27\\ (.321)} \\
LWTA & \shortstack{20/16/29\\ (.145)} & \shortstack{21/15/36\\ (.112)} & \shortstack{25/9/38\\ (.380)} & \shortstack{24/8/39\\ (.105)} & \shortstack{25/11/35\\ (.373)} & \shortstack{21/16/34\\ (.725)} & \shortstack{31/8/31\\ (.757)} & \shortstack{28/13/26\\ (.616)} \\
GELU & \shortstack{27/3/35\\ (.497)} & \shortstack{33/5/34\\ (.861)} & \shortstack{24/4/44$^{\ast}$\\ (.038)} & \shortstack{22/7/42$^{\ast}$\\ (.014)} & \shortstack{19/10/42$^{\ast}$\\ (.002)} & \shortstack{21/8/42$^{\ast}$\\ (.005)} & \shortstack{23/11/36$^{\ast}$\\ (.044)} & \shortstack{22/8/37\\ (.101)} \\
SiLU & \shortstack{23/6/36\\ (.085)} & \shortstack{30/3/39\\ (.328)} & \shortstack{22/4/46$^{\ast}$\\ (.003)} & \shortstack{20/4/47$^{\ast}$\\ (.002)} & \shortstack{15/6/50$^{\ast}$\\ ($<$.001)} & \shortstack{21/6/44$^{\ast}$\\ (.002)} & \shortstack{17/8/45$^{\ast}$\\ (.003)} & \shortstack{20/11/36$^{\ast}$\\ (.007)} \\
\bottomrule
\end{tabular}
\end{table}

\begin{table}[htb]
\caption{Binary codes (\cref{sec:res-binary}). The first two rows compare each binarized sparse code against the binarized dense layer; the last three give what binarizing costs each activation, as its binary form minus its own value-passing form. Each cell gives the median difference in accuracy points, wins/ties/losses over datasets and the two-sided Wilcoxon $p$-value; a star marks $p<0.05$. $M$ is learned on both sides.}
\label{tab:wtlbin}
\centering
\scriptsize
\setlength{\tabcolsep}{2.5pt}
\begin{tabular}{lcccccccc}
\toprule
comparison & $m/d=1$ & $m/d=2$ & $m/d=4$ & $m/d=8$ & $m/d=16$ & $m/d=64$ & $m/d=128$ & $m/d=256$ \\
\midrule
top-$k$ bin.\ $-$ $\relu$ bin. & \begin{tabular}[c]{@{}c@{}}$0.00$\\ 29/13/23\\ (.979)\end{tabular} & \begin{tabular}[c]{@{}c@{}}$0.00$\\ 27/11/34\\ (.845)\end{tabular} & \begin{tabular}[c]{@{}c@{}}$0.00$\\ 30/8/34\\ (.273)\end{tabular} & \begin{tabular}[c]{@{}c@{}}$+0.37$\\ 38/5/28\\ (.274)\end{tabular} & \begin{tabular}[c]{@{}c@{}}$+0.58$\\ 40/4/27$^{\ast}$\\ (.043)\end{tabular} & \begin{tabular}[c]{@{}c@{}}$+0.56$\\ 44/5/22$^{\ast}$\\ (.023)\end{tabular} & \begin{tabular}[c]{@{}c@{}}$+0.62$\\ 48/5/17$^{\ast}$\\ ($<$.001)\end{tabular} & \begin{tabular}[c]{@{}c@{}}$+0.50$\\ 43/2/22$^{\ast}$\\ (.002)\end{tabular} \\ \addlinespace[4pt]
LWTA bin.\ $-$ $\relu$ bin. & \begin{tabular}[c]{@{}c@{}}$0.00$\\ 25/12/28\\ (.468)\end{tabular} & \begin{tabular}[c]{@{}c@{}}$0.00$\\ 29/10/33\\ (.661)\end{tabular} & \begin{tabular}[c]{@{}c@{}}$0.00$\\ 30/8/34\\ (.466)\end{tabular} & \begin{tabular}[c]{@{}c@{}}$0.00$\\ 31/9/31\\ (.709)\end{tabular} & \begin{tabular}[c]{@{}c@{}}$+0.26$\\ 36/9/26\\ (.189)\end{tabular} & \begin{tabular}[c]{@{}c@{}}$+0.52$\\ 44/8/19$^{\ast}$\\ (.011)\end{tabular} & \begin{tabular}[c]{@{}c@{}}$+0.42$\\ 42/5/23$^{\ast}$\\ (.034)\end{tabular} & \begin{tabular}[c]{@{}c@{}}$+0.52$\\ 44/3/20$^{\ast}$\\ (.005)\end{tabular} \\ \addlinespace[4pt]
$\relu$ bin.\ $-$ $\relu$ & \begin{tabular}[c]{@{}c@{}}$0.00$\\ 21/17/35\\ (.226)\end{tabular} & \begin{tabular}[c]{@{}c@{}}$0.00$\\ 27/14/31\\ (.860)\end{tabular} & \begin{tabular}[c]{@{}c@{}}$-0.05$\\ 27/9/36\\ (.669)\end{tabular} & \begin{tabular}[c]{@{}c@{}}$-0.12$\\ 26/9/36\\ (.291)\end{tabular} & \begin{tabular}[c]{@{}c@{}}$0.00$\\ 28/9/34\\ (.356)\end{tabular} & \begin{tabular}[c]{@{}c@{}}$-0.29$\\ 29/6/36\\ (.703)\end{tabular} & \begin{tabular}[c]{@{}c@{}}$-0.43$\\ 21/10/39$^{\ast}$\\ (.027)\end{tabular} & \begin{tabular}[c]{@{}c@{}}$-0.39$\\ 16/8/43$^{\ast}$\\ (.004)\end{tabular} \\ \addlinespace[4pt]
top-$k$ bin.\ $-$ top-$k$ & \begin{tabular}[c]{@{}c@{}}$0.00$\\ 23/18/24\\ (.743)\end{tabular} & \begin{tabular}[c]{@{}c@{}}$-0.02$\\ 23/15/34\\ (.863)\end{tabular} & \begin{tabular}[c]{@{}c@{}}$0.00$\\ 27/11/34\\ (.545)\end{tabular} & \begin{tabular}[c]{@{}c@{}}$0.00$\\ 28/12/31\\ (.987)\end{tabular} & \begin{tabular}[c]{@{}c@{}}$0.00$\\ 30/11/30\\ (.506)\end{tabular} & \begin{tabular}[c]{@{}c@{}}$0.00$\\ 33/9/29\\ (.229)\end{tabular} & \begin{tabular}[c]{@{}c@{}}$0.00$\\ 30/13/27\\ (.565)\end{tabular} & \begin{tabular}[c]{@{}c@{}}$0.00$\\ 25/16/26\\ (.978)\end{tabular} \\ \addlinespace[4pt]
LWTA bin.\ $-$ LWTA & \begin{tabular}[c]{@{}c@{}}$-0.25$\\ 13/19/33$^{\ast}$\\ (.012)\end{tabular} & \begin{tabular}[c]{@{}c@{}}$0.00$\\ 26/14/32\\ (.855)\end{tabular} & \begin{tabular}[c]{@{}c@{}}$0.00$\\ 31/9/32\\ (.755)\end{tabular} & \begin{tabular}[c]{@{}c@{}}$0.00$\\ 26/16/29\\ (.666)\end{tabular} & \begin{tabular}[c]{@{}c@{}}$0.00$\\ 30/14/27\\ (.438)\end{tabular} & \begin{tabular}[c]{@{}c@{}}$+0.13$\\ 36/16/19$^{\ast}$\\ (.025)\end{tabular} & \begin{tabular}[c]{@{}c@{}}$0.00$\\ 25/18/27\\ (.779)\end{tabular} & \begin{tabular}[c]{@{}c@{}}$0.00$\\ 29/12/26\\ (.476)\end{tabular} \\ \addlinespace[4pt]
\bottomrule
\end{tabular}
\end{table}

\begin{table}[htb]
\caption{The comparisons of \cref{tab:wtlfull} and \cref{tab:wtlbin} on the $48$ regression tasks, where the metric is $R^2$, reported raw, with $k$ chosen on the raw validation score. Each cell gives the median difference, wins/ties/losses and the Wilcoxon $p$-value; a star marks $p<0.05$. The top block scores every activation on the same datasets within each expansion factor (41, 48, 48, 48, 48, 47, 47, 46 datasets at $m/d = 1, 2, 4, 8, 16, 64, 128, 256$); the bottom block pairs each comparison on its own datasets. With $R^2$ floored at zero every verdict is the same except: SiLU at $m=256d$ (significant when floored); $\relu$ bin.\ $-$ $\relu$ at $m=256d$ (not significant when floored).}
\label{tab:wtlreg}
\centering
\scriptsize
\setlength{\tabcolsep}{2.5pt}
\begin{tabular}{lcccccccc}
\toprule
comparison & $m/d=1$ & $m/d=2$ & $m/d=4$ & $m/d=8$ & $m/d=16$ & $m/d=64$ & $m/d=128$ & $m/d=256$ \\
\midrule
top-$k$ & \begin{tabular}[c]{@{}c@{}}$-0.044$\\ 7/0/34$^{\ast}$\\ ($<$.001)\end{tabular} & \begin{tabular}[c]{@{}c@{}}$-0.032$\\ 9/1/38$^{\ast}$\\ ($<$.001)\end{tabular} & \begin{tabular}[c]{@{}c@{}}$-0.010$\\ 16/1/31\\ (.056)\end{tabular} & \begin{tabular}[c]{@{}c@{}}$-0.014$\\ 15/1/32\\ (.070)\end{tabular} & \begin{tabular}[c]{@{}c@{}}$-0.007$\\ 15/1/32\\ (.332)\end{tabular} & \begin{tabular}[c]{@{}c@{}}$-0.001$\\ 22/1/24\\ (.747)\end{tabular} & \begin{tabular}[c]{@{}c@{}}$-0.001$\\ 21/3/23\\ (.522)\end{tabular} & \begin{tabular}[c]{@{}c@{}}$0.002$\\ 26/3/17\\ (.186)\end{tabular} \\
\addlinespace[4pt]
LWTA & \begin{tabular}[c]{@{}c@{}}$-0.039$\\ 7/0/34$^{\ast}$\\ ($<$.001)\end{tabular} & \begin{tabular}[c]{@{}c@{}}$-0.030$\\ 9/0/39$^{\ast}$\\ ($<$.001)\end{tabular} & \begin{tabular}[c]{@{}c@{}}$-0.014$\\ 14/1/33$^{\ast}$\\ (.010)\end{tabular} & \begin{tabular}[c]{@{}c@{}}$-0.011$\\ 15/1/32\\ (.082)\end{tabular} & \begin{tabular}[c]{@{}c@{}}$-0.007$\\ 15/2/31\\ (.120)\end{tabular} & \begin{tabular}[c]{@{}c@{}}$0.001$\\ 23/3/21\\ (.861)\end{tabular} & \begin{tabular}[c]{@{}c@{}}$0.001$\\ 23/4/20\\ (.273)\end{tabular} & \begin{tabular}[c]{@{}c@{}}$0.002$\\ 25/1/20\\ (.251)\end{tabular} \\
\addlinespace[4pt]
GELU & \begin{tabular}[c]{@{}c@{}}$-0.007$\\ 14/2/25$^{\ast}$\\ (.016)\end{tabular} & \begin{tabular}[c]{@{}c@{}}$-0.010$\\ 16/0/32$^{\ast}$\\ (.010)\end{tabular} & \begin{tabular}[c]{@{}c@{}}$-0.009$\\ 16/4/28$^{\ast}$\\ (.032)\end{tabular} & \begin{tabular}[c]{@{}c@{}}$-0.003$\\ 16/2/30$^{\ast}$\\ (.025)\end{tabular} & \begin{tabular}[c]{@{}c@{}}$-0.004$\\ 19/2/27\\ (.155)\end{tabular} & \begin{tabular}[c]{@{}c@{}}$-0.001$\\ 20/3/24\\ (.214)\end{tabular} & \begin{tabular}[c]{@{}c@{}}$-0.001$\\ 19/4/24\\ (.210)\end{tabular} & \begin{tabular}[c]{@{}c@{}}$-0.006$\\ 15/1/30$^{\ast}$\\ (.024)\end{tabular} \\
\addlinespace[4pt]
SiLU & \begin{tabular}[c]{@{}c@{}}$-0.012$\\ 9/2/30$^{\ast}$\\ (.001)\end{tabular} & \begin{tabular}[c]{@{}c@{}}$-0.014$\\ 15/1/32$^{\ast}$\\ ($<$.001)\end{tabular} & \begin{tabular}[c]{@{}c@{}}$-0.010$\\ 12/3/33$^{\ast}$\\ (.002)\end{tabular} & \begin{tabular}[c]{@{}c@{}}$-0.009$\\ 15/3/30$^{\ast}$\\ (.003)\end{tabular} & \begin{tabular}[c]{@{}c@{}}$-0.009$\\ 15/3/30$^{\ast}$\\ (.017)\end{tabular} & \begin{tabular}[c]{@{}c@{}}$-0.003$\\ 15/4/28$^{\ast}$\\ (.030)\end{tabular} & \begin{tabular}[c]{@{}c@{}}$-0.001$\\ 21/2/24\\ (.218)\end{tabular} & \begin{tabular}[c]{@{}c@{}}$-0.003$\\ 15/1/30\\ (.080)\end{tabular} \\
\addlinespace[4pt]
\midrule
top-$k$ bin.\ $-$ $\relu$ bin. & \begin{tabular}[c]{@{}c@{}}$-0.041$\\ 12/0/29$^{\ast}$\\ (.043)\end{tabular} & \begin{tabular}[c]{@{}c@{}}$-0.008$\\ 21/0/27\\ (.288)\end{tabular} & \begin{tabular}[c]{@{}c@{}}$0.010$\\ 28/1/19\\ (.136)\end{tabular} & \begin{tabular}[c]{@{}c@{}}$0.009$\\ 28/0/20\\ (.102)\end{tabular} & \begin{tabular}[c]{@{}c@{}}$0.014$\\ 30/0/18$^{\ast}$\\ (.030)\end{tabular} & \begin{tabular}[c]{@{}c@{}}$0.028$\\ 35/0/12$^{\ast}$\\ (.001)\end{tabular} & \begin{tabular}[c]{@{}c@{}}$0.017$\\ 29/3/15$^{\ast}$\\ (.009)\end{tabular} & \begin{tabular}[c]{@{}c@{}}$0.016$\\ 26/0/20\\ (.068)\end{tabular} \\
\addlinespace[4pt]
LWTA bin.\ $-$ $\relu$ bin. & \begin{tabular}[c]{@{}c@{}}$-0.042$\\ 11/0/30$^{\ast}$\\ (.014)\end{tabular} & \begin{tabular}[c]{@{}c@{}}$-0.012$\\ 19/2/27\\ (.123)\end{tabular} & \begin{tabular}[c]{@{}c@{}}$0.006$\\ 25/0/23\\ (.555)\end{tabular} & \begin{tabular}[c]{@{}c@{}}$-0.003$\\ 23/1/24\\ (.397)\end{tabular} & \begin{tabular}[c]{@{}c@{}}$0.009$\\ 25/2/21\\ (.165)\end{tabular} & \begin{tabular}[c]{@{}c@{}}$0.027$\\ 30/2/15$^{\ast}$\\ (.012)\end{tabular} & \begin{tabular}[c]{@{}c@{}}$0.009$\\ 27/2/18$^{\ast}$\\ (.046)\end{tabular} & \begin{tabular}[c]{@{}c@{}}$0.012$\\ 28/0/18\\ (.082)\end{tabular} \\
\addlinespace[4pt]
$\relu$ bin.\ $-$ $\relu$ & \begin{tabular}[c]{@{}c@{}}$-0.068$\\ 7/0/41$^{\ast}$\\ ($<$.001)\end{tabular} & \begin{tabular}[c]{@{}c@{}}$-0.051$\\ 5/0/43$^{\ast}$\\ ($<$.001)\end{tabular} & \begin{tabular}[c]{@{}c@{}}$-0.045$\\ 9/0/39$^{\ast}$\\ ($<$.001)\end{tabular} & \begin{tabular}[c]{@{}c@{}}$-0.023$\\ 11/0/37$^{\ast}$\\ ($<$.001)\end{tabular} & \begin{tabular}[c]{@{}c@{}}$-0.025$\\ 16/1/31$^{\ast}$\\ (.002)\end{tabular} & \begin{tabular}[c]{@{}c@{}}$-0.024$\\ 13/0/34$^{\ast}$\\ (.003)\end{tabular} & \begin{tabular}[c]{@{}c@{}}$-0.021$\\ 18/0/29$^{\ast}$\\ (.012)\end{tabular} & \begin{tabular}[c]{@{}c@{}}$-0.015$\\ 16/0/30$^{\ast}$\\ (.027)\end{tabular} \\
\addlinespace[4pt]
top-$k$ bin.\ $-$ top-$k$ & \begin{tabular}[c]{@{}c@{}}$-0.030$\\ 8/1/32$^{\ast}$\\ ($<$.001)\end{tabular} & \begin{tabular}[c]{@{}c@{}}$-0.038$\\ 7/0/41$^{\ast}$\\ ($<$.001)\end{tabular} & \begin{tabular}[c]{@{}c@{}}$-0.028$\\ 6/1/41$^{\ast}$\\ ($<$.001)\end{tabular} & \begin{tabular}[c]{@{}c@{}}$-0.021$\\ 10/0/38$^{\ast}$\\ ($<$.001)\end{tabular} & \begin{tabular}[c]{@{}c@{}}$-0.022$\\ 4/2/42$^{\ast}$\\ ($<$.001)\end{tabular} & \begin{tabular}[c]{@{}c@{}}$-0.011$\\ 10/3/34$^{\ast}$\\ (.001)\end{tabular} & \begin{tabular}[c]{@{}c@{}}$-0.010$\\ 8/3/36$^{\ast}$\\ ($<$.001)\end{tabular} & \begin{tabular}[c]{@{}c@{}}$-0.013$\\ 5/2/39$^{\ast}$\\ ($<$.001)\end{tabular} \\
\addlinespace[4pt]
LWTA bin.\ $-$ LWTA & \begin{tabular}[c]{@{}c@{}}$-0.048$\\ 6/1/34$^{\ast}$\\ ($<$.001)\end{tabular} & \begin{tabular}[c]{@{}c@{}}$-0.037$\\ 7/1/40$^{\ast}$\\ ($<$.001)\end{tabular} & \begin{tabular}[c]{@{}c@{}}$-0.032$\\ 6/2/40$^{\ast}$\\ ($<$.001)\end{tabular} & \begin{tabular}[c]{@{}c@{}}$-0.025$\\ 6/1/41$^{\ast}$\\ ($<$.001)\end{tabular} & \begin{tabular}[c]{@{}c@{}}$-0.021$\\ 9/0/39$^{\ast}$\\ ($<$.001)\end{tabular} & \begin{tabular}[c]{@{}c@{}}$-0.016$\\ 9/0/38$^{\ast}$\\ ($<$.001)\end{tabular} & \begin{tabular}[c]{@{}c@{}}$-0.013$\\ 7/0/40$^{\ast}$\\ ($<$.001)\end{tabular} & \begin{tabular}[c]{@{}c@{}}$-0.011$\\ 8/4/34$^{\ast}$\\ ($<$.001)\end{tabular} \\
\addlinespace[4pt]
\bottomrule
\end{tabular}
\end{table}

\begin{table}[htb]
\caption{Fixed sparsity levels (\cref{sec:res-sparsity}): top-$k$ at a fixed $k/m$, with no selection over $k$, minus $\relu$ at the same expanded dimension. A configuration counts at a level when its realized $k/m$ is within $\pm 25\%$ of the level and at most $0.105$; if several $k$ qualify, the one with the best validation score is used. Each cell gives the median difference in accuracy points, wins/ties/losses over datasets, the two-sided Wilcoxon $p$-value (a star marks $p<0.05$), and $\delta^\ast$, the smallest equivalence margin supported by TOST; $\delta^\ast$ is in bold when the level is within one point of $\relu$ (TOST at $\delta = 1$, $p < 0.05$). $M$ is learned on both sides.}
\label{tab:wtlsparse}
\centering
\scriptsize
\setlength{\tabcolsep}{2.5pt}
\begin{tabular}{lcccccccc}
\toprule
comparison & $m/d=1$ & $m/d=2$ & $m/d=4$ & $m/d=8$ & $m/d=16$ & $m/d=64$ & $m/d=128$ & $m/d=256$ \\
\midrule
$90\%$ ($k/m=0.1$) & \begin{tabular}[c]{@{}c@{}}$0.00$\\ 13/16/19\\ (.486)\\ \textbf{0.82}\end{tabular} & \begin{tabular}[c]{@{}c@{}}$0.00$\\ 20/14/27\\ (.758)\\ \textbf{0.37}\end{tabular} & \begin{tabular}[c]{@{}c@{}}$0.00$\\ 23/13/30\\ (.525)\\ \textbf{0.74}\end{tabular} & \begin{tabular}[c]{@{}c@{}}$-0.01$\\ 22/14/32\\ (.496)\\ \textbf{0.42}\end{tabular} & \begin{tabular}[c]{@{}c@{}}$0.00$\\ 21/18/32\\ (.324)\\ \textbf{0.42}\end{tabular} & \begin{tabular}[c]{@{}c@{}}$-0.01$\\ 27/12/32\\ (.592)\\ \textbf{0.28}\end{tabular} & \begin{tabular}[c]{@{}c@{}}$0.00$\\ 32/12/26\\ (.324)\\ \textbf{0.41}\end{tabular} & \begin{tabular}[c]{@{}c@{}}$+0.07$\\ 31/9/27\\ (.211)\\ \textbf{0.54}\end{tabular} \\ \addlinespace[4pt]
$97\%$ ($k/m=0.03$) & \begin{tabular}[c]{@{}c@{}}$-0.26$\\ 10/5/15\\ (.236)\\ 1.87\end{tabular} & \begin{tabular}[c]{@{}c@{}}$-0.36$\\ 13/11/24\\ (.111)\\ 1.36\end{tabular} & \begin{tabular}[c]{@{}c@{}}$-0.59$\\ 16/12/38$^{\ast}$\\ (.002)\\ 1.76\end{tabular} & \begin{tabular}[c]{@{}c@{}}$-0.53$\\ 18/9/44$^{\ast}$\\ (.003)\\ 1.30\end{tabular} & \begin{tabular}[c]{@{}c@{}}$-0.33$\\ 22/7/42$^{\ast}$\\ (.033)\\ \textbf{0.94}\end{tabular} & \begin{tabular}[c]{@{}c@{}}$-0.14$\\ 22/12/37\\ (.170)\\ \textbf{0.59}\end{tabular} & \begin{tabular}[c]{@{}c@{}}$-0.35$\\ 25/7/38\\ (.229)\\ \textbf{0.54}\end{tabular} & \begin{tabular}[c]{@{}c@{}}$-0.01$\\ 27/8/32\\ (.434)\\ \textbf{0.45}\end{tabular} \\ \addlinespace[4pt]
$99\%$ ($k/m=0.01$) & \begin{tabular}[c]{@{}c@{}}$-3.28$\\ 1/2/5\\ (.076)\\ 13.29\end{tabular} & \begin{tabular}[c]{@{}c@{}}$0.00$\\ 6/7/10\\ (.118)\\ 2.04\end{tabular} & \begin{tabular}[c]{@{}c@{}}$-2.42$\\ 6/7/28$^{\ast}$\\ ($<$.001)\\ 4.01\end{tabular} & \begin{tabular}[c]{@{}c@{}}$-2.06$\\ 5/7/41$^{\ast}$\\ ($<$.001)\\ 3.14\end{tabular} & \begin{tabular}[c]{@{}c@{}}$-2.00$\\ 5/6/54$^{\ast}$\\ ($<$.001)\\ 3.15\end{tabular} & \begin{tabular}[c]{@{}c@{}}$-0.99$\\ 16/5/50$^{\ast}$\\ ($<$.001)\\ 1.54\end{tabular} & \begin{tabular}[c]{@{}c@{}}$-0.61$\\ 16/9/45$^{\ast}$\\ (.008)\\ 1.13\end{tabular} & \begin{tabular}[c]{@{}c@{}}$-0.35$\\ 18/10/39$^{\ast}$\\ (.039)\\ \textbf{0.68}\end{tabular} \\ \addlinespace[4pt]
\bottomrule
\end{tabular}
\end{table}

\begin{table}[htb]
\caption{Uncertainty and equivalence for the comparisons of \cref{tab:wtlfull,tab:wtlbin,tab:wtlreg}, on the same datasets. Each cell gives the median paired difference, its exact distribution-free $95\%$ confidence interval (order statistics), and $\delta^\ast$, the smallest margin $\delta$ for which two one-sided Wilcoxon tests reject a location of at least $\delta$ in absolute value at level $0.05$ (TOST). A comparison is equivalent within a tolerance $\delta$ whenever $\delta \geq \delta^\ast$. Classification differences are in accuracy points, regression differences in $R^2$.}
\label{tab:equiv}
\centering
\scriptsize
\setlength{\tabcolsep}{2pt}
\resizebox{\textwidth}{!}{%
\begin{tabular}{lcccccccc}
\toprule
comparison & $m/d=1$ & $m/d=2$ & $m/d=4$ & $m/d=8$ & $m/d=16$ & $m/d=64$ & $m/d=128$ & $m/d=256$ \\
\midrule
\multicolumn{9}{l}{\emph{classification, accuracy points}} \\
top-$k$ $-$ $\relu$ & \begin{tabular}[c]{@{}c@{}}$0.00$\\ $[-1.35,0.00]$\\ $1.53$\end{tabular} & \begin{tabular}[c]{@{}c@{}}$-0.06$\\ $[-0.78,0.00]$\\ $0.68$\end{tabular} & \begin{tabular}[c]{@{}c@{}}$-0.05$\\ $[-0.48,+0.04]$\\ $0.70$\end{tabular} & \begin{tabular}[c]{@{}c@{}}$-0.22$\\ $[-0.58,0.00]$\\ $0.59$\end{tabular} & \begin{tabular}[c]{@{}c@{}}$-0.05$\\ $[-0.45,+0.04]$\\ $0.42$\end{tabular} & \begin{tabular}[c]{@{}c@{}}$-0.01$\\ $[-0.33,+0.19]$\\ $0.31$\end{tabular} & \begin{tabular}[c]{@{}c@{}}$0.00$\\ $[-0.17,+0.33]$\\ $0.37$\end{tabular} & \begin{tabular}[c]{@{}c@{}}$+0.05$\\ $[-0.16,+0.33]$\\ $0.51$\end{tabular} \\ \addlinespace[3pt]
LWTA $-$ $\relu$ & \begin{tabular}[c]{@{}c@{}}$0.00$\\ $[-0.47,0.00]$\\ $1.09$\end{tabular} & \begin{tabular}[c]{@{}c@{}}$-0.11$\\ $[-0.67,0.00]$\\ $0.76$\end{tabular} & \begin{tabular}[c]{@{}c@{}}$-0.23$\\ $[-0.67,+0.04]$\\ $0.70$\end{tabular} & \begin{tabular}[c]{@{}c@{}}$-0.18$\\ $[-0.59,+0.03]$\\ $0.75$\end{tabular} & \begin{tabular}[c]{@{}c@{}}$-0.09$\\ $[-0.42,0.00]$\\ $0.48$\end{tabular} & \begin{tabular}[c]{@{}c@{}}$-0.01$\\ $[-0.25,+0.02]$\\ $0.26$\end{tabular} & \begin{tabular}[c]{@{}c@{}}$0.00$\\ $[-0.31,+0.29]$\\ $0.39$\end{tabular} & \begin{tabular}[c]{@{}c@{}}$0.00$\\ $[-0.15,+0.18]$\\ $0.30$\end{tabular} \\ \addlinespace[3pt]
\texttt{GELU} $-$ $\relu$ & \begin{tabular}[c]{@{}c@{}}$-0.31$\\ $[-1.17,+0.30]$\\ $1.12$\end{tabular} & \begin{tabular}[c]{@{}c@{}}$-0.05$\\ $[-0.71,+0.67]$\\ $0.88$\end{tabular} & \begin{tabular}[c]{@{}c@{}}$-0.36$\\ $[-1.08,-0.02]$\\ $1.29$\end{tabular} & \begin{tabular}[c]{@{}c@{}}$-0.38$\\ $[-0.72,0.00]$\\ $0.79$\end{tabular} & \begin{tabular}[c]{@{}c@{}}$-0.44$\\ $[-1.04,-0.02]$\\ $1.29$\end{tabular} & \begin{tabular}[c]{@{}c@{}}$-0.31$\\ $[-0.90,0.00]$\\ $0.86$\end{tabular} & \begin{tabular}[c]{@{}c@{}}$-0.16$\\ $[-0.56,0.00]$\\ $0.68$\end{tabular} & \begin{tabular}[c]{@{}c@{}}$-0.24$\\ $[-0.43,0.00]$\\ $0.66$\end{tabular} \\ \addlinespace[3pt]
\texttt{SiLU} $-$ $\relu$ & \begin{tabular}[c]{@{}c@{}}$-0.56$\\ $[-2.58,0.00]$\\ $1.78$\end{tabular} & \begin{tabular}[c]{@{}c@{}}$-0.40$\\ $[-1.32,+0.30]$\\ $1.28$\end{tabular} & \begin{tabular}[c]{@{}c@{}}$-0.87$\\ $[-1.83,-0.22]$\\ $2.10$\end{tabular} & \begin{tabular}[c]{@{}c@{}}$-0.62$\\ $[-1.33,-0.22]$\\ $1.27$\end{tabular} & \begin{tabular}[c]{@{}c@{}}$-0.79$\\ $[-1.22,-0.33]$\\ $1.51$\end{tabular} & \begin{tabular}[c]{@{}c@{}}$-0.48$\\ $[-0.94,-0.09]$\\ $0.90$\end{tabular} & \begin{tabular}[c]{@{}c@{}}$-0.25$\\ $[-0.86,-0.14]$\\ $0.74$\end{tabular} & \begin{tabular}[c]{@{}c@{}}$-0.44$\\ $[-0.70,0.00]$\\ $0.80$\end{tabular} \\ \addlinespace[3pt]
top-$k$ bin.\ $-$ top-$k$ & \begin{tabular}[c]{@{}c@{}}$0.00$\\ $[-0.06,0.00]$\\ $0.31$\end{tabular} & \begin{tabular}[c]{@{}c@{}}$-0.02$\\ $[-0.41,0.00]$\\ $0.30$\end{tabular} & \begin{tabular}[c]{@{}c@{}}$0.00$\\ $[-0.38,+0.07]$\\ $0.34$\end{tabular} & \begin{tabular}[c]{@{}c@{}}$0.00$\\ $[-0.28,+0.16]$\\ $0.25$\end{tabular} & \begin{tabular}[c]{@{}c@{}}$0.00$\\ $[-0.17,+0.30]$\\ $0.36$\end{tabular} & \begin{tabular}[c]{@{}c@{}}$0.00$\\ $[-0.14,+0.28]$\\ $0.26$\end{tabular} & \begin{tabular}[c]{@{}c@{}}$0.00$\\ $[-0.10,+0.12]$\\ $0.27$\end{tabular} & \begin{tabular}[c]{@{}c@{}}$0.00$\\ $[-0.10,+0.06]$\\ $0.14$\end{tabular} \\ \addlinespace[3pt]
top-$k$ bin.\ $-$ $\relu$ bin. & \begin{tabular}[c]{@{}c@{}}$0.00$\\ $[0.00,+0.38]$\\ $0.52$\end{tabular} & \begin{tabular}[c]{@{}c@{}}$0.00$\\ $[-0.58,+0.06]$\\ $0.51$\end{tabular} & \begin{tabular}[c]{@{}c@{}}$0.00$\\ $[-0.49,+0.24]$\\ $0.76$\end{tabular} & \begin{tabular}[c]{@{}c@{}}$+0.37$\\ $[-0.33,+0.83]$\\ $0.87$\end{tabular} & \begin{tabular}[c]{@{}c@{}}$+0.58$\\ $[-0.24,+1.17]$\\ $1.08$\end{tabular} & \begin{tabular}[c]{@{}c@{}}$+0.56$\\ $[+0.02,+0.90]$\\ $0.82$\end{tabular} & \begin{tabular}[c]{@{}c@{}}$+0.62$\\ $[+0.35,+1.04]$\\ $1.02$\end{tabular} & \begin{tabular}[c]{@{}c@{}}$+0.50$\\ $[+0.18,+0.87]$\\ $0.97$\end{tabular} \\ \addlinespace[3pt]
learned $-$ random $\rmM$ & \begin{tabular}[c]{@{}c@{}}$+3.44$\\ $[+0.49,+8.16]$\\ $11.32$\end{tabular} & \begin{tabular}[c]{@{}c@{}}$+3.10$\\ $[+1.35,+4.71]$\\ $6.24$\end{tabular} & \begin{tabular}[c]{@{}c@{}}$+2.54$\\ $[+1.55,+3.60]$\\ $4.08$\end{tabular} & \begin{tabular}[c]{@{}c@{}}$+1.55$\\ $[+0.83,+2.45]$\\ $2.86$\end{tabular} & \begin{tabular}[c]{@{}c@{}}$+1.00$\\ $[+0.55,+1.95]$\\ $1.83$\end{tabular} & \begin{tabular}[c]{@{}c@{}}$+0.30$\\ $[0.00,+0.56]$\\ $0.65$\end{tabular} & \begin{tabular}[c]{@{}c@{}}$+0.24$\\ $[0.00,+0.49]$\\ $0.54$\end{tabular} & \begin{tabular}[c]{@{}c@{}}$+0.13$\\ $[0.00,+0.40]$\\ $0.40$\end{tabular} \\ \addlinespace[3pt]
\midrule
\multicolumn{9}{l}{\emph{regression, $R^2$}} \\
top-$k$ $-$ $\relu$ & \begin{tabular}[c]{@{}c@{}}$-0.044$\\ $[-0.075,-0.027]$\\ $0.079$\end{tabular} & \begin{tabular}[c]{@{}c@{}}$-0.032$\\ $[-0.060,-0.012]$\\ $0.061$\end{tabular} & \begin{tabular}[c]{@{}c@{}}$-0.010$\\ $[-0.022,-0.001]$\\ $0.024$\end{tabular} & \begin{tabular}[c]{@{}c@{}}$-0.014$\\ $[-0.028,-0.002]$\\ $0.022$\end{tabular} & \begin{tabular}[c]{@{}c@{}}$-0.007$\\ $[-0.014,-0.001]$\\ $0.011$\end{tabular} & \begin{tabular}[c]{@{}c@{}}$-0.001$\\ $[-0.008,+0.005]$\\ $0.012$\end{tabular} & \begin{tabular}[c]{@{}c@{}}$-0.001$\\ $[-0.005,+0.008]$\\ $0.015$\end{tabular} & \begin{tabular}[c]{@{}c@{}}$+0.002$\\ $[-0.009,+0.012]$\\ $0.016$\end{tabular} \\ \addlinespace[3pt]
LWTA $-$ $\relu$ & \begin{tabular}[c]{@{}c@{}}$-0.039$\\ $[-0.058,-0.023]$\\ $0.072$\end{tabular} & \begin{tabular}[c]{@{}c@{}}$-0.030$\\ $[-0.059,-0.013]$\\ $0.063$\end{tabular} & \begin{tabular}[c]{@{}c@{}}$-0.014$\\ $[-0.029,-0.004]$\\ $0.026$\end{tabular} & \begin{tabular}[c]{@{}c@{}}$-0.011$\\ $[-0.021,-0.003]$\\ $0.019$\end{tabular} & \begin{tabular}[c]{@{}c@{}}$-0.007$\\ $[-0.017,-0.001]$\\ $0.015$\end{tabular} & \begin{tabular}[c]{@{}c@{}}$+0.001$\\ $[-0.009,+0.006]$\\ $0.010$\end{tabular} & \begin{tabular}[c]{@{}c@{}}$+0.001$\\ $[-0.009,+0.014]$\\ $0.015$\end{tabular} & \begin{tabular}[c]{@{}c@{}}$+0.002$\\ $[-0.008,+0.020]$\\ $0.019$\end{tabular} \\ \addlinespace[3pt]
learned $-$ random $\rmM$ & \begin{tabular}[c]{@{}c@{}}$+0.166$\\ $[+0.104,+0.218]$\\ $0.239$\end{tabular} & \begin{tabular}[c]{@{}c@{}}$+0.147$\\ $[+0.091,+0.185]$\\ $0.209$\end{tabular} & \begin{tabular}[c]{@{}c@{}}$+0.092$\\ $[+0.064,+0.140]$\\ $0.143$\end{tabular} & \begin{tabular}[c]{@{}c@{}}$+0.050$\\ $[+0.028,+0.101]$\\ $0.085$\end{tabular} & \begin{tabular}[c]{@{}c@{}}$+0.034$\\ $[+0.014,+0.054]$\\ $0.052$\end{tabular} & \begin{tabular}[c]{@{}c@{}}$+0.015$\\ $[+0.006,+0.031]$\\ $0.028$\end{tabular} & \begin{tabular}[c]{@{}c@{}}$+0.002$\\ $[-0.005,+0.012]$\\ $0.012$\end{tabular} & \begin{tabular}[c]{@{}c@{}}$+0.002$\\ $[-0.002,+0.008]$\\ $0.009$\end{tabular} \\ \addlinespace[3pt]
\bottomrule
\end{tabular}}
\end{table}

\begin{table}[htb]
\caption{Holm correction across the eight expansion factors, for the comparisons
of \cref{tab:equiv} and the fixed sparsity levels of \cref{tab:wtlsparse}. Upper
part, each cell: the Holm-adjusted $p$ of the two-sided Wilcoxon test
($^{\ast}$: below $0.05$), and below it the smallest margin $\delta$ at which the
Holm-adjusted TOST rejects at $0.05$. Lower part: Holm-adjusted TOST $p$ at a
margin of one accuracy point (\textbf{bold}: below $0.05$, so the level is
affordable after correction).}
\label{tab:holm}
\centering
\scriptsize
\setlength{\tabcolsep}{2pt}
\resizebox{\textwidth}{!}{%
\begin{tabular}{lcccccccc}
\toprule
comparison & $m/d=1$ & $m/d=2$ & $m/d=4$ & $m/d=8$ & $m/d=16$ & $m/d=64$ & $m/d=128$ & $m/d=256$ \\
\midrule
\multicolumn{9}{l}{\emph{classification, accuracy points}} \\
top-$k$ $-$ $\relu$ & \begin{tabular}[c]{@{}c@{}}.276\\ 1.53\end{tabular} & \begin{tabular}[c]{@{}c@{}}1.000\\ 0.80\end{tabular} & \begin{tabular}[c]{@{}c@{}}1.000\\ 0.83\end{tabular} & \begin{tabular}[c]{@{}c@{}}1.000\\ 0.72\end{tabular} & \begin{tabular}[c]{@{}c@{}}1.000\\ 0.59\end{tabular} & \begin{tabular}[c]{@{}c@{}}1.000\\ 0.59\end{tabular} & \begin{tabular}[c]{@{}c@{}}1.000\\ 0.61\end{tabular} & \begin{tabular}[c]{@{}c@{}}1.000\\ 0.67\end{tabular} \\ \addlinespace[2pt]
LWTA $-$ $\relu$ & \begin{tabular}[c]{@{}c@{}}.870\\ 1.09\end{tabular} & \begin{tabular}[c]{@{}c@{}}.843\\ 0.92\end{tabular} & \begin{tabular}[c]{@{}c@{}}1.000\\ 0.92\end{tabular} & \begin{tabular}[c]{@{}c@{}}.843\\ 0.92\end{tabular} & \begin{tabular}[c]{@{}c@{}}1.000\\ 0.62\end{tabular} & \begin{tabular}[c]{@{}c@{}}1.000\\ 0.39\end{tabular} & \begin{tabular}[c]{@{}c@{}}1.000\\ 0.60\end{tabular} & \begin{tabular}[c]{@{}c@{}}1.000\\ 0.58\end{tabular} \\ \addlinespace[2pt]
\texttt{GELU} $-$ $\relu$ & \begin{tabular}[c]{@{}c@{}}.995\\ 1.43\end{tabular} & \begin{tabular}[c]{@{}c@{}}.995\\ 1.27\end{tabular} & \begin{tabular}[c]{@{}c@{}}.189\\ 1.43\end{tabular} & \begin{tabular}[c]{@{}c@{}}.086\\ 1.01\end{tabular} & \begin{tabular}[c]{@{}c@{}}.016$^{\ast}$\\ 1.43\end{tabular} & \begin{tabular}[c]{@{}c@{}}.037$^{\ast}$\\ 1.01\end{tabular} & \begin{tabular}[c]{@{}c@{}}.189\\ 0.90\end{tabular} & \begin{tabular}[c]{@{}c@{}}.303\\ 0.89\end{tabular} \\ \addlinespace[2pt]
\texttt{SiLU} $-$ $\relu$ & \begin{tabular}[c]{@{}c@{}}.170\\ 2.00\end{tabular} & \begin{tabular}[c]{@{}c@{}}.328\\ 1.62\end{tabular} & \begin{tabular}[c]{@{}c@{}}.015$^{\ast}$\\ 2.10\end{tabular} & \begin{tabular}[c]{@{}c@{}}.014$^{\ast}$\\ 1.53\end{tabular} & \begin{tabular}[c]{@{}c@{}}$<$.001$^{\ast}$\\ 1.67\end{tabular} & \begin{tabular}[c]{@{}c@{}}.015$^{\ast}$\\ 1.09\end{tabular} & \begin{tabular}[c]{@{}c@{}}.015$^{\ast}$\\ 0.90\end{tabular} & \begin{tabular}[c]{@{}c@{}}.020$^{\ast}$\\ 1.02\end{tabular} \\ \addlinespace[2pt]
top-$k$ bin.\ $-$ top-$k$ & \begin{tabular}[c]{@{}c@{}}1.000\\ 0.43\end{tabular} & \begin{tabular}[c]{@{}c@{}}1.000\\ 0.43\end{tabular} & \begin{tabular}[c]{@{}c@{}}1.000\\ 0.43\end{tabular} & \begin{tabular}[c]{@{}c@{}}1.000\\ 0.36\end{tabular} & \begin{tabular}[c]{@{}c@{}}1.000\\ 0.43\end{tabular} & \begin{tabular}[c]{@{}c@{}}1.000\\ 0.36\end{tabular} & \begin{tabular}[c]{@{}c@{}}1.000\\ 0.38\end{tabular} & \begin{tabular}[c]{@{}c@{}}1.000\\ 0.22\end{tabular} \\ \addlinespace[2pt]
top-$k$ bin.\ $-$ $\relu$ bin. & \begin{tabular}[c]{@{}c@{}}1.000\\ 0.91\end{tabular} & \begin{tabular}[c]{@{}c@{}}1.000\\ 0.76\end{tabular} & \begin{tabular}[c]{@{}c@{}}1.000\\ 1.02\end{tabular} & \begin{tabular}[c]{@{}c@{}}1.000\\ 1.09\end{tabular} & \begin{tabular}[c]{@{}c@{}}.216\\ 1.10\end{tabular} & \begin{tabular}[c]{@{}c@{}}.135\\ 0.99\end{tabular} & \begin{tabular}[c]{@{}c@{}}.003$^{\ast}$\\ 1.10\end{tabular} & \begin{tabular}[c]{@{}c@{}}.015$^{\ast}$\\ 1.10\end{tabular} \\ \addlinespace[2pt]
learned $-$ random $\rmM$ & \begin{tabular}[c]{@{}c@{}}$<$.001$^{\ast}$\\ 11.32\end{tabular} & \begin{tabular}[c]{@{}c@{}}$<$.001$^{\ast}$\\ 6.56\end{tabular} & \begin{tabular}[c]{@{}c@{}}$<$.001$^{\ast}$\\ 4.42\end{tabular} & \begin{tabular}[c]{@{}c@{}}$<$.001$^{\ast}$\\ 3.12\end{tabular} & \begin{tabular}[c]{@{}c@{}}$<$.001$^{\ast}$\\ 2.04\end{tabular} & \begin{tabular}[c]{@{}c@{}}.016$^{\ast}$\\ 0.82\end{tabular} & \begin{tabular}[c]{@{}c@{}}.016$^{\ast}$\\ 0.70\end{tabular} & \begin{tabular}[c]{@{}c@{}}.127\\ 0.49\end{tabular} \\ \addlinespace[2pt]
\midrule
\multicolumn{9}{l}{\emph{regression, $R^2$}} \\
top-$k$ $-$ $\relu$ & \begin{tabular}[c]{@{}c@{}}$<$.001$^{\ast}$\\ 0.079\end{tabular} & \begin{tabular}[c]{@{}c@{}}$<$.001$^{\ast}$\\ 0.065\end{tabular} & \begin{tabular}[c]{@{}c@{}}.335\\ 0.029\end{tabular} & \begin{tabular}[c]{@{}c@{}}.351\\ 0.026\end{tabular} & \begin{tabular}[c]{@{}c@{}}.997\\ 0.015\end{tabular} & \begin{tabular}[c]{@{}c@{}}1.000\\ 0.023\end{tabular} & \begin{tabular}[c]{@{}c@{}}1.000\\ 0.024\end{tabular} & \begin{tabular}[c]{@{}c@{}}.745\\ 0.023\end{tabular} \\ \addlinespace[2pt]
LWTA $-$ $\relu$ & \begin{tabular}[c]{@{}c@{}}$<$.001$^{\ast}$\\ 0.071\end{tabular} & \begin{tabular}[c]{@{}c@{}}$<$.001$^{\ast}$\\ 0.069\end{tabular} & \begin{tabular}[c]{@{}c@{}}.063\\ 0.030\end{tabular} & \begin{tabular}[c]{@{}c@{}}.411\\ 0.023\end{tabular} & \begin{tabular}[c]{@{}c@{}}.481\\ 0.020\end{tabular} & \begin{tabular}[c]{@{}c@{}}.861\\ 0.021\end{tabular} & \begin{tabular}[c]{@{}c@{}}.754\\ 0.022\end{tabular} & \begin{tabular}[c]{@{}c@{}}.754\\ 0.025\end{tabular} \\ \addlinespace[2pt]
learned $-$ random $\rmM$ & \begin{tabular}[c]{@{}c@{}}$<$.001$^{\ast}$\\ 0.239\end{tabular} & \begin{tabular}[c]{@{}c@{}}$<$.001$^{\ast}$\\ 0.222\end{tabular} & \begin{tabular}[c]{@{}c@{}}$<$.001$^{\ast}$\\ 0.150\end{tabular} & \begin{tabular}[c]{@{}c@{}}$<$.001$^{\ast}$\\ 0.093\end{tabular} & \begin{tabular}[c]{@{}c@{}}$<$.001$^{\ast}$\\ 0.060\end{tabular} & \begin{tabular}[c]{@{}c@{}}.001$^{\ast}$\\ 0.032\end{tabular} & \begin{tabular}[c]{@{}c@{}}.280\\ 0.017\end{tabular} & \begin{tabular}[c]{@{}c@{}}.280\\ 0.013\end{tabular} \\ \addlinespace[2pt]
\midrule
\multicolumn{9}{l}{\emph{fixed sparsity, top-$k$ $-$ $\relu$, classification: Holm-adjusted TOST $p$ at one point}} \\
$90\%$ sparsity & \textbf{.021} & \textbf{.002} & \textbf{.021} & \textbf{$<$.001} & \textbf{.002} & \textbf{$<$.001} & \textbf{$<$.001} & \textbf{.002} \\
$97\%$ sparsity & .885 & .695 & .885 & .695 & .183 & \textbf{.003} & \textbf{.015} & \textbf{.002} \\
$99\%$ sparsity & 1.000 & 1.000 & 1.000 & 1.000 & 1.000 & 1.000 & .693 & .064 \\
\bottomrule
\end{tabular}}
\end{table}

\FloatBarrier

\section{Dataset counts per expansion factor}
\label{app:counts}

The number of datasets entering each comparison changes with $m/d$, for three
separate reasons, and we list them here so that no count in the main text is
unexplained.

\begin{table}[h]
\caption{Datasets supporting each group of methods at each expansion factor.}
\label{tab:counts}
\centering
\small
\begin{tabular}{lcccccccc}
\toprule
 & $m/d{=}1$ & $2$ & $4$ & $8$ & $16$ & $64$ & $128$ & $256$ \\
\midrule
dense activations       & 73 & 72 & 72 & 71 & 71 & 71 & 70 & 67 \\
sparse activations      & 65 & 72 & 72 & 71 & 71 & 71 & 70 & 67 \\
data-supported $\rmM$      & 57 & 66 & 68 & 71 & 71 & 71 & 70 & 67 \\
\bottomrule
\end{tabular}
\end{table}

First, the integer-$k$ floor. The smallest sparsity level an integer $k$ can
reach is $k/m = 1/m$, so a dataset admits a sparse method at $k \leq 0.1m$ only
when $m \geq 10$. At $m = d$ this excludes the eight classification datasets
with $d < 10$, which is the whole difference between the first two rows at
$m = d$ ($73$ against $65$). In the submitted version $31$ further
configurations at $m \leq 4d$ had no sparse result, because the only admissible
level, $k/m = 0.1$, rounded to an integer $k$ just above our $k \leq 0.105m$
cap. We reran them with $k = \lfloor 0.1m \rfloor$ (\cref{tab:rerun}), so from
$m = 2d$ the sparse and dense rows now agree. The data-supported $\rmM$ was not
rerun, which is why its row is still lower at $m \leq 4d$.

Second, the $m \leq 30{,}000$ cap and our compute budget. Two very
high-dimensional datasets ($d = 9{,}983$ and $2{,}983$ after one-hot encoding)
were run only up to $m = d$ and $m = 4d$, which removes one dataset at $m = 2d$
and one at $m = 8d$ from every row; the cap then reduces the rows from $71$ to
$67$ between $m = 64d$ and $m = 256d$. This affects every method equally and so
does not bias any comparison.

Third, two gaps that are not about expansion factors at all. Three of the $76$
admissible classification tasks are absent because they one-hot encode to an
input dimension of roughly $10{,}000$, which made every expansion too expensive
for our budget. And our two sweeps used different sparsity grids, so the $95\%$
level is defined over only part of the expansion range and we do not report it.

The regression pool behaves the same way. The paired comparison against $\relu$
rests on $41$ of the $48$ tasks at $m = d$, where the integer-$k$ floor excludes
the seven tasks with $d < 10$, and on all $48$ from $m = 2d$ to $m = 16d$; the
cap leaves $47$, $47$ and $46$ at the three widest settings.

\Cref{tab:figcounts} gives the number of datasets behind every point of
\cref{fig:scaled}, where a dataset counts only if all three constructions of
$\rmM$ are present at that expansion factor and sparsity level. A dash marks a
point supported by fewer than $10$ datasets, which the figure omits.

\begin{table}[h]
\caption{Datasets behind each point of \cref{fig:scaled}, by task type, sparsity
level and expansion factor.}
\label{tab:figcounts}
\centering
\small
\begin{tabular}{llcccccccc}
\toprule
task & $k/m$ & $m/d{=}1$ & $2$ & $4$ & $8$ & $16$ & $64$ & $128$ & $256$ \\
\midrule
binary      & $0.01$ & --  & 12 & 26 & 28 & 37 & 37 & 37 & 35 \\
            & $0.1$  & 27  & 34 & 34 & 35 & 37 & 37 & 37 & 35 \\
            & tuned  & 30  & 34 & 35 & 37 & 37 & 37 & 37 & 35 \\
\midrule
multi-class & $0.01$ & --  & 14 & 19 & 28 & 34 & 34 & 33 & 32 \\
            & $0.1$  & 22  & 30 & 31 & 33 & 34 & 34 & 33 & 32 \\
            & tuned  & 27  & 32 & 33 & 34 & 34 & 34 & 33 & 32 \\
\midrule
regression  & $0.01$ & --  & -- & 16 & 27 & 33 & 32 & 32 & 31 \\
            & $0.1$  & 17  & 19 & 31 & 32 & 33 & 32 & 32 & 31 \\
            & tuned  & 19  & 29 & 32 & 33 & 33 & 32 & 32 & 31 \\
\bottomrule
\end{tabular}
\end{table}

\paragraph{How many units $\lwta$ keeps.}
\label{app:lwta}
$\lwta$ keeps one unit per block of size $\mathrm{round}(m/k)$, which is
$\lceil m / \mathrm{round}(m/k) \rceil$ units: exactly $k$ when $k$ divides $m$, and
otherwise usually $k$ or $k + 1$. Among the configurations selected on validation,
$63\%$ keep exactly $k$ units, $33\%$ keep $k + 1$, $4\%$ keep between $k + 2$ and
$k + 8$, and two keep $k - 1$. The realized sparsity then exceeds $k/m = 0.105$ in about $4\%$ of the
classification configurations and $6\%$ of the regression ones, all at
$m \leq 16d$. Removing them changes no verdict of
\cref{tab:wtlfull,tab:wtlreg}.

\section{Distinct rows of the data-supported expansion}
\label{app:distinct}

\begin{table}[h]
\caption{How many of the $m$ rows of a data-supported $\rmM$ are distinct. The
rows are drawn from the training inputs with replacement, so the count is
bounded by $n$ and the fraction falls as $m$ grows. $D$ is the number of
(dataset, sparsity level) configurations entering the column. Each row is a
separate median over those configurations, so the last row is the median of
the per-configuration fraction and not the quotient of the two rows above it.}
\label{tab:distinct}
\centering
\small
\begin{tabular}{lcccccccc}
\toprule
 & $m/d{=}1$ & $2$ & $4$ & $8$ & $16$ & $64$ & $128$ & $256$ \\
\midrule
configurations $D$      & 120 & 166 & 192 & 207 & 237 & 275 & 285 & 276 \\
median $m$              &  58 & 106 & 184 & 296 & 592 & 2304 & 4096 & 6912 \\
median distinct rows    &  54 & 102 & 158 & 244 & 326 &  622 &  665 &  693 \\
median fraction distinct & $96.9\%$ & $94.7\%$ & $89.9\%$ & $85.2\%$ & $75.6\%$
                        & $38.5\%$ & $21.6\%$ & $12.0\%$ \\
\bottomrule
\end{tabular}
\end{table}

The difference between the two comparisons in \cref{sec:res-projection} has a cause we can measure. The rows of a data-supported $\rmM$
are training points, so it can never hold more than $n$ distinct directions
however large $m$ becomes, while the other two constructions hold $m$. The
fraction of $\rmM$ that is distinct falls from $97\%$ at $m = d$ to $76\%$ at
$m = 16d$ and $12\%$ at $m = 256d$ (\cref{tab:distinct}). Compared against a
data-independent $\rmM$, the sampled one wins at $m = d$, $2d$ and $4d$, by a
median $1.64$, $1.55$ and $0.91$ points, all significant; is not significantly different
at $m = 8d$ and $m = 16d$; and \emph{loses} from $m = 64d$, by $0.41$, $0.38$
and $0.77$ points, again all significant. The sign changes in exactly the range
where the fraction of distinct rows falls through three quarters, which is what
we would expect if its advantage comes from the directions being real data and
its disadvantage from there not being enough of them. Regression shows the same
crossover at the same expansion factors: the sampled $\rmM$ is ahead of a
data-independent one at $m \leq 4d$, by a median $0.044$, $0.024$ and $0.013$ in
$R^2$ ($p \leq 0.025$), not significantly different at $m = 8d$ and $m = 16d$, and behind from
$m = 64d$, by $0.019$ to $0.024$ ($p < 0.0002$), with $R^2$ raw or floored.

The saturation is a property of sampling with replacement rather than of our
data: once $m$ exceeds the training set size, no further distinct rows are
available to draw. Its consequence is the crossover above:
identical rows give identical pre-activations, so a top-$k$ over a saturated
$\rmM$ is choosing among ties, and the expansion has far fewer distinct
directions than its nominal size.

\section{Reproducibility details}
\label{app:repro}

\paragraph{Datasets.} \Cref{tab:openml} lists the OpenML ID, the input dimension
$d$ after one-hot encoding and the numbers of training and test points of all
$121$ datasets. Datasets were admitted by OpenML's \texttt{NumberOfFeatures}
$\in [8, 120]$, which counts the target column, and \texttt{NumberOfInstances}
$\in [300, 20{,}000]$. One-hot encoding can raise $d$ above $120$, and a
dataset with $8$ OpenML features has $d = 7$ inputs.

\paragraph{Preprocessing.} Categorical columns are one-hot encoded. Missing
values are imputed with the column median, computed on the full dataset before
splitting. Each dataset is split $70/10/20$ into training, validation and test
sets (stratified for classification) with three fixed split seeds. Features are
standardized with statistics of the training split, and every input is then
scaled to unit Euclidean norm.

\paragraph{Model and training.} The rows of $\rmM$ are renormalized to unit
norm at every forward pass, so every pre-activation is a cosine in $[-1, 1]$.
The hidden layer has no bias, and $\rmW$ is initialized to the constant $1/m$.
Binary classification uses a single logit with the logistic loss, multi-class
classification the cross-entropy, and regression the squared error. All models
are trained with Adam (learning rate $10^{-3}$, no weight decay, batch size
$256$) for up to $300$ epochs with early stopping on the validation split
(patience $30$). For a requested level $f \in \{0.001, 0.01, 0.03, 0.05, 0.10\}$
we use $k = \mathrm{round}(f m)$, admit a level only if $f m \geq 1$, and keep
only configurations with $k/m \leq 0.105$.

\paragraph{Straight-through estimator.} For a binary code the forward pass
emits the $0/1$ mask of the selected units. The backward pass uses the gradient
of $\rvz = \rvy \odot \text{mask}$, where $\rvy = \rmM\rvx$: the identity on the
selected units and zero elsewhere. For the dense binary baseline the mask is
$\mathbb{1}[\rvy > 0]$, so its backward pass is that of $\relu$.

\paragraph{Configurations rerun for this version.} In the first version of this paper,
$31$ (dataset, $m/d$) configurations had no sparse result.
At these widths the only admissible level, $f = 0.10$, rounds to an integer $k$
just above the $k \leq 0.105m$ cap (for example $m = 16$ gives $k = 2$). We
reran them with $k = \lfloor 0.1 m \rfloor$ for every sparse activation and both
constructions of $\rmM$, and every result in the paper includes them
(\cref{tab:rerun}). The data-supported $\rmM$ was not rerun, so those comparisons
are unchanged.

\begin{table}[h]
\caption{The $31$ configurations rerun with $k = \lfloor 0.1m \rfloor$.}
\label{tab:rerun}
\centering
\scriptsize
\setlength{\tabcolsep}{4pt}
\begin{tabular}{llrrrrr}
\toprule
pool & dataset & OpenML ID & $m/d$ & $m$ & $k$ & $k/m$ \\
\midrule
cls. & letter & 6 & 1 & 16 & 1 & 0.062 \\
cls. & nursery & 26 & 1 & 27 & 2 & 0.074 \\
cls. & pendigits & 32 & 1 & 16 & 1 & 0.062 \\
cls. & schizo & 466 & 1 & 17 & 1 & 0.059 \\
cls. & segment & 36 & 1 & 19 & 1 & 0.053 \\
cls. & tic-tac-toe & 50 & 1 & 27 & 2 & 0.074 \\
cls. & vehicle & 54 & 1 & 18 & 1 & 0.056 \\
cls. & vowel & 307 & 1 & 27 & 2 & 0.074 \\
cls. & JapaneseVowels & 375 & 2 & 28 & 2 & 0.071 \\
cls. & bank8FM & 725 & 2 & 16 & 1 & 0.062 \\
cls. & breast-w & 15 & 2 & 18 & 1 & 0.056 \\
cls. & diabetes & 37 & 2 & 16 & 1 & 0.062 \\
cls. & diggle\_table\_a2 & 694 & 2 & 16 & 1 & 0.062 \\
cls. & yeast & 181 & 2 & 16 & 1 & 0.062 \\
cls. & ecoli & 39 & 4 & 28 & 2 & 0.071 \\
reg. & elevators & 216 & 1 & 18 & 1 & 0.056 \\
reg. & kdd\_coil\_1 & 567 & 1 & 18 & 1 & 0.056 \\
reg. & kdd\_coil\_2 & 568 & 1 & 18 & 1 & 0.056 \\
reg. & kdd\_coil\_3 & 570 & 1 & 18 & 1 & 0.056 \\
reg. & kdd\_coil\_4 & 575 & 1 & 18 & 1 & 0.056 \\
reg. & kdd\_coil\_5 & 576 & 1 & 18 & 1 & 0.056 \\
reg. & kdd\_coil\_6 & 577 & 1 & 18 & 1 & 0.056 \\
reg. & kdd\_coil\_7 & 578 & 1 & 18 & 1 & 0.056 \\
reg. & plasma\_retinol & 511 & 1 & 18 & 1 & 0.056 \\
reg. & bank8FM & 572 & 2 & 16 & 1 & 0.062 \\
reg. & kin8nm & 189 & 2 & 16 & 1 & 0.062 \\
reg. & places & 509 & 2 & 16 & 1 & 0.062 \\
reg. & puma8NH & 225 & 2 & 16 & 1 & 0.062 \\
reg. & stock & 223 & 2 & 18 & 1 & 0.056 \\
reg. & no2 & 547 & 4 & 28 & 2 & 0.071 \\
reg. & pm10 & 522 & 4 & 28 & 2 & 0.071 \\
\bottomrule
\end{tabular}
\end{table}

\begin{table}[p]
\caption{All $121$ datasets: OpenML ID, task (binary or multi-class
classification, or regression), input dimension $d$ after one-hot encoding,
and the numbers of training and test points (the validation split is the
remaining $10\%$).}
\label{tab:openml}
\centering
\tiny
\setlength{\tabcolsep}{2pt}
\begin{tabular}{rlcrrr@{\hspace{8pt}}rlcrrr@{\hspace{8pt}}rlcrrr}
\toprule
ID & name & task & $d$ & train & test & ID & name & task & $d$ & train & test & ID & name & task & $d$ & train & test \\
\midrule
183 & abalone & multi. & 10 & 2923 & 836 & 757 & meta & bin. & 65 & 369 & 106 & 534 & cps\_85\_wages & reg. & 23 & 373 & 107 \\
734 & ailerons & bin. & 40 & 9625 & 2750 & 14 & mfeat-fourier & multi. & 76 & 1400 & 400 & 197 & cpu\_act & reg. & 21 & 5733 & 1639 \\
458 & analcatdata\_authorship & multi. & 70 & 588 & 169 & 16 & mfeat-karhunen & multi. & 64 & 1400 & 400 & 227 & cpu\_small & reg. & 12 & 5733 & 1639 \\
454 & analcatdata\_halloffame & multi. & 22 & 938 & 268 & 22 & mfeat-zernike & multi. & 47 & 1400 & 400 & 216 & elevators & reg. & 18 & 11619 & 3320 \\
728 & analcatdata\_supreme & bin. & 7 & 2835 & 811 & 24 & mushroom & bin. & 116 & 5686 & 1625 & 595 & fri\_c0\_1000\_10 & reg. & 10 & 700 & 200 \\
2 & anneal & multi. & 55 & 628 & 180 & 26 & nursery & multi. & 27 & 9072 & 2592 & 598 & fri\_c0\_1000\_25 & reg. & 25 & 700 & 200 \\
725 & bank8FM & bin. & 8 & 5733 & 1639 & 311 & oil\_spill & bin. & 49 & 655 & 188 & 590 & fri\_c0\_1000\_50 & reg. & 50 & 700 & 200 \\
185 & baseball & multi. & 22 & 938 & 268 & 28 & optdigits & multi. & 64 & 3934 & 1124 & 593 & fri\_c1\_1000\_10 & reg. & 10 & 700 & 200 \\
186 & braziltourism & multi. & 21 & 287 & 83 & 30 & page-blocks & multi. & 10 & 3830 & 1095 & 583 & fri\_c1\_1000\_50 & reg. & 50 & 700 & 200 \\
15 & breast-w & bin. & 9 & 489 & 140 & 32 & pendigits & multi. & 16 & 7693 & 2199 & 582 & fri\_c1\_500\_25 & reg. & 25 & 350 & 100 \\
455 & cars & multi. & 11 & 283 & 82 & 750 & pm10 & bin. & 7 & 350 & 100 & 589 & fri\_c2\_1000\_25 & reg. & 25 & 700 & 200 \\
473 & cjs & multi. & 99 & 1956 & 560 & 722 & pol & bin. & 48 & 10500 & 3000 & 586 & fri\_c3\_1000\_25 & reg. & 25 & 700 & 200 \\
786 & cleveland & bin. & 25 & 211 & 61 & 171 & primary-tumor & multi. & 37 & 237 & 68 & 581 & fri\_c3\_500\_25 & reg. & 25 & 350 & 100 \\
23 & cmc & multi. & 24 & 1030 & 295 & 470 & profb & bin. & 66 & 469 & 135 & 588 & fri\_c4\_1000\_100 & reg. & 100 & 700 & 200 \\
25 & colic & bin. & 137 & 257 & 74 & 752 & puma32H & bin. & 32 & 5733 & 1639 & 592 & fri\_c4\_1000\_25 & reg. & 25 & 700 & 200 \\
488 & colleges\_aaup & multi. & 65 & 812 & 233 & 717 & rmftsa\_ladata & bin. & 10 & 355 & 102 & 604 & fri\_c4\_500\_10 & reg. & 10 & 350 & 100 \\
761 & cpu\_act & bin. & 21 & 5733 & 1639 & 182 & satimage & multi. & 36 & 4501 & 1286 & 584 & fri\_c4\_500\_25 & reg. & 25 & 350 & 100 \\
735 & cpu\_small & bin. & 12 & 5733 & 1639 & 466 & schizo & bin. & 17 & 238 & 68 & 567 & kdd\_coil\_1 & reg. & 18 & 220 & 64 \\
29 & credit-approval & bin. & 46 & 483 & 138 & 36 & segment & multi. & 19 & 1617 & 462 & 568 & kdd\_coil\_2 & reg. & 18 & 220 & 64 \\
31 & credit-g & bin. & 61 & 700 & 200 & 38 & sick & bin. & 53 & 2639 & 755 & 570 & kdd\_coil\_3 & reg. & 18 & 220 & 64 \\
35 & dermatology & multi. & 130 & 255 & 74 & 42 & soybean & multi. & 99 & 477 & 137 & 575 & kdd\_coil\_4 & reg. & 18 & 220 & 64 \\
37 & diabetes & bin. & 8 & 537 & 154 & 44 & spambase & bin. & 57 & 3220 & 921 & 576 & kdd\_coil\_5 & reg. & 18 & 220 & 64 \\
694 & diggle\_table\_a2 & multi. & 8 & 217 & 62 & 337 & SPECTF & bin. & 44 & 244 & 70 & 577 & kdd\_coil\_6 & reg. & 18 & 220 & 64 \\
39 & ecoli & multi. & 7 & 234 & 68 & 46 & splice & multi. & 287 & 2233 & 638 & 578 & kdd\_coil\_7 & reg. & 18 & 220 & 64 \\
188 & eucalyptus & multi. & 91 & 514 & 148 & 377 & synthetic\_control & multi. & 60 & 420 & 120 & 189 & kin8nm & reg. & 8 & 5733 & 1639 \\
779 & fri\_c1\_500\_25 & bin. & 25 & 350 & 100 & 50 & tic-tac-toe & bin. & 27 & 670 & 192 & 299 & libras\_move & reg. & 90 & 252 & 72 \\
766 & fri\_c1\_500\_50 & bin. & 50 & 350 & 100 & 54 & vehicle & multi. & 18 & 591 & 170 & 566 & meta & reg. & 65 & 369 & 106 \\
740 & fri\_c3\_1000\_10 & bin. & 10 & 700 & 200 & 56 & vote & bin. & 32 & 304 & 87 & 547 & no2 & reg. & 7 & 350 & 100 \\
715 & fri\_c3\_1000\_25 & bin. & 25 & 700 & 200 & 307 & vowel & multi. & 27 & 693 & 198 & 200 & pbc & reg. & 29 & 292 & 84 \\
751 & fri\_c4\_1000\_10 & bin. & 10 & 700 & 200 & 60 & waveform-5000 & multi. & 40 & 3500 & 1000 & 516 & pbcseq & reg. & 1046 & 1361 & 389 \\
718 & fri\_c4\_1000\_100 & bin. & 100 & 700 & 200 & 181 & yeast & multi. & 8 & 1038 & 297 & 509 & places & reg. & 8 & 230 & 66 \\
723 & fri\_c4\_1000\_25 & bin. & 25 & 700 & 200 & 316 & yeast\_ml8 & bin. & 129 & 1691 & 484 & 511 & plasma\_retinol & reg. & 18 & 220 & 63 \\
742 & fri\_c4\_500\_100 & bin. & 100 & 350 & 100 & 296 & Ailerons & reg. & 40 & 9625 & 2750 & 522 & pm10 & reg. & 7 & 350 & 100 \\
49 & heart-c & bin. & 25 & 211 & 61 & 196 & autoMpg & reg. & 25 & 278 & 80 & 201 & pol & reg. & 48 & 10500 & 3000 \\
57 & hypothyroid & multi. & 53 & 2639 & 755 & 558 & bank32nh & reg. & 32 & 5733 & 1639 & 225 & puma8NH & reg. & 8 & 5733 & 1639 \\
59 & ionosphere & bin. & 34 & 245 & 71 & 572 & bank8FM & reg. & 8 & 5733 & 1639 & 294 & satellite\_image & reg. & 36 & 4504 & 1287 \\
382 & ipums\_la\_97-small & multi. & 2983 & 4913 & 1404 & 531 & boston & reg. & 22 & 353 & 102 & 546 & sensory & reg. & 36 & 402 & 116 \\
378 & ipums\_la\_99-small & multi. & 9983 & 6190 & 1769 & 543 & boston\_corrected & reg. & 119 & 353 & 102 & 223 & stock & reg. & 9 & 665 & 190 \\
375 & JapaneseVowels & multi. & 14 & 6972 & 1993 & 204 & cholesterol & reg. & 25 & 211 & 61 & 287 & wine\_quality & reg. & 11 & 4547 & 1300 \\
3 & kr-vs-kp & bin. & 73 & 2236 & 640 & 194 & cleveland & reg. & 25 & 211 & 61 &  &  &  &  &  &  \\
6 & letter & multi. & 16 & 14000 & 4000 & 298 & coil2000 & reg. & 85 & 6874 & 1965 &  &  &  &  &  &  \\
\bottomrule
\end{tabular}
\end{table}

\FloatBarrier

\section{Per-seed results of the architecture experiments}
\label{app:archs-extra}
\Cref{tab:archs-mixer,tab:archs-lm,tab:archs-lm256,tab:archs-lm512} give every run behind
\cref{tab:archs}. The Mixer's channel MLP and token MLP use the same expansion
ratio $r$. ``Both'' means top-$k$ in both mixers, with the same $k/m$. In the
Transformer, $\relu$ keeps on average $13\%$, $9\%$, $6\%$ of FFN
units active at $r = 4, 8, 16$ (mean over layers, validation tokens and seeds);
top-$k$ keeps exactly $k/m$. The width-$256$ Transformer has $6$ layers and $4$
heads like the width-$128$ one; the width-$512$ one has $8$ layers and $8$ heads
($51$M, $68$M and $101$M parameters at $r = 4$, $8$, $16$ with the tied
embedding). Both are otherwise trained identically, with the same ratios, arms
and seeds as the width-$128$ one except \texttt{GELU}.
\Cref{fig:archs} plots the paired differences of \cref{tab:archs}.

\begin{figure}[h]
\centering
\includegraphics[width=\linewidth]{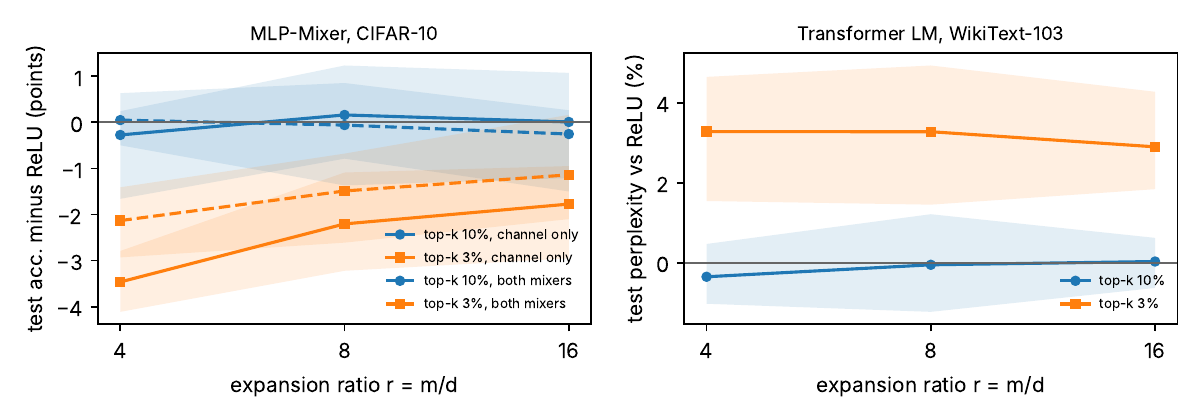}
\caption{Paired differences to $\relu$ at the same expansion ratio and seed.
Left: MLP-Mixer test accuracy (solid: top-$k$ in the channel MLP only; dashed:
in both mixers). Right: Transformer (width $128$) test perplexity relative to
$\relu$. Lines are means over five seeds and ribbons span the minimum to the
maximum.}
\label{fig:archs}
\end{figure}

\begin{table}[h]
\centering\small
\caption{MLP-Mixer on CIFAR-10: test accuracy in \%, per seed, and mean training accuracy (on $10{,}000$ training images).}
\label{tab:archs-mixer}
\begin{tabular}{l c c c c c c c}
\toprule
Arm & $r$ & seed 0 & seed 1 & seed 2 & seed 3 & seed 4 & train \\
\midrule
$\relu$ & 4 & 81.97 & 81.81 & 81.54 & 82.71 & 81.85 & 99.97 \\
$\relu$ & 8 & 83.35 & 82.83 & 81.60 & 81.34 & 81.89 & 99.99 \\
$\relu$ & 16 & 80.97 & 82.33 & 82.24 & 82.08 & 81.81 & 99.99 \\
top-$k$ 10\%, channel & 4 & 81.90 & 82.05 & 81.70 & 81.05 & 81.80 & 99.90 \\
top-$k$ 10\%, channel & 8 & 82.56 & 82.37 & 82.34 & 82.57 & 81.96 & 99.99 \\
top-$k$ 10\%, channel & 16 & 82.04 & 82.33 & 80.74 & 82.61 & 81.74 & 99.99 \\
top-$k$ 10\%, both & 4 & 81.82 & 81.97 & 82.17 & 82.21 & 81.94 & 99.87 \\
top-$k$ 10\%, both & 8 & 81.98 & 82.23 & 81.90 & 81.85 & 82.74 & 99.95 \\
top-$k$ 10\%, both & 16 & 80.72 & 82.01 & 81.03 & 82.32 & 82.07 & 99.98 \\
top-$k$ 3\%, channel & 4 & 78.49 & 78.51 & 78.75 & 78.60 & 78.22 & 99.47 \\
top-$k$ 3\%, channel & 8 & 80.13 & 79.78 & 80.51 & 79.81 & 79.76 & 99.90 \\
top-$k$ 3\%, channel & 16 & 79.38 & 80.19 & 79.29 & 81.13 & 80.57 & 99.97 \\
top-$k$ 3\%, both & 4 & 80.56 & 79.96 & 79.44 & 79.78 & 79.48 & 98.91 \\
top-$k$ 3\%, both & 8 & 80.74 & 80.96 & 80.92 & 80.18 & 80.76 & 99.83 \\
top-$k$ 3\%, both & 16 & 81.13 & 81.03 & 80.14 & 80.18 & 81.25 & 99.95 \\
\texttt{GELU} & 4 & 82.51 & 82.72 & 82.94 & 82.83 & 82.76 & 99.98 \\
\texttt{GELU} & 8 & 83.73 & 82.69 & 82.89 & 83.06 & 82.59 & 99.98 \\
\texttt{GELU} & 16 & 82.41 & 83.04 & 82.78 & 83.23 & 82.82 & 100.00 \\
\bottomrule
\end{tabular}
\end{table}

\begin{table}[h]
\centering\small
\caption{Transformer (width $128$) on WikiText-103: test perplexity per seed, and mean validation perplexity.}
\label{tab:archs-lm}
\begin{tabular}{l c c c c c c c}
\toprule
Arm & $r$ & seed 0 & seed 1 & seed 2 & seed 3 & seed 4 & val \\
\midrule
$\relu$ & 4 & 60.88 & 61.44 & 61.20 & 61.56 & 61.67 & 61.71 \\
$\relu$ & 8 & 57.74 & 57.52 & 57.74 & 58.26 & 57.94 & 58.01 \\
$\relu$ & 16 & 54.70 & 54.78 & 54.70 & 54.46 & 54.10 & 54.63 \\
top-$k$ 10\% & 4 & 61.14 & 61.73 & 60.82 & 60.94 & 61.04 & 61.45 \\
top-$k$ 10\% & 8 & 58.44 & 57.34 & 57.03 & 57.61 & 58.62 & 58.09 \\
top-$k$ 10\% & 16 & 54.57 & 54.43 & 54.75 & 54.62 & 54.44 & 54.54 \\
top-$k$ 3\% & 4 & 63.72 & 64.13 & 63.21 & 63.12 & 62.62 & 63.74 \\
top-$k$ 3\% & 8 & 60.59 & 59.73 & 58.58 & 60.00 & 59.77 & 60.00 \\
top-$k$ 3\% & 16 & 55.70 & 55.90 & 55.93 & 56.68 & 56.41 & 56.19 \\
\texttt{GELU} & 4 & 62.41 & 62.57 & 62.51 & 62.55 & 63.17 & 63.09 \\
\texttt{GELU} & 8 & 59.74 & 60.39 & 60.13 & 60.19 & 59.99 & 60.34 \\
\texttt{GELU} & 16 & 58.28 & 57.49 & 57.72 & 58.29 & 57.88 & 58.00 \\
\bottomrule
\end{tabular}
\end{table}

\begin{table}[h]
\centering\small
\caption{Transformer (width $256$) on WikiText-103: test perplexity per seed, and mean validation perplexity.}
\label{tab:archs-lm256}
\begin{tabular}{l c c c c c c c}
\toprule
Arm & $r$ & seed 0 & seed 1 & seed 2 & seed 3 & seed 4 & val \\
\midrule
$\relu$ & 4 & 45.03 & 44.32 & 44.73 & 44.67 & 44.75 & 44.69 \\
$\relu$ & 8 & 42.38 & 41.87 & 41.83 & 41.99 & 42.25 & 41.98 \\
$\relu$ & 16 & 40.57 & 39.94 & 40.17 & 40.33 & 40.36 & 40.12 \\
top-$k$ 10\% & 4 & 45.26 & 44.67 & 44.71 & 45.31 & 45.02 & 45.00 \\
top-$k$ 10\% & 8 & 42.27 & 42.40 & 41.94 & 42.30 & 42.31 & 42.06 \\
top-$k$ 10\% & 16 & 39.93 & 40.10 & 39.94 & 40.16 & 40.12 & 39.82 \\
top-$k$ 3\% & 4 & 47.17 & 46.49 & 46.71 & 46.81 & 46.68 & 46.80 \\
top-$k$ 3\% & 8 & 43.55 & 43.56 & 43.30 & 43.49 & 43.73 & 43.39 \\
top-$k$ 3\% & 16 & 41.00 & 40.71 & 40.89 & 41.06 & 41.01 & 40.76 \\
\bottomrule
\end{tabular}
\end{table}

\begin{table}[h]
\centering\small
\caption{Transformer (width $512$, $8$ layers) on WikiText-103: test perplexity per seed, and mean validation perplexity.}
\label{tab:archs-lm512}
\begin{tabular}{l c c c c c c c}
\toprule
Arm & $r$ & seed 0 & seed 1 & seed 2 & seed 3 & seed 4 & val \\
\midrule
$\relu$ & 4 & 34.01 & 33.97 & 33.84 & 34.26 & 34.12 & 33.70 \\
$\relu$ & 8 & 32.72 & 32.42 & 32.45 & 32.81 & 32.46 & 32.20 \\
$\relu$ & 16 & 31.48 & 31.20 & 31.15 & 31.55 & 31.51 & 30.98 \\
top-$k$ 10\% & 4 & 33.91 & 33.77 & 33.84 & 34.20 & 33.97 & 33.64 \\
top-$k$ 10\% & 8 & 32.38 & 32.33 & 32.30 & 32.45 & 32.56 & 32.02 \\
top-$k$ 10\% & 16 & 31.37 & 31.13 & 31.22 & 31.45 & 31.30 & 30.86 \\
top-$k$ 3\% & 4 & 34.65 & 34.76 & 34.73 & 35.04 & 34.83 & 34.57 \\
top-$k$ 3\% & 8 & 32.94 & 32.92 & 32.91 & 33.04 & 32.94 & 32.63 \\
top-$k$ 3\% & 16 & 31.24 & 31.25 & 31.21 & 31.50 & 31.44 & 30.97 \\
\bottomrule
\end{tabular}
\end{table}

\FloatBarrier

\section{Approximation and generalization results} \label{app:apprx-gen-theory}
In this section we bound the excess Bayes risk when $\rmM\in\R^{m\times d}$ is random, with rows drawn independently and uniformly from the unit sphere $\mathbb{S}^{d-1}$, the inputs lie on the unit sphere, $\Xcal\subseteq\mathbb{S}^{d-1}$ (as in our experiments), and $\rvw\in\R^m$ is learned via SGD using binary cross entropy (\bce) loss in the binary classification setting. For any $R>0$, we learn $\rvw$ by minimizing the empirical loss over the following function class, for example with projected SGD, which projects $\rvw$ onto the ball $\|\rvw\|_2\leq R$ after every step:
\[\Fcal_R = \left\{f(\rvx;\rvw) = \rvw\cdot \sigma(\rmM \rvx): \rvw = (w_1,\ldots,w_m)\in\R^m, \|\rvw\|_2\leq R\right\}\]
where $\sigma$ is the $k$ winners-take-all activation such that for any $\rvz\in\R^m$, $\sigma(\rvz)$ sets the largest $k$ entries of $\rvz$ to 1 and the rest to zero.

We first show that the \bce loss and \logit loss are mathematically equivalent as this will help us work with \logit loss instead of \bce loss. For any $f\in\Fcal_R$, let $z=f(\rvx;\rvw)$. Then the  \bce loss is defined as:
\begin{equation}\label{eq:bce_loss}
    \bce(y,z) = -y\ln(s(z))-(1-y)\ln(1-s(z))
\end{equation}
where $s(z) = 1/(1+e^{-z})$ is the sigmoid function. Note that $y\in\{0,1\}$. Next, the \logit loss function is defined as: 
\begin{equation}\label{eq:logistic_loss_pm}
    \logit (y',z) = \log_2(1+e^{-y'z}) = \log_2 e\cdot\ln(1+e^{-y'z})
\end{equation}
Note that here $y'\in\{-1,1\}$. It is easy to use the transformation $y' = 2y-1$ and write the \logit loss in terms of $y$ as follows:
\begin{equation}\label{eq:logistic_loss}
    \logit (y,z) = \log_2(1+e^{-(2y-1)z}) = \log_2 e\cdot\ln(1+e^{-(2y-1)z})
\end{equation}
\begin{lemma}\label{lem:bce_logistic}
    The \bce loss given in \eqref{eq:bce_loss} and the \logit loss given in \eqref{eq:logistic_loss} are equivalent.
\end{lemma}
\begin{proof}
    For $y=1$,
    \begin{eqnarray*}
        \bce(1,z) &=& -\ln s(z) = \ln (1+e^{-z}) = \left(\frac{1}{\log_2 e}\right)\cdot\logit(1,z).
    \end{eqnarray*}

Similarly, for $y=0$,
    \begin{eqnarray*}
        \bce(0,z) &= & -\ln (1-s(z)) = -\ln\left(\frac{e^{-z}}{1+e^{-z}}\right) = \ln (1+e^{-z})+z\\
        &=& \ln\left(\frac{1+e^{z}}{e^{z}}\right)+z = \ln (1+e^z) = \left(\frac{1}{\log_2 e}\right)\cdot\logit(0,z).
    \end{eqnarray*}
\end{proof}

From now on, we will assume that the labels come from the set $\{1,-1\}$ and because of the equivalence above, use the \logit loss given in \eqref{eq:logistic_loss_pm}. 

To formalize the learning problem, we introduce a probabilistic setting, and let $(\rmX, Y)$ be an $\Xcal\times \{1,-1\}$-valued
random pair, modeling observation and its corresponding class. The distribution of the random pair
$(\rmX, Y)$ may be described by the probability distribution of $\rmX$ (given by the probabilities $\pr(\rmX\in A)=\mu(A)$ for all
measurable subsets $A\subset\Xcal)$ and conditional probability, $\eta(\rvx) = \pr(Y = 1|\rmX = \rvx)$. For any classifier $g:\Xcal\rightarrow \{1,-1\}$, its probability of error (also known as risk) is given by 
\begin{equation}\label{eq:expected_loss}
    L(g) = \pr(g(\rmX)\neq Y)
\end{equation}
and given a training set $(\rvx_1,y_1),\ldots, (\rvx_n,y_n)$ sampled i.i.d. from $\Xcal\times \{1,-1\}$, its empirical version is given by:
\begin{equation}\label{eq:empirical_loss}
    L_n(g) = \frac{1}{n}\sum_{i=1}^n\I_{\{g(\rvx_i)\neq y_i\}}
\end{equation}
Given $\eta$, one may easily construct a classifier with minimal probability of error, known as the Bayes optimal classifier as follows:
\begin{equation}\label{eq:bayes_optimal}
g^*(\rvx) = 
\begin{cases} 
1, & \text{if } \eta(\rvx) \geq 1/2 \\
-1, & \text{otherwise} 
\end{cases}
\end{equation}
The minimal risk $L^* \triangleq L(g^*)$. Given a training set $(\rvx_1,y_1),\ldots, (\rvx_n,y_n)$ sampled i.i.d. from $\Xcal\times \{1,-1\}$, we learn a function $f_n\in\Fcal_R$ that minimizes the empirical \logit loss over $\Fcal_R$ and we aim to bound $L(f_n)-L^*$.

We make the following assumption.
\begin{assumption}\label{ass:eta}
    The function $f^*(\rvx)=\ln\left(\frac{\eta(\rvx)}{1-\eta(\rvx)}\right)$ is $\lambda$-Lipschitz.
\end{assumption}
Note that it is straightforward to check that the above assumption implies $\eta(\rvx)$ is $\lambda/4$-Lipschitz. Since $\Xcal$ is compact, it also implies that $f^*$ is bounded, that is, $\eta$ is bounded away from $0$ and $1$: the assumption excludes noiseless regions where $\eta\in\{0,1\}$.
Define the function $f^{\dagger}$ as:
\begin{equation}\label{eq:approx_f_in_F}
    f^{\dagger}(\rvx) = \rvw\cdot \sigma(\rmM\rvx) ;  w_j = \frac{1}{k}\cdot f^*(C_j) =\frac{1}{k\cdot\mu(C_j)}\int_{C_j}f^*d\mu
\end{equation} 

where $C_j = \{\rvx\in\Xcal: z_j=1, \rvz=\sigma(\rmM\rvx)\}$ (see \citet{dasguptatosh2020} for details). If $\mu(C_j)=0$ we set $w_j=0$.

\begin{assumption}\label{ass:R}
    The function $f^*(\rvx)=\ln\left(\frac{\eta(\rvx)}{1-\eta(\rvx)}\right)$ satisfies $\frac{1}{k}\sqrt{\sum_{j=1}^m(f^*(C_j))^2}\leq R$.
\end{assumption}

It is easy to see that if Assumption \ref{ass:R} is satisfied, then $f^{\dagger}\in\Fcal_{R}$. Assumption \ref{ass:R} can always be met by the choice of $R$: since $f^*$ is $\lambda$-Lipschitz on the compact set $\mathbb{S}^{d-1}$, it is bounded, and $|f^*(C_j)|\leq\|f^*\|_\infty$ for every $j$, so Assumption \ref{ass:R} holds for any $R\geq\sqrt{m}\,\|f^*\|_\infty/k$, whatever $\rmM$ is drawn.

\begin{theorem}\label{th:excess_bayes_risk}
     Pick any $R\geq 1/\sqrt{k}$ and $\delta\in(0,1)$, and suppose that the function $f^*(\rvx)=\ln\left(\frac{\eta(\rvx)}{1-\eta(\rvx)}\right)$ satisfies Assumption \ref{ass:eta} and \ref{ass:R}. Let the rows of $\rmM$ be drawn uniformly from $\mathbb{S}^{d-1}$, and let $f_n\in\Fcal_{R}$ be a minimizer of the empirical \logit loss over $\Fcal_R$ (computed, for example, by projected SGD). Then with probability at least $1-2\delta$, the following holds:
    \[L(f_n)-L^*\leq c \underbrace{\sqrt{R\sqrt{\frac{k(1+\log(2/\delta))}{n}}}}_{\text{estimation error}}+ \underbrace{\sqrt{2c_d\lambda\left(\frac{2k}{m}+\frac{2c_0(d\log m+\log(1/\delta))}{m}\right)^{1/(d-1)}}}_{\text{approximation error}}\]
    where $c,c_0$ are positive constants and $c_d$ is a constant that depends on $d$.
\end{theorem}
\begin{remark}
    If $\left(\frac{R^2k}{n}\right)^{1/4}\to 0$ then the estimation error vanishes and if $\left(\frac{\max\{k,d\log m\}}{m}\right)^{1/(d-1)}\to 0$ then the approximation error vanishes.
\end{remark}
In this paper we consider the regime $k=sm$ where $s\leq 0.1$. For this setting, we then immediately have the following result.
\begin{corollary}
    Suppose $k=sm$, where $s\leq 0.1$. Then under the condition of Theorem \ref{th:excess_bayes_risk}, with probability at least $1-2\delta$, the following holds.
    \[L(f_n)-L^* = \underbrace{O\left(\left(\frac{R^2k\log(1/\delta)}{n}\right)^{1/4}\right)}_{\text{estimation error}}+\underbrace{O\left(\left(s+\frac{d\log m +\log(1/\delta)}{m}\right)^{1/(2(d-1))}\right)}_{\text{approximation error}}\]
\end{corollary}
\begin{remark}
    While the estimation error still vanishes as long as $\left(\frac{R^2k}{n}\right)^{1/4}\to 0$, unlike before, our bound on the approximation error does not vanish with increasing $m$; the bound concerns the averaging weights of \eqref{eq:approx_f_in_F}, and a learned $\rvw$ may approximate better. The estimation error need not grow with $m$ either: with $R=\max\{\sqrt{m}\,\|f^*\|_\infty/k,\,1/\sqrt{k}\}$, which satisfies Assumption \ref{ass:R} and the condition of Theorem \ref{th:excess_bayes_risk}, we have $R^2k=\max\{\|f^*\|_\infty^2/s,\,1\}$, so the estimation error is $O\big((\max\{\|f^*\|_\infty^2/s,1\}\log(1/\delta)/n)^{1/4}\big)$, independent of $m$.
\end{remark}
\noindent{\textbf{Proof of Theorem \ref{th:excess_bayes_risk}}}
\begin{proof}
    Consider the logit cost function $\phi(x) = \log_2(1+e^x)$, which is a non-negative non-decreasing convex function satisfying $\lim_{x\to -\infty}\phi(x)=0$ and $\phi(0)=1$. For any function $f: \Xcal\to \R$ (with $L(f)$ the probability of error of the classifier $\sign(f)$), note that cost functional $A(f) = \E \phi(-Yf(\rmX))$ is the expected logistic loss and its empirical version is $A_n(f) = \frac{1}{n}\sum_{i=1}^n\phi(-y_if(\rvx_i))$. It is easy to check that $L(f)\leq A(f)$ and $L_n(f)\leq A_n(f)$. Note that for any $\rvx\in\Xcal$, we have
    \[\E [\phi(-Yf(\rmX))|\rmX=\rvx] = \eta(\rvx)\phi(-f(\rvx))+(1-\eta(\rvx))\phi(f(\rvx))\]

    For each $\eta(\rvx)\in[0,1]$, define 
    $h_{\eta(\rvx)}(\alpha)=\eta(\rvx)\phi(-\alpha)+(1-\eta(\rvx))\phi(\alpha)$. Therefore, the function $f^*$ that minimizes the cost functional $A(f)$ is given by
    \[ f^*(\rvx) = \argmin_{\alpha} h_{\eta(\rvx)}(\alpha)\]
Since $h_{\eta}$ is strictly convex and differentiable, $f^*$ is well defined and the minimum is achieved for the value of $\alpha$ for which $h_{\eta}'(\alpha)=0$, that is when,
\[(1-\eta)\phi'(\alpha)-\eta\phi'(-\alpha)=0\]
Since $\phi'(\alpha) = e^{\alpha}/(1+e^{\alpha})$ and $\phi'(-\alpha)= e^{-\alpha}/(1+e^{-\alpha}) = 1/(1+e^{\alpha})$, plugging in these values and simplifying (the common factor $\log_2 e$ of $\phi'$ cancels) yields $\alpha = \ln\left(\frac{\eta}{1-\eta}\right)$, that is $f^*(\rvx) = \ln\left(\frac{\eta(\rvx)}{1-\eta(\rvx)}\right)$. Note that $f^*(\rvx)$ is positive if and only if $\eta(\rvx)>1/2$. Therefore, the function $g^*(\rvx) = 2\I_{\{f^*(\rvx)>0\}}-1$ is the Bayes optimal classifier.

Let $A^* = A(f^*) = \inf_f A(f)$. Now we want to bound the excess probability of error $L(f)-L^*$ in terms of $A(f)-A^*$. Towards this end we use an inequality from \citet{zhang2004_statistical} which says that if the function $H: [0,1]\to \R$ defined by $H(\eta) = \inf_{\alpha}h_{\eta}(\alpha)$ and cost function $\phi$ is such that for some constants $q\geq 1$ and $c_Z>0$ it satisfies
\begin{equation}\label{eq:zhang_condition}
    \left|1/2-\eta\right|^q \leq c_Z^q (1-H(\eta)), \eta\in[0,1]
\end{equation}

then, for any function $f: \Xcal\to \R$ it holds that 
\[L(f)-L^* \leq 2c_Z\left(A(f)-A^*\right)^{1/q}\]

In our case, since $\alpha = \ln\left(\frac{\eta}{1-\eta}\right)$ minimizes $h_{\eta}(\alpha)$, plugging in this value and simplifying yields $H(\eta) = -\eta\log_2\eta-(1-\eta)\log_2(1-\eta)$, the binary entropy. By Pinsker's inequality, $1-H(\eta)\geq \frac{2}{\ln 2}\left(1/2-\eta\right)^2\geq 2\left(1/2-\eta\right)^2$ for all $\eta\in[0,1]$. Therefore, \eqref{eq:zhang_condition} holds with $q=2$ and $c_Z=1/\sqrt{2}$, and $L(f)-L^*\leq\sqrt{2(A(f)-A^*)}$ for any $f:\Xcal\to\R$. 

Let $f_n\in\Fcal_R$ be a minimizer of the empirical \logit loss $A_n$ over $\Fcal_R$.  
Then we have,
\begin{eqnarray*}
    L(f_n) - L^* &\leq& \sqrt{2(A(f_n)-A^*)}\\
   &=& \sqrt{2(A(f_n)-\inf_{f\in\Fcal_{R}}A(f) + \underbrace{\inf_{f\in\Fcal_{R}}A(f)-A(f^{\dagger})}_{\leq 0} + A(f^{\dagger}) -A^*)}\\
   &\leq&\sqrt{2|A(f_n)-\inf_{f\in\Fcal_{R}}A(f)|}+\sqrt{2|A(f^{\dagger})-A^*|}\\
    &\stackrel{a}{\leq}& \sqrt{4\sup_{f\in\Fcal_{R}}|A(f)-A_n(f)|} + \sqrt{2|A(f^{\dagger})-A^*|}\\
    &=&2\sqrt{\sup_{f\in\Fcal_{R}}|A(f)-A_n(f)|}+\sqrt{2|A(f^{\dagger})-A^*|}\\
    &\stackrel{b}{\leq}&2\sqrt{2\log_2 e\,R\sqrt{\frac{k}{n}}+(1+R\sqrt{k}\log_2e)\sqrt{\frac{2\log(2/\delta)}{n}}}+\sqrt{2|A(f^{\dagger})-A^*|} \\
    &\stackrel{c}{\leq}& c \sqrt{R\sqrt{\frac{k(1+\log(2/\delta))}{n}}}+ \sqrt{2c_d\lambda\left(\frac{2k}{m}+\frac{2c_0(d\log m+\log(1/\delta))}{m}\right)^{1/(d-1)}}
\end{eqnarray*}
where, inequality $a$ follows using elementary inequality (for example, see (2) of \citet{boucheron2005_theory_of_classification}) since $f_n$ minimizes $A_n$ over $\Fcal_R$. Inequality $b$ follows from Lemma \ref{lem:estimation_error} and noting the simple fact that $B$ in Lemma \ref{lem:estimation_error} is at most $1+R\sqrt{k}\log_2e$. Inequality $c$ is due to Lemma \ref{lem:approximation_error} and simplification of the first term on the line above, using $R\sqrt{k}\geq 1$ (so that $1+R\sqrt{k}\log_2 e\leq(1+\log_2 e)R\sqrt{k}$) and $1+\sqrt{a}\leq\sqrt{2(1+a)}$. 
\end{proof}

\begin{lemma}\label{lem:estimation_error}(Estimation error) Pick any $R>0$. Let $B$ denote a uniform upper bound on $\phi(-f(\rvx)y)$ for any $f\in\Fcal_R$ and let $L_{\phi}=\log_2 e$ be the Lipschitz constant of the logit cost function $\phi(x) = \log_2(1+e^x)$. Fix any $\delta\in (0,1)$. Then with probability at least $1-\delta$, the following holds:
\[\sup_{f\in\Fcal_R}|A(f)-A_n(f)|\leq 2\log_2 e\,R\sqrt{\frac{k}{n}}+B\sqrt{\frac{2\log(2/\delta)}{n}}.\]
\end{lemma}
\begin{proof}
Let $\Hcal_R$ be the class of functions $\Xcal\times \{1,-1\}\to\R$ of the form $-f(\rvx)y, f\in\Fcal_R$. For the training set $(\rvx_1,y_1),\ldots,(\rvx_n,y_n)$, let $\rvz_1^n = ((\rvx_1,y_1),\ldots,(\rvx_n,y_n))$ and $\rvx_1^n = (\rvx_1,\ldots,\rvx_n)$. Let $\Fcal_R(\rvx_1^n)=\{((f(\rvx_1),\ldots,f(\rvx_n)): f\in\Fcal_R\}$ and $\Hcal_R(\rvz_1^n)=\{((-f(\rvx_1)y_1,\ldots,-f(\rvx_n)y_n): f\in\Fcal_R\}$, where $\rvz_1,\ldots,\rvz_n \triangleq (\rvx_1,y_1),\ldots,(\rvx_n,y_n)$. Using Theorem 3.2 of \citet{boucheron2005_theory_of_classification}, with probability at least $1-\delta$,
    \begin{eqnarray*}
        \sup_{f\in\Fcal_R} |A(f)-A_n(f)|&\leq& 2\Rcal(\phi\circ \Hcal_R(\rvz_1^n))+ B\sqrt{\frac{2\log(2/\delta)}{n}}\\
        &\stackrel{a}{\leq}& 2L_{\phi}\Rcal(\Hcal_R(\rvz_1^n))+ B\sqrt{\frac{2\log(2/\delta)}{n}} \\
        &=& 2\log_2 e\,\Rcal(\Hcal_R(\rvz_1^n))+ B\sqrt{\frac{2\log(2/\delta)}{n}} \\
        &\stackrel{b}{=}& 2\log_2 e\,\Rcal(\Fcal_R(\rvx_1^n))+ B\sqrt{\frac{2\log(2/\delta)}{n}} \\
        &\stackrel{c}{\leq}& \frac{2\log_2 e\,R\max_i\|\sigma(\rmM\rvx_i)\|_2}{\sqrt{n}} + B\sqrt{\frac{2\log(2/\delta)}{n}} \\
        &\stackrel{d}{\leq}& 2\log_2 e\,R\sqrt{\frac{k}{n}} + B\sqrt{\frac{2\log(2/\delta)}{n}}
    \end{eqnarray*}
    where, inequality $a$ is due to contraction principle (Theorem 3.3 of \citet{boucheron2005_theory_of_classification}), equality $b$ is due to the fact that Rademacher complexity $\Rcal(\Hcal_R(\rvz_1^n))=\Rcal(\Fcal_R(\rvx_1^n))$ (again, due to Theorem 3.3 of \citet{boucheron2005_theory_of_classification}), and inequality $c$ is due to Theorem 26.10 of \citet{shalev2014understanding}. 
    Finally, inequality $d$ follows from the fact that 
    for any $\rvx\in\R^d$, since $\sigma(\rmM\rvx)$ is a $k$-sparse binary vector the quantity $\max_i\|\sigma(\rmM\rvx_i)\|_2 = \sqrt{k}$.
\end{proof}

\begin{lemma}\label{lem:approximation_error}(Approximation error) 
Pick any $R>0$ and suppose that the function $f^*(\rvx)=\ln\left(\frac{\eta(\rvx)}{1-\eta(\rvx)}\right)$ satisfies Assumption \ref{ass:eta} and \ref{ass:R}. Pick any $\delta\in(0,1)$. Then with probability at least $1-\delta$, over the choice of $\rmM$, we have,
\[|A(f^{\dagger})-A^*|\leq c_d\lambda\left(\frac{2k}{m}+\frac{2c_0(d\log m+\log(1/\delta))}{m}\right)^{1/(d-1)}\]
where $c_d$ is a constant that depends on $d$ and $c_0$ is a positive constant.
\end{lemma}

\begin{proof}
From Assumption \ref{ass:R}, it follows that $f^{\dagger}\in\Fcal_R$. Then we have,
    \begin{eqnarray*}
        |A(f^{\dagger})-A^*|&=&|A(f^{\dagger})-A(f^*)|\\
        &=& \left|\E\left(\log_2\left(1+e^{-yf^{\dagger}(\rvx)}\right)\right)-\E\left(\log_2\left(1+e^{-yf^*(\rvx)}\right)\right)\right| \\
        &=&\left|\E\left(\log_2\left(1+e^{-yf^{\dagger}(\rvx)}\right)-\log_2\left(1+e^{-yf^*(\rvx)}\right)\right)\right|\\
        &\stackrel{a}{\leq}& \E\left|\left(\log_2\left(1+e^{-yf^{\dagger}(\rvx)}\right)-\log_2\left(1+e^{-yf^*(\rvx)}\right)\right)\right| \\
        &\stackrel{b}{\leq}& \log_2 e\,\E\left|yf^{\dagger}(\rvx)-yf^{*}(\rvx)\right| = \log_2 e\,\E\left|f^{\dagger}(\rvx)-f^{*}(\rvx)\right| \\
        &\stackrel{c}{\leq}& c_d\lambda\left(\frac{2k}{m}+\frac{2c_0(d\log m+\log(1/\delta))}{m}\right)^{1/(d-1)}
    \end{eqnarray*}
    where, inequality $a$ is due to Jensen's inequality, inequality $b$ follows from the fact that the function $f(t)=\log_2(1+e^{-t})$ is $\log_2 e$-Lipschitz. Inequality $c$ holds with probability at least $1-\delta$ and follows from Theorem 3 of \citet{dasguptatosh2020}, which applies since the rows of $\rmM$ are uniform on $\mathbb{S}^{d-1}$ and $\Xcal\subseteq\mathbb{S}^{d-1}$; the factor $\log_2 e$ is absorbed into $c_d$.
\end{proof}

\FloatBarrier
\section{Convergence and generalization gap}
\label{app:conv}
\suppressfloats[t]

Beyond test performance, we ask whether a learned \eas representation takes
longer to train than a dense one, and whether it overfits more. Both can be read
from the runs behind \cref{sec:res-vs-relu} without training anything new. For
every configuration we record the number of epochs after which early stopping
ended training (patience $30$, at most $300$ epochs) and the training score of
the model that early stopping kept, each averaged over the three splits; the
generalization gap is the training score minus the test score. We compare each
activation against $\relu$ exactly as in \cref{sec:res-vs-relu}: $\rmM$ and $w$
are learned in both, $k$ is selected on validation and every quantity is read
from the selected configuration, the comparison is paired within each expansion
factor, and significance is a two-sided Wilcoxon signed rank test.
\Cref{fig:conv} shows the median differences for every activation, and
\cref{tab:convcls,tab:convreg} give the numbers for the sparse ones.

\begin{figure}[!ht]
\centering
\includegraphics[width=\linewidth]{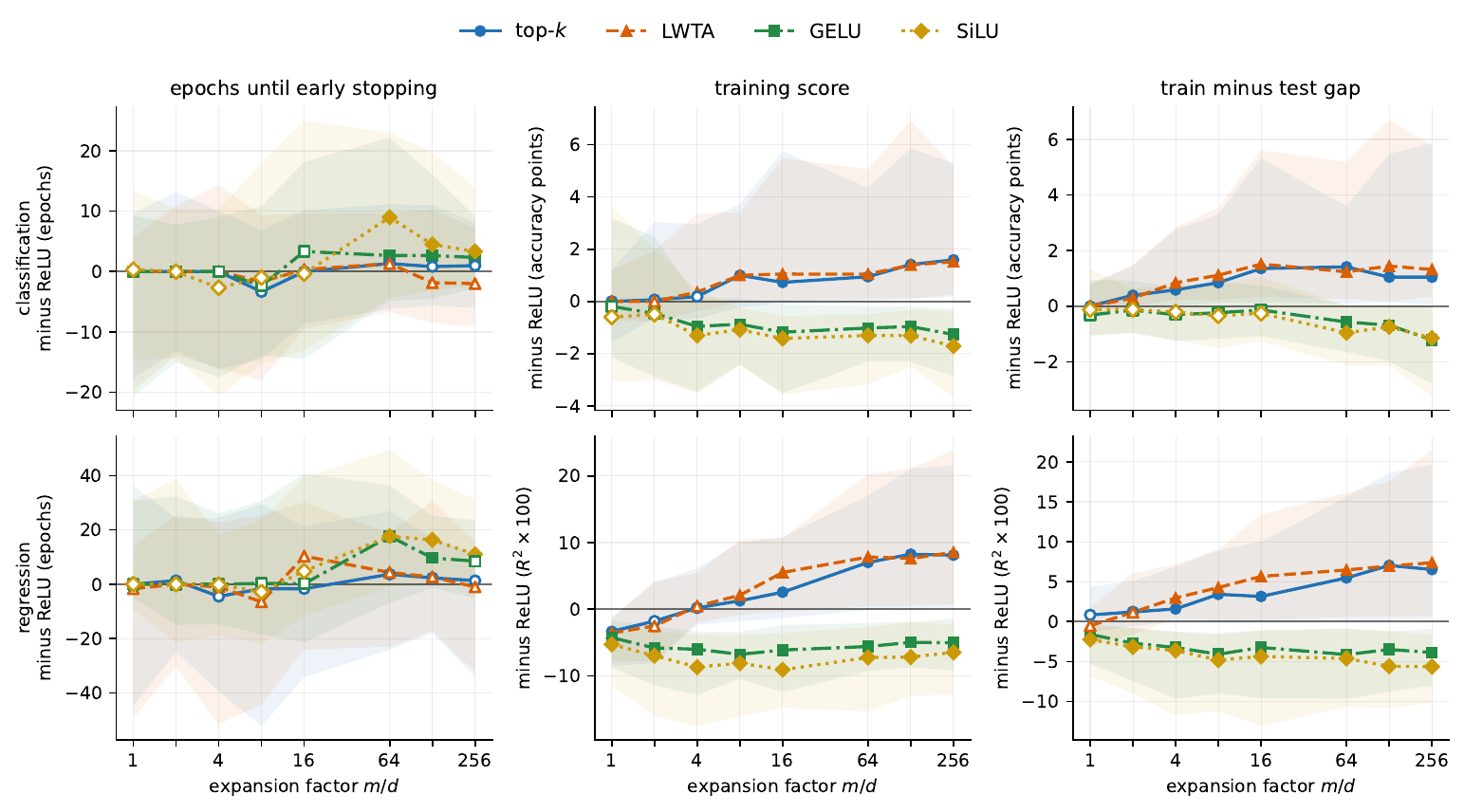}
\caption{Convergence and generalization gap against $\relu$ at equal expanded
dimension, with $\rmM$ and $w$ learned for both methods and $k$ selected on
validation for the sparse activations. [Left] Epochs until early stopping ended
training. [Middle] Training score. [Right] Training minus test score. Top row:
the $73$ classification tasks, in accuracy points; bottom row: the $48$
regression tasks, in $R^2$ multiplied by $100$. Lines are medians over
datasets, ribbons the interquartile range; filled markers are significant
(two-sided Wilcoxon, $p < 0.05$) and hollow ones are not. The test-score
comparison is in \cref{fig:boxes}.}
\label{fig:conv}
\end{figure}

\begin{table}[t]
\caption{Convergence and generalization gap on the $73$ classification tasks: each sparse activation minus $\relu$, with $\rmM$ and $w$ learned in both and $k$ selected on validation. Epochs are counted until early stopping ended training; the training score and the gap (training minus test score) are in accuracy points. Each cell gives the median difference, wins/ties/losses over datasets and the two-sided Wilcoxon $p$-value; a star marks $p<0.05$. The test-score comparison is in \cref{tab:wtlfull}.}
\label{tab:convcls}
\centering
\scriptsize
\setlength{\tabcolsep}{2.5pt}
\begin{tabular}{lcccccccc}
\toprule
comparison & $m/d=1$ & $m/d=2$ & $m/d=4$ & $m/d=8$ & $m/d=16$ & $m/d=64$ & $m/d=128$ & $m/d=256$ \\
\midrule
epochs, top-$k$ & \begin{tabular}[c]{@{}c@{}}$0.0$\\ 18/12/27\\ (.245)\end{tabular} & \begin{tabular}[c]{@{}c@{}}$0.0$\\ 28/11/27\\ (.847)\end{tabular} & \begin{tabular}[c]{@{}c@{}}$0.0$\\ 29/8/34\\ (.479)\end{tabular} & \begin{tabular}[c]{@{}c@{}}$-3.3$\\ 28/4/39\\ (.166)\end{tabular} & \begin{tabular}[c]{@{}c@{}}$0.0$\\ 34/4/33\\ (.788)\end{tabular} & \begin{tabular}[c]{@{}c@{}}$+1.3$\\ 38/1/32\\ (.361)\end{tabular} & \begin{tabular}[c]{@{}c@{}}$+0.8$\\ 37/2/31\\ (.465)\end{tabular} & \begin{tabular}[c]{@{}c@{}}$+1.0$\\ 39/0/28\\ (.277)\end{tabular} \\
\addlinespace[4pt]
epochs, LWTA & \begin{tabular}[c]{@{}c@{}}$0.0$\\ 16/13/28\\ (.163)\end{tabular} & \begin{tabular}[c]{@{}c@{}}$0.0$\\ 29/10/27\\ (.893)\end{tabular} & \begin{tabular}[c]{@{}c@{}}$0.0$\\ 34/5/32\\ (.821)\end{tabular} & \begin{tabular}[c]{@{}c@{}}$-1.7$\\ 31/4/36\\ (.352)\end{tabular} & \begin{tabular}[c]{@{}c@{}}$+0.3$\\ 36/4/31\\ (.923)\end{tabular} & \begin{tabular}[c]{@{}c@{}}$+1.3$\\ 38/4/29\\ (.474)\end{tabular} & \begin{tabular}[c]{@{}c@{}}$-1.8$\\ 31/1/38\\ (.936)\end{tabular} & \begin{tabular}[c]{@{}c@{}}$-2.0$\\ 29/0/38\\ (.500)\end{tabular} \\
\addlinespace[4pt]
training score, top-$k$ & \begin{tabular}[c]{@{}c@{}}$0.00$\\ 23/13/21\\ (.874)\end{tabular} & \begin{tabular}[c]{@{}c@{}}$+0.06$\\ 31/13/22\\ (.105)\end{tabular} & \begin{tabular}[c]{@{}c@{}}$+0.19$\\ 36/9/26\\ (.056)\end{tabular} & \begin{tabular}[c]{@{}c@{}}$+1.00$\\ 40/10/21$^{\ast}$\\ ($<$.001)\end{tabular} & \begin{tabular}[c]{@{}c@{}}$+0.73$\\ 46/9/16$^{\ast}$\\ ($<$.001)\end{tabular} & \begin{tabular}[c]{@{}c@{}}$+0.95$\\ 51/8/12$^{\ast}$\\ ($<$.001)\end{tabular} & \begin{tabular}[c]{@{}c@{}}$+1.41$\\ 52/10/8$^{\ast}$\\ ($<$.001)\end{tabular} & \begin{tabular}[c]{@{}c@{}}$+1.59$\\ 52/7/8$^{\ast}$\\ ($<$.001)\end{tabular} \\
\addlinespace[4pt]
training score, LWTA & \begin{tabular}[c]{@{}c@{}}$0.00$\\ 24/14/19\\ (.555)\end{tabular} & \begin{tabular}[c]{@{}c@{}}$0.00$\\ 31/14/21\\ (.143)\end{tabular} & \begin{tabular}[c]{@{}c@{}}$+0.36$\\ 38/11/22$^{\ast}$\\ (.009)\end{tabular} & \begin{tabular}[c]{@{}c@{}}$+1.00$\\ 47/7/17$^{\ast}$\\ ($<$.001)\end{tabular} & \begin{tabular}[c]{@{}c@{}}$+1.05$\\ 51/11/9$^{\ast}$\\ ($<$.001)\end{tabular} & \begin{tabular}[c]{@{}c@{}}$+1.05$\\ 53/13/5$^{\ast}$\\ ($<$.001)\end{tabular} & \begin{tabular}[c]{@{}c@{}}$+1.40$\\ 53/9/8$^{\ast}$\\ ($<$.001)\end{tabular} & \begin{tabular}[c]{@{}c@{}}$+1.52$\\ 55/9/3$^{\ast}$\\ ($<$.001)\end{tabular} \\
\addlinespace[4pt]
gap, top-$k$ & \begin{tabular}[c]{@{}c@{}}$+0.01$\\ 26/16/15$^{\ast}$\\ (.032)\end{tabular} & \begin{tabular}[c]{@{}c@{}}$+0.39$\\ 42/12/12$^{\ast}$\\ ($<$.001)\end{tabular} & \begin{tabular}[c]{@{}c@{}}$+0.59$\\ 51/10/10$^{\ast}$\\ ($<$.001)\end{tabular} & \begin{tabular}[c]{@{}c@{}}$+0.85$\\ 55/7/9$^{\ast}$\\ ($<$.001)\end{tabular} & \begin{tabular}[c]{@{}c@{}}$+1.36$\\ 53/8/10$^{\ast}$\\ ($<$.001)\end{tabular} & \begin{tabular}[c]{@{}c@{}}$+1.41$\\ 53/7/11$^{\ast}$\\ ($<$.001)\end{tabular} & \begin{tabular}[c]{@{}c@{}}$+1.05$\\ 53/7/10$^{\ast}$\\ ($<$.001)\end{tabular} & \begin{tabular}[c]{@{}c@{}}$+1.05$\\ 50/8/9$^{\ast}$\\ ($<$.001)\end{tabular} \\
\addlinespace[4pt]
gap, LWTA & \begin{tabular}[c]{@{}c@{}}$0.00$\\ 24/14/19\\ (.320)\end{tabular} & \begin{tabular}[c]{@{}c@{}}$+0.29$\\ 41/15/10$^{\ast}$\\ ($<$.001)\end{tabular} & \begin{tabular}[c]{@{}c@{}}$+0.84$\\ 56/6/9$^{\ast}$\\ ($<$.001)\end{tabular} & \begin{tabular}[c]{@{}c@{}}$+1.11$\\ 57/8/6$^{\ast}$\\ ($<$.001)\end{tabular} & \begin{tabular}[c]{@{}c@{}}$+1.52$\\ 58/9/4$^{\ast}$\\ ($<$.001)\end{tabular} & \begin{tabular}[c]{@{}c@{}}$+1.24$\\ 58/8/5$^{\ast}$\\ ($<$.001)\end{tabular} & \begin{tabular}[c]{@{}c@{}}$+1.44$\\ 55/8/7$^{\ast}$\\ ($<$.001)\end{tabular} & \begin{tabular}[c]{@{}c@{}}$+1.32$\\ 54/8/5$^{\ast}$\\ ($<$.001)\end{tabular} \\
\bottomrule
\end{tabular}
\end{table}

\begin{table}[t]
\caption{The comparisons of \cref{tab:convcls} on the $48$ regression tasks, with the training score and the gap in $R^2$, reported raw. The test-score comparison is in \cref{tab:wtlreg}.}
\label{tab:convreg}
\centering
\scriptsize
\setlength{\tabcolsep}{2.5pt}
\begin{tabular}{lcccccccc}
\toprule
comparison & $m/d=1$ & $m/d=2$ & $m/d=4$ & $m/d=8$ & $m/d=16$ & $m/d=64$ & $m/d=128$ & $m/d=256$ \\
\midrule
epochs, top-$k$ & \begin{tabular}[c]{@{}c@{}}$0.0$\\ 15/5/12\\ (.683)\end{tabular} & \begin{tabular}[c]{@{}c@{}}$+1.3$\\ 24/3/16\\ (.984)\end{tabular} & \begin{tabular}[c]{@{}c@{}}$-4.5$\\ 20/1/25\\ (.245)\end{tabular} & \begin{tabular}[c]{@{}c@{}}$-1.7$\\ 22/1/25\\ (.377)\end{tabular} & \begin{tabular}[c]{@{}c@{}}$-1.7$\\ 22/0/26\\ (.335)\end{tabular} & \begin{tabular}[c]{@{}c@{}}$+3.7$\\ 28/0/19\\ (.615)\end{tabular} & \begin{tabular}[c]{@{}c@{}}$+2.3$\\ 27/0/20\\ (.676)\end{tabular} & \begin{tabular}[c]{@{}c@{}}$+1.3$\\ 24/0/22\\ (.611)\end{tabular} \\
\addlinespace[4pt]
epochs, LWTA & \begin{tabular}[c]{@{}c@{}}$-1.7$\\ 10/5/17\\ (.100)\end{tabular} & \begin{tabular}[c]{@{}c@{}}$0.0$\\ 20/3/20\\ (.614)\end{tabular} & \begin{tabular}[c]{@{}c@{}}$-1.0$\\ 20/2/24\\ (.246)\end{tabular} & \begin{tabular}[c]{@{}c@{}}$-6.5$\\ 21/0/27\\ (.343)\end{tabular} & \begin{tabular}[c]{@{}c@{}}$+10.3$\\ 28/0/20\\ (.678)\end{tabular} & \begin{tabular}[c]{@{}c@{}}$+4.3$\\ 25/0/22\\ (.800)\end{tabular} & \begin{tabular}[c]{@{}c@{}}$+2.7$\\ 27/0/20\\ (.804)\end{tabular} & \begin{tabular}[c]{@{}c@{}}$-1.0$\\ 22/0/24\\ (.523)\end{tabular} \\
\addlinespace[4pt]
training score, top-$k$ & \begin{tabular}[c]{@{}c@{}}$-0.033$\\ 7/0/25$^{\ast}$\\ (.002)\end{tabular} & \begin{tabular}[c]{@{}c@{}}$-0.018$\\ 15/0/28\\ (.136)\end{tabular} & \begin{tabular}[c]{@{}c@{}}$+0.002$\\ 23/0/23\\ (.315)\end{tabular} & \begin{tabular}[c]{@{}c@{}}$+0.013$\\ 28/1/19$^{\ast}$\\ (.019)\end{tabular} & \begin{tabular}[c]{@{}c@{}}$+0.026$\\ 29/1/18$^{\ast}$\\ (.004)\end{tabular} & \begin{tabular}[c]{@{}c@{}}$+0.070$\\ 34/0/13$^{\ast}$\\ ($<$.001)\end{tabular} & \begin{tabular}[c]{@{}c@{}}$+0.083$\\ 33/2/12$^{\ast}$\\ ($<$.001)\end{tabular} & \begin{tabular}[c]{@{}c@{}}$+0.081$\\ 36/1/9$^{\ast}$\\ ($<$.001)\end{tabular} \\
\addlinespace[4pt]
training score, LWTA & \begin{tabular}[c]{@{}c@{}}$-0.036$\\ 7/0/25$^{\ast}$\\ ($<$.001)\end{tabular} & \begin{tabular}[c]{@{}c@{}}$-0.025$\\ 16/0/27\\ (.160)\end{tabular} & \begin{tabular}[c]{@{}c@{}}$+0.005$\\ 23/2/21\\ (.169)\end{tabular} & \begin{tabular}[c]{@{}c@{}}$+0.021$\\ 28/1/19$^{\ast}$\\ (.006)\end{tabular} & \begin{tabular}[c]{@{}c@{}}$+0.055$\\ 31/3/14$^{\ast}$\\ ($<$.001)\end{tabular} & \begin{tabular}[c]{@{}c@{}}$+0.078$\\ 36/1/10$^{\ast}$\\ ($<$.001)\end{tabular} & \begin{tabular}[c]{@{}c@{}}$+0.076$\\ 37/1/9$^{\ast}$\\ ($<$.001)\end{tabular} & \begin{tabular}[c]{@{}c@{}}$+0.085$\\ 36/0/10$^{\ast}$\\ ($<$.001)\end{tabular} \\
\addlinespace[4pt]
gap, top-$k$ & \begin{tabular}[c]{@{}c@{}}$+0.008$\\ 18/1/13\\ (.472)\end{tabular} & \begin{tabular}[c]{@{}c@{}}$+0.012$\\ 29/0/14\\ (.067)\end{tabular} & \begin{tabular}[c]{@{}c@{}}$+0.016$\\ 32/3/11$^{\ast}$\\ (.001)\end{tabular} & \begin{tabular}[c]{@{}c@{}}$+0.034$\\ 33/2/13$^{\ast}$\\ ($<$.001)\end{tabular} & \begin{tabular}[c]{@{}c@{}}$+0.032$\\ 35/1/12$^{\ast}$\\ ($<$.001)\end{tabular} & \begin{tabular}[c]{@{}c@{}}$+0.055$\\ 34/4/9$^{\ast}$\\ ($<$.001)\end{tabular} & \begin{tabular}[c]{@{}c@{}}$+0.070$\\ 35/2/10$^{\ast}$\\ ($<$.001)\end{tabular} & \begin{tabular}[c]{@{}c@{}}$+0.065$\\ 36/1/9$^{\ast}$\\ ($<$.001)\end{tabular} \\
\addlinespace[4pt]
gap, LWTA & \begin{tabular}[c]{@{}c@{}}$-0.005$\\ 13/1/18\\ (.438)\end{tabular} & \begin{tabular}[c]{@{}c@{}}$+0.012$\\ 28/1/14\\ (.104)\end{tabular} & \begin{tabular}[c]{@{}c@{}}$+0.030$\\ 36/1/9$^{\ast}$\\ ($<$.001)\end{tabular} & \begin{tabular}[c]{@{}c@{}}$+0.043$\\ 35/3/10$^{\ast}$\\ ($<$.001)\end{tabular} & \begin{tabular}[c]{@{}c@{}}$+0.057$\\ 37/4/7$^{\ast}$\\ ($<$.001)\end{tabular} & \begin{tabular}[c]{@{}c@{}}$+0.065$\\ 37/2/8$^{\ast}$\\ ($<$.001)\end{tabular} & \begin{tabular}[c]{@{}c@{}}$+0.070$\\ 39/1/7$^{\ast}$\\ ($<$.001)\end{tabular} & \begin{tabular}[c]{@{}c@{}}$+0.074$\\ 38/0/8$^{\ast}$\\ ($<$.001)\end{tabular} \\
\bottomrule
\end{tabular}
\end{table}

\paragraph{Convergence.} The sparse activations do not take longer to train.
The number of epochs until early stopping is not significantly different from
$\relu$ at any expansion factor, on classification or on regression, for
$\topk$ ($p \geq 0.166$) or for $\lwta$ ($p \geq 0.100$). The medians stay
within $3.3$ epochs of $\relu$ on classification and within $4.5$ ($\topk$)
and $10.3$ ($\lwta$) on regression, against median training lengths of $57$ and
$106$ epochs. It is the smooth dense activations that train longer: from
$m = 64d$, \texttt{GELU} and \texttt{SiLU} stop a median $2.3$ to $9.0$ epochs
later than $\relu$ on classification and $8.5$ to $17.7$ epochs later on
regression, significantly at every one of these expansion factors except
\texttt{GELU} on regression at $m = 256d$ ($p = 0.080$). Two remarks qualify
this. First, epochs measure optimization progress rather than wall-clock time,
which depends on how the selection of the $k$ largest units is implemented
(\cref{app:sparse-advantages}). Second, the $300$-epoch cap stopped a few runs
before early stopping did (on all three splits in $2$ of $544$ classification
cells for $\topk$ and $4$ of $567$ for $\relu$, and in $14$ of $357$ and $42$ of
$380$ regression cells); the cap truncates $\relu$ more often, so it does not
make the sparse activations look faster than they are.

\paragraph{Generalization gap.} The sparse activations do have a larger gap
between training and test score than $\relu$. For $\topk$ the difference is
significant at every expansion factor on classification and from $m = 4d$ on
regression, and it is larger at wider expansions: a median of $0.85$ to $1.41$
accuracy points from $m = 8d$, and $0.055$ to $0.070$ in $R^2$ from $m = 64d$.
$\lwta$ is the same (significant at $7$ and $6$ of the $8$ expansion factors).
The larger gap comes from the training score, not from the test score:
from $m = 8d$, $\topk$ reaches a significantly higher training score than
$\relu$ on both tasks (by a median $0.73$ to $1.59$ accuracy points and $0.013$
to $0.083$ in $R^2$), while its test score is not significantly different at any
expansion factor on classification and from $m = 4d$ on regression
(\cref{sec:res-vs-relu}). Where $\topk$ does trail $\relu$ on regression, at
$m = d$ and $m = 2d$, its training score is lower too (by a median $0.033$ and
$0.018$ in $R^2$, significantly at $m = d$, $p = 0.002$), so there the sparse
representation underfits rather than overfits. The smooth
activations go the other way: \texttt{GELU} and \texttt{SiLU} reach a lower
training score than $\relu$ (significantly at $6$ of the $8$ expansion factors
on classification and at all $8$ on regression) and have a smaller gap (from
$m = 64d$ on classification and at every expansion factor on regression),
while their test score is also lower.

In summary, learning an \eas representation by gradient descent is not slower
than learning a dense $\relu$ representation of the same size, and the learned
\eas representation fits the training data more closely without losing test
performance (from $m = 4d$ on regression). Its gap is therefore larger, not
smaller: the additional fit to the training data does not carry over to the
test data. As in \cref{app:tuning}, $k$ is
selected for the sparse activations while $\relu$ has no comparable choice. That
selection, like early stopping, uses the validation score and never the training
score, so a higher training score is not what it optimizes for: the values of
$k$ chosen for their validation score also turn out to fit the training data
more closely than $\relu$.

\section{Scaling the pre-activation}
\label{app:scale}
Inputs and the rows of $\rmM$ have unit norm and the hidden layer has no bias, so
every pre-activation lies in $[-1, 1]$. To test how much of the comparison
between activations this normalization decides, we multiply every
pre-activation by a fixed $c \in \{4, 16\}$ before the activation, for
\texttt{GELU} and \texttt{SiLU} and, as controls, for $\relu$ and top-$k$, with
$\rmM$ learned, at $m/d \in \{2, 4, 8, 16, 64\}$, on the classification tasks for
which every arm is available ($66$ to $71$ per width). Everything else follows \cref{sec:expsetup};
$k$ is chosen on validation over the same grid as in \cref{sec:res-vs-relu}.
The runs with $c = 1$ reproduce the cells of \cref{tab:wtlfull} exactly.

\Cref{tab:scale} gives the results. Scaling helps every activation at narrow
widths, $\relu$ and top-$k$ included: at $c = 16$ and $m = 4d$ the median gain
is $+1.13$ points for $\relu$, $+1.00$ for top-$k$,
$+1.42$ for \texttt{GELU} and $+1.98$ for \texttt{SiLU}, and the
gains shrink with width. $\relu$ and top-$k$ are positively homogeneous, so for
them $c$ only rescales the hidden output, which with the readout initialized at
$1/m$ and trained by Adam acts like a larger initial readout and step size; their
gain is an effect on optimization, not on the function class. At equal $c \in \{4, 16\}$,
\texttt{GELU} and \texttt{SiLU} are equivalent to $\relu$ (median gaps within
$0.13$ points, $\delta^\ast \leq 0.49$). So the deficit of the smooth
activations in \cref{sec:res-vs-relu} comes from our normalization, which keeps
them in a nearly quadratic regime, and not from the activations themselves. With
$c$ chosen on validation for every activation, as $k$ is, \texttt{GELU} and
\texttt{SiLU} match $\relu$ ($\delta^\ast \leq 0.34$), and top-$k$ is within
$1.24$, $1.18$, $0.87$, $0.79$ and $0.35$ points of $\relu$ at
$m/d = 2, 4, 8, 16, 64$: top-$k$ gains a little less from scaling than $\relu$, and
the difference shrinks as the layer widens.

\begin{table}[h]
\centering\small
\caption{Scaling every pre-activation by $c$ ($\rmM$ learned, classification,
$66$ to $71$ tasks per column). Each cell: median paired difference in accuracy points,
$^{\ast}$ if a two-sided Wilcoxon test gives $p < 0.05$, and in parentheses
$\delta^\ast$, the smallest margin at which two one-sided Wilcoxon tests reject at
$0.05$.}
\label{tab:scale}
\setlength{\tabcolsep}{3pt}
\resizebox{\textwidth}{!}{%
\begin{tabular}{lccccc}
\toprule
& $m/d = 2$ & $m/d = 4$ & $m/d = 8$ & $m/d = 16$ & $m/d = 64$ \\
\midrule
\multicolumn{6}{l}{\emph{Gain from scaling: activation at $c$ minus the same activation at $c = 1$}} \\
$\relu$, $c = 4$ & $+0.53$$^{\ast}$ (2.67) & $+0.26$$^{\ast}$ (1.59) & $+0.49$$^{\ast}$ (1.16) & $+0.33$$^{\ast}$ (0.70) & $+0.12$$^{\ast}$ (0.59) \\
\texttt{GELU}, $c = 4$ & $+1.49$$^{\ast}$ (3.82) & $+1.08$$^{\ast}$ (2.49) & $+0.94$$^{\ast}$ (2.10) & $+1.19$$^{\ast}$ (1.95) & $+0.53$$^{\ast}$ (1.47) \\
\texttt{SiLU}, $c = 4$ & $+1.80$$^{\ast}$ (4.28) & $+1.32$$^{\ast}$ (3.61) & $+1.48$$^{\ast}$ (2.15) & $+1.30$$^{\ast}$ (2.32) & $+0.71$$^{\ast}$ (1.28) \\
top-$k$, $c = 4$ & $+0.80$$^{\ast}$ (2.17) & $+0.50$$^{\ast}$ (1.39) & $+0.33$$^{\ast}$ (0.83) & $+0.36$$^{\ast}$ (0.91) & $+0.17$ (0.40) \\
$\relu$, $c = 16$ & $+1.36$$^{\ast}$ (3.90) & $+1.13$$^{\ast}$ (2.43) & $+1.26$$^{\ast}$ (2.03) & $+0.67$$^{\ast}$ (1.26) & $+0.04$ (0.54) \\
\texttt{GELU}, $c = 16$ & $+2.13$$^{\ast}$ (4.48) & $+1.42$$^{\ast}$ (3.47) & $+1.67$$^{\ast}$ (2.74) & $+1.27$$^{\ast}$ (2.50) & $+0.45$$^{\ast}$ (1.11) \\
\texttt{SiLU}, $c = 16$ & $+2.51$$^{\ast}$ (5.57) & $+1.98$$^{\ast}$ (5.03) & $+1.87$$^{\ast}$ (2.79) & $+1.50$$^{\ast}$ (2.67) & $+0.47$$^{\ast}$ (1.15) \\
top-$k$, $c = 16$ & $+0.72$$^{\ast}$ (2.58) & $+1.00$$^{\ast}$ (2.11) & $+0.75$$^{\ast}$ (1.76) & $+0.55$$^{\ast}$ (0.96) & $+0.17$ (0.48) \\
\midrule
\multicolumn{6}{l}{\emph{At equal $c$: activation minus $\relu$ at the same $c$}} \\
\texttt{GELU}, $c = 1$ & $-0.03$ (1.01) & $-0.39$$^{\ast}$ (1.34) & $-0.38$$^{\ast}$ (0.79) & $-0.44$$^{\ast}$ (1.29) & $-0.31$$^{\ast}$ (0.86) \\
\texttt{SiLU}, $c = 1$ & $-0.45$ (1.43) & $-0.92$$^{\ast}$ (2.18) & $-0.62$$^{\ast}$ (1.27) & $-0.79$$^{\ast}$ (1.51) & $-0.48$$^{\ast}$ (0.90) \\
top-$k$, $c = 1$ & $0.00$ (0.45) & $-0.09$ (0.71) & $-0.22$ (0.59) & $-0.05$ (0.42) & $-0.01$ (0.31) \\
\texttt{GELU}, $c = 4$ & $0.00$ (0.25) & $0.00$ (0.20) & $0.00$ (0.21) & $0.00$ (0.15) & $0.00$ (0.21) \\
\texttt{SiLU}, $c = 4$ & $-0.13$ (0.49) & $0.00$ (0.33) & $+0.04$ (0.30) & $+0.04$ (0.28) & $-0.05$ (0.33) \\
top-$k$, $c = 4$ & $-0.03$ (0.68) & $-0.24$ (0.80) & $-0.33$$^{\ast}$ (0.76) & $-0.04$ (0.39) & $-0.18$ (0.46) \\
\texttt{GELU}, $c = 16$ & $0.00$ (0.25) & $0.00$ (0.10) & $0.00$ (0.11) & $+0.02$ (0.14) & $0.00$ (0.13) \\
\texttt{SiLU}, $c = 16$ & $0.00$ (0.28) & $0.00$ (0.14) & $0.00$ (0.10) & $-0.04$ (0.23) & $0.00$ (0.32) \\
top-$k$, $c = 16$ & $-0.65$$^{\ast}$ (1.46) & $-0.63$$^{\ast}$ (1.17) & $-0.47$$^{\ast}$ (0.81) & $-0.35$$^{\ast}$ (0.71) & $-0.01$ (0.44) \\
\midrule
\multicolumn{6}{l}{\emph{$c \in \{1, 4, 16\}$ chosen on validation for each activation, minus $\relu$ chosen the same way}} \\
\texttt{GELU} & $0.00$ (0.20) & $0.00$ (0.08) & $0.00$ (0.19) & $+0.08$ (0.22) & $+0.04$ (0.34) \\
\texttt{SiLU} & $-0.05$ (0.22) & $+0.07$$^{\ast}$ (0.33) & $+0.08$ (0.18) & $0.00$ (0.20) & $0.00$ (0.17) \\
top-$k$ & $-0.67$$^{\ast}$ (1.24) & $-0.37$$^{\ast}$ (1.18) & $-0.43$$^{\ast}$ (0.87) & $-0.37$ (0.79) & $-0.10$ (0.35) \\
\bottomrule
\end{tabular}}
\end{table}
\FloatBarrier

\section{Measured speed of a top-$k$ FFN}
\label{app:bench}
The argument of \cref{app:sparse-advantages} can be checked directly. We time one FFN, $\rvy = \rmW_2 \sigma(\rmW_1 \rvx)$
with $\rmW_1 \in \R^{m \times d}$, at batch size $1$, the memory-bound setting of
autoregressive decoding in which activation sparsity is used in practice
\citep{liu2023dejavu, liu2025teal}. The dense version applies $\relu$ and a dense
second matrix; the top-$k$ version selects the $k$ largest pre-activations
(\texttt{torch.topk}) and multiplies only the $k$ matching rows of $\rmW_2^\top$,
stored row-major so the selected rows are contiguous. The first matrix is dense
in both, since top-$k$ needs every pre-activation, so the saving can only come
from the second. We report the median over $300$ repetitions after warm-up, on an
RTX 6000 Ada Generation GPU (half precision, CUDA events) and on one CPU core (single precision),
with PyTorch 2.6.0 and no custom kernels (\cref{tab:bench}).
Top-$k$ is slower than the dense FFN for small layers, where the cost of
selection dominates: on the GPU up to $d = 1024$ and at $d = 2048$ with $r = 4$,
and on the CPU at $d = 128$ with $r = 4$. It is faster for large ones: on the GPU from $d = 2048$ at $r = 16$
and at $d = 4096$, by $1.28\times$ to $1.78\times$ for the whole FFN
and up to $7.1\times$ for the second matrix; on the CPU from $d = 512$, by
$1.36\times$ to $2.46\times$ for the whole FFN. With a dedicated kernel that fuses
selection and the sparse product, and with the first matrix also predicted
\citep{liu2023dejavu}, the savings would be larger; we report what a plain
implementation gives.

\begin{table}[h]
\centering\small
\caption{Time of one FFN at batch size $1$, in microseconds (median of $300$
repetitions), for the second matrix alone and for the whole FFN, dense $\relu$
against top-$k$ at $k/m$, and the speed-up (dense time over top-$k$ time; above
$1$ means top-$k$ is faster). $r = m/d$. CPU runs with $m \cdot d > 64$M are
skipped.}
\label{tab:bench}
\setlength{\tabcolsep}{4pt}
\begin{tabular}{lrrrrrrrrr}
\toprule
& & & & \multicolumn{3}{c}{second matrix} & \multicolumn{3}{c}{whole FFN} \\
\cmidrule(lr){5-7}\cmidrule(lr){8-10}
device & $d$ & $r$ & $k/m$ & dense & top-$k$ & speed-up & dense & top-$k$ & speed-up \\
\midrule
GPU & 128 & 4 & 10\% & 15 & 25 & 0.62 & 21 & 37 & 0.56 \\
GPU & 128 & 4 & 3\% & 15 & 28 & 0.55 & 21 & 31 & 0.67 \\
GPU & 128 & 16 & 10\% & 16 & 30 & 0.55 & 22 & 36 & 0.60 \\
GPU & 128 & 16 & 3\% & 16 & 26 & 0.64 & 22 & 37 & 0.59 \\
GPU & 512 & 4 & 10\% & 17 & 30 & 0.58 & 22 & 36 & 0.62 \\
GPU & 512 & 4 & 3\% & 17 & 25 & 0.67 & 22 & 30 & 0.75 \\
GPU & 512 & 16 & 10\% & 20 & 50 & 0.41 & 25 & 52 & 0.47 \\
GPU & 512 & 16 & 3\% & 20 & 46 & 0.44 & 25 & 54 & 0.45 \\
GPU & 1024 & 4 & 10\% & 20 & 34 & 0.61 & 25 & 41 & 0.60 \\
GPU & 1024 & 4 & 3\% & 20 & 29 & 0.71 & 25 & 33 & 0.75 \\
GPU & 1024 & 16 & 10\% & 27 & 71 & 0.38 & 56 & 112 & 0.50 \\
GPU & 1024 & 16 & 3\% & 27 & 67 & 0.40 & 56 & 107 & 0.52 \\
GPU & 2048 & 4 & 10\% & 25 & 51 & 0.48 & 54 & 92 & 0.59 \\
GPU & 2048 & 4 & 3\% & 25 & 45 & 0.55 & 54 & 86 & 0.63 \\
GPU & 2048 & 16 & 10\% & 174 & 81 & 2.15 & 326 & 232 & 1.40 \\
GPU & 2048 & 16 & 3\% & 174 & 57 & 3.05 & 326 & 208 & 1.57 \\
GPU & 4096 & 4 & 10\% & 164 & 91 & 1.79 & 314 & 245 & 1.28 \\
GPU & 4096 & 4 & 3\% & 164 & 68 & 2.42 & 314 & 222 & 1.42 \\
GPU & 4096 & 16 & 10\% & 624 & 233 & 2.67 & 1228 & 842 & 1.46 \\
GPU & 4096 & 16 & 3\% & 624 & 88 & 7.08 & 1228 & 690 & 1.78 \\
\midrule
CPU & 128 & 4 & 10\% & 5 & 10 & 0.55 & 10 & 14 & 0.73 \\
CPU & 128 & 4 & 3\% & 5 & 9 & 0.63 & 10 & 13 & 0.80 \\
CPU & 128 & 16 & 10\% & 11 & 17 & 0.68 & 45 & 33 & 1.37 \\
CPU & 128 & 16 & 3\% & 11 & 14 & 0.82 & 45 & 32 & 1.40 \\
CPU & 512 & 4 & 10\% & 127 & 26 & 4.92 & 257 & 185 & 1.39 \\
CPU & 512 & 4 & 3\% & 127 & 17 & 7.48 & 257 & 157 & 1.64 \\
CPU & 512 & 16 & 10\% & 494 & 232 & 2.13 & 995 & 730 & 1.36 \\
CPU & 512 & 16 & 3\% & 494 & 58 & 8.54 & 995 & 581 & 1.71 \\
CPU & 1024 & 4 & 10\% & 479 & 161 & 2.98 & 953 & 651 & 1.47 \\
CPU & 1024 & 4 & 3\% & 479 & 33 & 14.45 & 953 & 541 & 1.76 \\
CPU & 1024 & 16 & 10\% & 2670 & 754 & 3.54 & 7362 & 4021 & 1.83 \\
CPU & 1024 & 16 & 3\% & 2670 & 287 & 9.29 & 7362 & 3119 & 2.36 \\
CPU & 2048 & 4 & 10\% & 2682 & 609 & 4.40 & 7273 & 3846 & 1.89 \\
CPU & 2048 & 4 & 3\% & 2682 & 197 & 13.61 & 7273 & 2952 & 2.46 \\
\bottomrule
\end{tabular}
\end{table}

\end{document}